\documentclass{article}

\PassOptionsToPackage{numbers, compress}{natbib}
\usepackage[final]{include/neurips_2021}

\usepackage[utf8]{inputenc} %
\usepackage[T1]{fontenc}    %
\usepackage{include/titletoc}
\usepackage{hyperref}       %
\usepackage{url}            %
\usepackage{booktabs}       %
\usepackage{amsfonts}       %
\usepackage{nicefrac}       %
\usepackage{microtype}      %
\usepackage{xcolor}         %

\usepackage[title]{appendix}
\usepackage{multirow}

\usepackage[utf8]{inputenc}
\usepackage{microtype}
\usepackage{graphicx}
\usepackage{booktabs} %
\usepackage{hyperref}

\usepackage{amsmath}
\usepackage{amsthm}
\usepackage{amssymb}
\usepackage{bbm} 
\usepackage{bm}
\usepackage{verbatim}
\usepackage{float}
\usepackage{color,soul}
\usepackage{enumitem}
\usepackage{mathtools}
\usepackage{hhline}
\usepackage[title]{appendix}
\usepackage{natbib}
\usepackage[nameinlink]{cleveref} %
\usepackage[font=small,labelfont=bf,
  justification=justified,
  format=plain]{caption}
\usepackage{subcaption}
\usepackage{tabularx}
\usepackage{tikz}
\usepackage{wrapfig,booktabs}
\usepackage{xspace}
\usepackage{cancel}
\usepackage{sidecap}

\graphicspath{ {../figures/} }

{\par\egroup\vskip 0.25ex}

\DeclarePairedDelimiterX{\infdivx}[2]{(}{)}{%
  #1\;\delimsize\|\;#2%
}

\newcommand{\gps}{\textsc{gp}s\xspace}
\newcommand{\gp}{\textsc{gp}\xspace}

\newcommand{\kld}{KL divergence\xspace}

\crefname{appsec}{appendix}{appendices}
\Crefname{appsec}{Appendix}{Appendices}

\definecolor{mydarkblue}{rgb}{0,0.08,0.45}

\newtheorem{proposition}{Proposition}

\newtheorem{innercustomproposition}{Proposition}
\newenvironment{customproposition}[1]
  {\renewcommand\theinnercustomproposition{#1}\innercustomproposition}
  {\endinnercustomproposition}

\usetikzlibrary{shapes.geometric, arrows, bayesnet, calc, positioning}

\newcommand{\ba}{\mathbf{a}}

\newcommand{\bs}{\mathbf s}
\newcommand{\bS}{\mathbf S}

\newcommand{\calH}{\mathcal{H}}

\newcommand{\calN}{\mathcal{N}}
\newcommand{\calO}{\mathcal{O}}

\renewcommand{\a}{\mathbf{a}}

\newcommand{\bmu}{{\boldsymbol{\mu}}}
\newcommand{\bsigma}{{\boldsymbol{\sigma}}}

\newcommand{\btau}{\bm{\tau}}

\newcommand{\bSigma}{\boldsymbol\Sigma}

\newcommand{\closer}[3]{{\kern-#1ex{#2}\kern-#3ex}}

\newcommand{\DKL}{\mathbb{D}_{\textrm{KL}}\infdivx}

\DeclareMathOperator*{\argmax}{arg\,max}

\mathchardef\mhyphen="2D

\DeclareMathOperator{\E}{\mathbb{E}}

\definecolor{darkgreen}{rgb}{0.0, 0.7, 0.0}

\newcommand\defines{\,\,\dot{=}\,\,}

\newcommand{\vbar}{\,|\,}

\newcommand{\calS}{\mathcal{S}}
\newcommand{\calA}{\mathcal{A}}
\newcommand{\calD}{\mathcal{D}}
\newcommand{\bA}{\mathbf{A}}

\hypersetup{
    colorlinks,
    citecolor=mydarkblue,
    linkcolor=mydarkblue,
    urlcolor=mydarkblue,
    }

\usepackage{algorithm,algorithmicx,algpseudocode}

\title{On Pathologies in KL-Regularized Reinforcement Learning from Expert Demonstrations}

\author{%
    \hspace*{8pt}Tim G. J. Rudner\thanks{Equal contribution. $^\dagger$\,Corresponding author: \href{mailto:tim.rudner@cs.ox.ac.uk}{\texttt{tim.rudner@cs.ox.ac.uk}}.}~~$^\dagger$ \\
    University of Oxford \\
    \And
    \hspace*{4pt}Cong Lu$^\ast$ \\
    University of Oxford \\
    \AND
    Michael A. Osborne \\
    University of Oxford \\
    \And
    Yarin Gal \\
    University of Oxford \\
    \And
    Yee Whye Teh \\
    University of Oxford \\
}

\begin{document}

\maketitle

\begin{abstract}
KL-regularized reinforcement learning from expert demonstrations has proved successful in improving the sample efficiency of deep reinforcement learning algorithms, allowing them to be applied to challenging physical real-world tasks. However, we show that KL-regularized reinforcement learning with behavioral reference policies derived from expert demonstrations can suffer from pathological training dynamics that can lead to slow, unstable, and suboptimal online learning. We show empirically that the pathology occurs for commonly chosen behavioral policy classes and demonstrate its impact on sample efficiency and online policy performance. Finally, we show that the pathology can be remedied by \textit{non-parametric} behavioral reference policies and that this allows KL-regularized reinforcement learning to significantly outperform state-of-the-art approaches on a variety of challenging locomotion and dexterous hand manipulation tasks.
\end{abstract}

\section{Introduction}
\label{sec:intro}

Reinforcement learning (RL)~\citep{kaelbling1996reinforcement, mnih-atari-2013,Sutton1998,tesauro1995temporal} is a powerful paradigm for learning complex behaviors.
Unfortunately, many modern reinforcement learning algorithms require agents to carry out millions of interactions with their environment to learn desirable behaviors, making them of limited use for a wide range of practical applications that cannot be simulated~\citep{dulac2019challenges,navarro2012real}.
This limitation has motivated the study of algorithms that can incorporate pre-collected offline data into the training process, either fully offline or with online exploration, to improve sample efficiency, performance, and reliability~\citep{pmlr-v139-ball21a,cang2021behavioral,morel,lu2021revisiting,mopo,combo}.
An important and well-motivated subset of these methods consists of approaches for efficiently incorporating expert demonstrations into the learning process~\citep{brys2015reinforcement, gao2018reinforcement,konidaris2012robot,schaal1997learning}.

\begin{figure}[t!]
\vspace*{-5pt}
\centering
\begin{subfigure}[t]{0.49\columnwidth}
    \begin{center}
        \large Parametric
    \end{center}
    \includegraphics[width=\columnwidth]{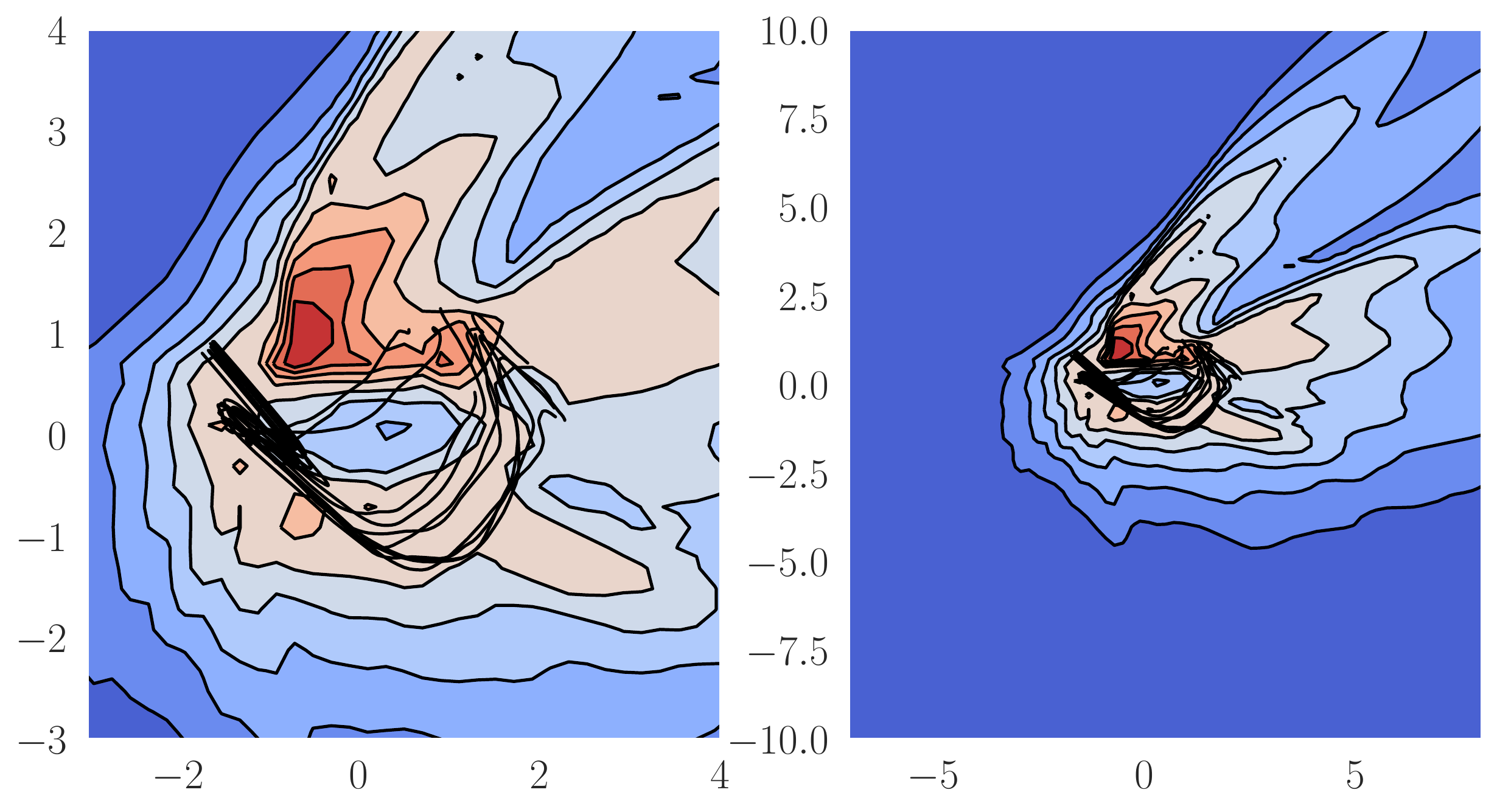}
\end{subfigure}
\begin{subfigure}[t]{0.49\columnwidth}
    \begin{center}
        \large Non-Parametric
    \end{center}
    \includegraphics[width=\columnwidth]{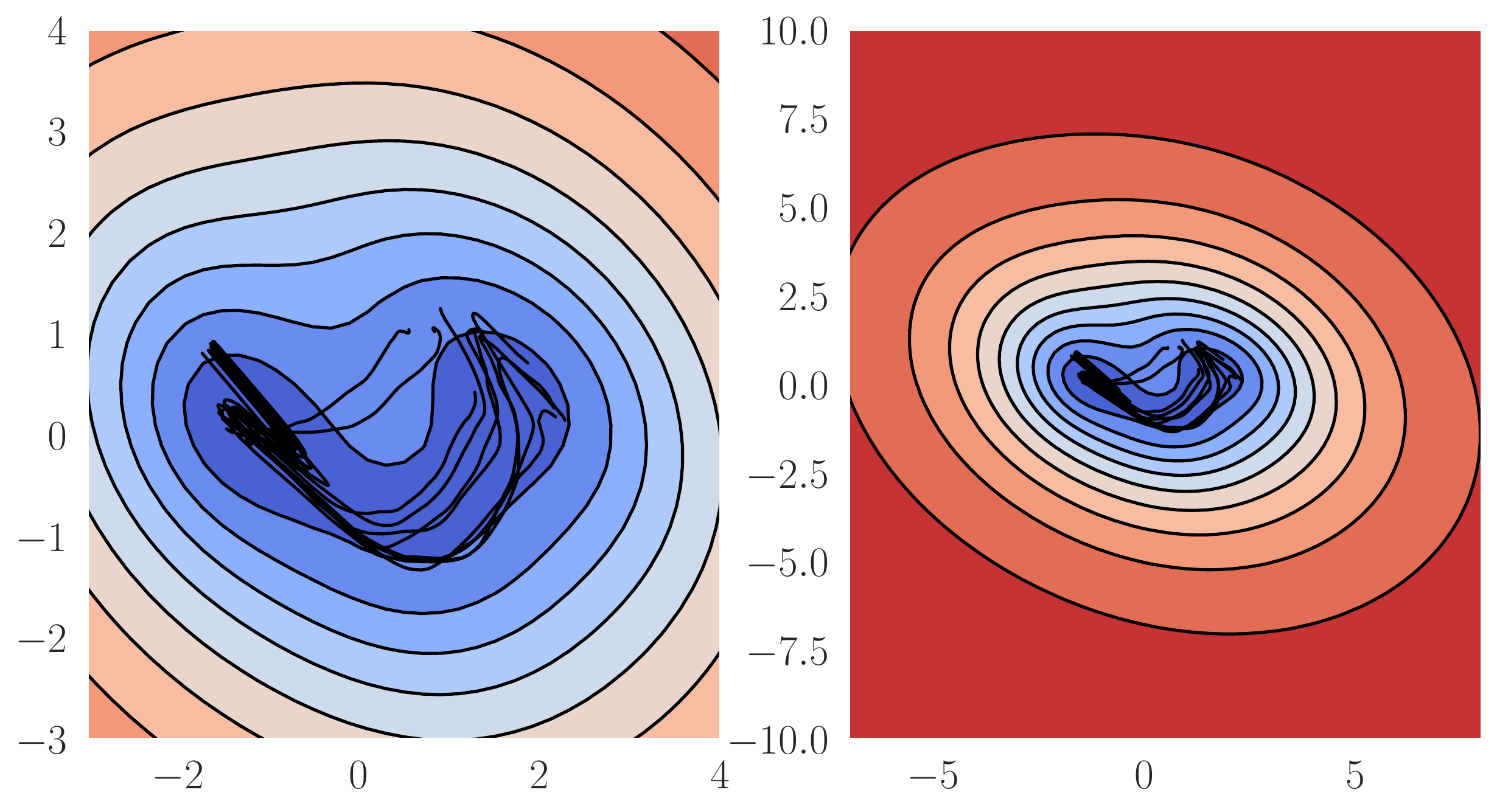}
\end{subfigure}
\\
\centering
\hspace*{6pt}
\begin{subfigure}[t]{0.47\columnwidth}
    \includegraphics[width=\columnwidth]{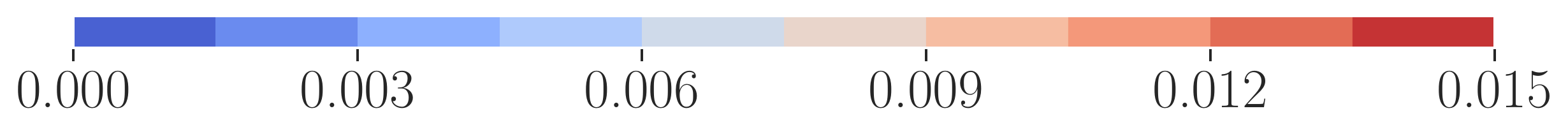}
\end{subfigure}
\hspace*{6pt}
\begin{subfigure}[t]{0.47\columnwidth}
    \includegraphics[width=\columnwidth]{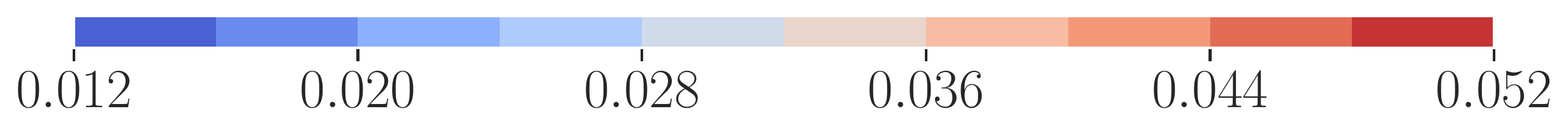}
\end{subfigure}
\caption{
    Predictive variances of non-parametric and parametric behavioral policies on a low dimensional representation (the first two principal components) of a 39-dimensional dexterous hand manipulation state space (see ``door-binary-v0'' in~\Cref{fig:dex_eval}).
    \textbf{Left}: Parametric neural network Gaussian behavioral policy \mbox{$\pi_{\psi}(\cdot \vbar \bs) = \calN(\bmu_{\psi}(\bs), \bsigma^2_{\psi}(\bs))$}.
    \textbf{Right}: Non-parametric Gaussian process posterior behavioral policy \mbox{$\pi_{\mathcal{GP}}(\cdot \vbar \bs, \calD_{0}) = \mathcal{GP}(\bmu_{0}(\bs), \bSigma_{0}(\bs, \bs'))$}.
    Expert trajectories $\calD$ used to train the behavioral policies are shown in black.
    The \gp predictive variance is well-calibrated:
    It is small near the expert trajectories and large in other parts of the state space.
    In contrast, the neural network predictive variance is poorly calibrated:
    It is relatively small on the expert trajectories, and collapses to near zero elsewhere.
    Note the significant difference in scales.
    }
\label{fig:heatmaps}
\vspace{-10pt}
\end{figure}

Reinforcement learning with Kullback-Leibler (KL) regularization is a particularly successful approach for doing so~\citep{boularias2011relative,nair2020accelerating,10.5555/645529.657801,peng2019advantageweighted, Siegel2020Keep,wu2019behavior}.
In KL-regularized reinforcement learning, the standard reinforcement learning objective is augmented by a Kullback-Leibler divergence term that penalizes dissimilarity between the online policy and a behavioral reference policy derived from expert demonstrations.
The resulting regularized objective pulls the agent's online policy towards the behavioral reference policy while also allowing it to improve upon the behavioral reference policy by exploring and interacting with the environment.
Recent advances that leverage explicit or implicit KL-regularized objectives, such as BRAC~\citep{wu2019behavior}, ABM~\citep{Siegel2020Keep}, and AWAC~\citep{nair2020accelerating}, have shown that KL-regularized reinforcement learning from expert demonstrations is able to significantly improve the sample efficiency of online training and reliably solve challenging environments previously unsolved by standard deep reinforcement learning algorithms.

\newpage

\textbf{Contributions.}$~$
In this paper, we show that despite some empirical success, KL-regularized reinforcement learning from expert demonstrations can suffer from previously unrecognized pathologies that lead to instability and sub-optimality in online learning.
To summarize, our core contributions are as follows:\vspace{-5pt}
\begin{itemize}[leftmargin=20pt]
\setlength\itemsep{0.1em}
    \item
    We illustrate empirically that commonly used classes of parametric behavioral policies experience a collapse in predictive variance about states away from the expert demonstrations.
    \item
    We demonstrate theoretically and empirically that KL-regularized reinforcement learning algorithms can suffer from pathological training dynamics in online learning when regularized against behavioral policies that exhibit such a collapse in predictive variance.
    \item
    We show that the pathology can be remedied by \textit{non-parametric} behavioral policies, whose predictive variances are well-calibrated and guaranteed not to collapse about previously unseen states, and that fixing the pathology results in online policies that significantly outperform state-of-the-art approaches on a range of challenging locomotion and dexterous hand manipulation tasks.
\end{itemize}
The left panel of~\Cref{fig:heatmaps} shows an example of the collapse in predictive variance away from the expert trajectories in parametric behavioral policies.
In contrast, the right panel of~\Cref{fig:heatmaps} shows the predictive variance of a non-parametric behavioral policy, which---unlike in the case of the parametric policy---increases off the expert trajectories.
By avoiding the pathology, we obtain a stable and reliable approach to sample-efficient reinforcement learning, applicable to a wide range of reinforcement learning algorithms that leverage KL-regularized objectives.\footnote{Code and visualizations of our results can be found at~\url{https://sites.google.com/view/nppac}.}

\vspace*{-3pt}
\section{Background}
\label{sec:background}

We consider the standard reinforcement learning setting where an agent interacts with a discounted Markov Decision Process (MDP)~\citep{Sutton1998} given by a 5-tuple $(\calS, \calA, p, r, \gamma)$, where $\calS$ and $\calA$ are the state and action spaces, $p(\cdot \vbar \bs_{t}, \ba_{t})$ are the transition dynamics, $r(\bs_{t}, \ba_{t})$ is the reward function, and $\gamma$ is a discount factor.
$\rho_{\pi}(\btau_{t})$ denotes the state--action trajectory distribution from time $t$ induced by a policy $\pi(\cdot \vbar \bs_{t})$.
The discounted return from time step $t$ is given by $R(\btau_{t}) = \sum_{k=t}^\infty \gamma^{k} r(\bs_{k}, \ba_{k})$ for $t \in \mathbb{N}_{0}$.
The standard reinforcement learning objective to be maximized is the expected discounted return $J_{\pi}(\btau_{0}) = \mathbb{E}_{\rho_{\pi}(\btau_{0})}[R(\btau_{0})]$ under the policy trajectory distribution.

\vspace*{-3pt}
\subsection{Improving and Accelerating Online Training via Behavioral Cloning}

We consider settings where we have a set of expert demonstrations \emph{without reward}, \mbox{$\calD_{0} = \{ (\bs_{n}, \ba_{n}) \}_{n=1}^{N} = \{ \bar{\bS}, \bar{\bA} \}$}, which we would like to use to speed up and improve online learning~\citep{brys2015reinforcement,schaal1997learning}.
A standard approach for turning expert trajectories into a policy is behavioral cloning~\citep{bain1995framework, bratko1995behavioural} which involves learning a mapping from states in the expert demonstrations to their corresponding actions, that is, $\pi_{0}: \calS \rightarrow \calA$.
As such, behavioral cloning does not assume or require access to a reward function and only involves learning a mapping from states to action in a supervised fashion.

Since expert demonstrations are costly to obtain and often only available in small number, behavioral cloning alone is typically insufficient for agents to learn good policies in complex environments and has to be complemented by a method that enables the learner to build on the cloned behavior by interacting with the environment.
A particularly successful and popular class of algorithms used for incorporating behavioral policies into online training is KL-regularized reinforcement learning~\citep{DBLP:conf/iclr/GalashovJHTSDCT19,rawlik2012stochastic,schulman2018equivalence,NIPS2006_d806ca13}.

\vspace*{-3pt}
\subsection{KL-Regularized Objectives in Reinforcement Learning}

KL-regularized reinforcement learning modifies the standard reinforcement learning objective by augmenting the return with a negative KL divergence term from the learned policy $\pi$ to a reference policy $\pi_{0}$, given a temperature parameter $\alpha$.
The resulting discounted return from time step $t \in \mathbb{N}_{0}$ is then given by
\begin{align}
\SwapAboveDisplaySkip
\label{eq:return_kl_regularized}
    \tilde{R}(\btau_{t})
    =
    \sum_{k=t}^\infty \gamma^{k} \bigl[r(\bs_{k}, \ba_{k}) - \alpha \DKL{ \pi (\cdot \vbar \bs_{k}) }{ \pi_{0}(\cdot \vbar \bs_{k}) } \bigr]
\end{align}
and the reinforcement learning objective becomes
$
    \tilde{J}_{\pi}(\btau_{0}) = \mathbb{E}_{\rho_{\pi}(\btau_{0})}[\tilde{R}(\btau_{0})].
$
When the reference policy $\pi_{0}$ is given by a uniform distribution, we recover the entropy-regularized reinforcement learning objective used in Soft Actor--Critic (SAC)~\citep{haarnoja2019soft} up to an additive constant.

Under a uniform reference policy $\pi_{0}$, the resulting objective encourages exploration, while also choosing high-reward actions.
In contrast, when $\pi_{0}$ is non-uniform, the agent is discouraged to explore areas of the state space $\calS$ where the variance of $\pi_{0}(\cdot \vbar \bs)$ is low (i.e., more certain) and encouraged to explore areas of the state space where the variance of $\pi_{0}(\cdot \vbar \bs)$ is high.
The KL-regularized reinforcement learning objective can be optimized via policy--gradient and actor--critic algorithms.

\vspace*{-3pt}
\subsection{KL-Regularized Actor--Critic}

An optimal policy $\pi$ that maximizes the expected KL-augmented discounted return $\tilde{J}_{\pi}$ can be learned by directly optimizing the policy gradient $\nabla_{\pi} \tilde{J}_{\pi}$.
However, this policy gradient estimator exhibits high variance, which can lead to unstable learning.
Actor--critic algorithms \citep{degris2012off,konda2000actor,peters2005natural, rosenstein2004supervised} attempt to reduce this variance by making use of the state value function \mbox{$V^{\pi}(\bs_{t}) = \mathbb{E}_{\rho_{\pi}(\btau_{t})}[\tilde{R}(\btau_{t}) \vbar \bs_{t}] $} or the state--action value function \mbox{$Q^{\pi}(\bs_{t}, \ba_{t}) = \mathbb{E}_{\rho_{\pi}(\btau_{t})}[\tilde{R}(\btau_{t}) \vbar \bs_{t}, \ba_{t}] $} to stabilize training.

Given a reference policy $\pi_{0}(\ba_{t} \vbar \bs_{t})$, the state value function can be shown to satisfy the modified Bellman equation
\begin{align*}
    V^\pi(\bs_{t})
    \defines
    \mathbb{E}_{\ba_{t}\sim \pi(\cdot | \bs_{t})} [ Q^\pi(\bs_{t}, \ba_{t}) ] - \alpha \mathbb{D}_{\textrm{KL}}\bigl( \pi(\cdot \vbar \bs_{t}) \,||\, \pi_{0}(\cdot \vbar \bs_{t}) \bigr)
\end{align*}
with a recursively defined $Q$-function
\begin{align*}
    Q^{\pi}(\bs_{t}, \ba_{t})
    \defines
    r(\bs_{t}, \ba_{t}) + \gamma \E_{\bs_{t+1} \sim p(\cdot | \bs_{t}, \ba_{t})} [  V^\pi(\bs_{t+1}) ].
\end{align*}
Instead of directly optimizing the objective function $\tilde{J}_\pi$ via the policy gradient, actor--critic methods alternate between policy evaluation and policy improvement~\citep{degris2012off,haarnoja2019soft}:

\textbf{Policy Evaluation.}$~$ During the policy evaluation step, $Q_{\theta}^{\pi}(\bs, \ba)$, parameterized by parameters $\theta$, is trained by minimizing the Bellman residual
\begin{align}
\begin{split}
\label{eq:objective_qfunction}
   J_{Q}(\theta)
   \defines
   \mathbb{E}_{(\bs_{t}, \ba_{t}) \sim \calD} \Big[ (Q_{\theta } (\bs_{t}, \ba_{t})
   - (r(\bs_{t}, \ba_{t}) + \gamma \mathbb{E}_{\bs_{t+1} \sim p(\cdot | \bs_{t}, \ba_{t})} [V_{\bar{\theta}}(\bs_{t+1})]))^{2}\Big] ,
\end{split}
\end{align}
where $\calD$ is a replay buffer and $\bar{\theta}$ is a stabilizing moving average of parameters.

\textbf{Policy Improvement.}$~$ In the policy improvement step, the policy $\pi_{\phi}$, parameterized by parameters $\phi$, is updated towards the exponential of the KL-augmented $Q$-function,
\begin{align}
\begin{split}
\label{eq:objective_policy}
   J_{\pi}(\phi )
   \defines
   \mathbb{E}_{\bs_{t} \sim \calD} \left[ \alpha \DKL{ \pi_\phi (\cdot \vbar \bs_{t}) }{ \pi_{0}(\cdot \vbar \bs_{t}) } \right]
   - \mathbb{E}_{\bs_{t} \sim \calD} \left[ \mathbb{E}_{\ba_{t} \sim \pi_{\phi}(\cdot | \bs_{t})} \left[ Q_{\theta}(\bs_{t}, \ba_{t}) \right] \right] ,
\end{split}
\end{align}
with states sampled from a replay buffer $\calD$ and actions sampled from the parameterized online policy $\pi_\phi$.
The following sections will focus on the policy improvement objective and how certain types of references policies can lead to pathologies when optimizing $J_{\pi}(\phi )$ with respect to $\phi$.

\section{Identifying the Pathology}
\label{sec:pathology_identify}

In this section, we investigate the effect of KL-regularization on the training dynamics.
To do so, we first consider the properties of the \kld to identify a potential failure mode for KL-regularized reinforcement learning.
Next, we consider parametric Gaussian behavioral reference policies commonly used in practice for continuous control tasks~\citep{haarnoja2019soft,wu2019behavior} and show that for Gaussian behavioral reference policies with small predictive variance, the policy improvement objective suffers from exploding gradients with respect to the policy parameters $\phi$.
We confirm that this failure occurs empirically and demonstrate that it results in slow, unstable, and suboptimal online learning.
Lastly, we show that various regularization techniques used for estimating behavioral policies are unable to prevent this failure and also lead to suboptimal online policies.

\vspace*{-3pt}
\subsection{When Are KL-Regularized Reinforcement Learning Objectives Meaningful?}
\label{subsec:kl_divergence}

We start by considering the properties of the \kld and discuss how these properties can lead to potential failure modes in KL-regularized objectives.
A well-known property of KL-regularized objectives in the variational inference literature is the occurrence of singularities when the support of one distribution is not contained in the support of the other.

To illustrate this problem, we consider the case of Gaussian behavioral and online policies commonly used in practice.
Mathematically, the KL divergence between two full Gaussian distributions is always finite and well-defined.
Hence, we might hope KL-regularized reinforcement learning with Gaussian behavioral and online policies to be unaffected by the failure mode described above.
However, the support of a Gaussian online policy $\pi_{\phi}(\cdot \vbar \bs_{t})$ will not be contained in the support of a behavioral reference policy $\pi_{0}(\cdot \vbar \bs_{t})$ as the predictive variance $\bsigma^2_{0}(\bs_{t})$ tends to zero, and hence
$\DKL{ \pi_\phi (\cdot \vbar \bs_{t}) }{ \pi_{0}(\cdot \vbar \bs_{t}) } \rightarrow \infty$ as $\bsigma^2_{0}(\bs_{t}) \rightarrow 0.$
In other words, as the variance of a behavioral reference policy tends to zero and the behavioral distribution becomes degenerate, the \kld blows up to infinity~\citep{murphy2013probabilistic}.
While in practice, Gaussian behavioral policy would not operate in the limit of zero variance, the functional form of the \kld between (univariate) Gaussians,
\begin{align*}
    \DKL{ \pi_\phi (\cdot \vbar \bs_{t}) }{ \pi_{0}(\cdot \vbar \bs_{t}) }
    \propto
    \log \frac{\bsigma_{0}(\bs_{t})}{\bsigma_{\phi}(\bs_{t})}+\frac{\bsigma_{\phi}^{2}(\bs_{t}) + (\bmu_{\phi}(\bs_{t}) - \bmu_{0}(\bs_{t}))^{2}}{2 \bsigma_{0}^{2}(\bs_{t}) },
\end{align*}
implies a continuous, quadratic increase in the magnitude of the divergence as $\bsigma_{0}(\bs_{t})$ decreases, further exacerbated by a large difference in predictive means, $|\bmu_{\phi}(\bs_{t}) - \bmu_{0}(\bs_{t})|$.

As a result, for Gaussian behavioral reference policies $\pi_{0}(\cdot \vbar \bs_{t})$ that assign very low probability to sets of points in sample space far away from the distribution's mean $\bmu_{0}(\bs_{t})$, computing the \kld can result in divergence values so large to cause numerical instabilities and arithmetic overflow.
Hence, even for a suitably chosen behavioral reference policy class, vanishingly small behavioral reference policy predictive variances can cause the \kld to `blow up' and cause numerical issues at evaluation points far away from states in the expert demonstrations.

One way to address this failure mode may be to lower-bound the output of the variance network (e.g., by adding a small constant bias).
However, placing a floor on the predictive variance of the behavioral reference policy is not sufficient to encourage effective learning.
While it would prevent the \kld from blowing up, it would also lead to poor gradient signals, as well-calibrated predictive variance estimates that \emph{increase} on states far away from the expert trajectories are necessary to keep the KL penalty from pulling the predictive mean of the online policy towards poor behavioral reference policy predictive means on states off the expert trajectories.
Another possible solution could be to use heavy-tailed behavioral reference policies distributions, for example, Laplace distributions, to avoid pathological training dynamics.
However, in~\Cref{appsec:laplace} we show that Laplace behavioral reference policies also suffer from pathological training dynamics, albeit less severely.

In the following sections, we explain how an explosion in \mbox{$ \DKL{ \pi_\phi(\cdot \vbar \bs_{t}) }{ \pi_{0}(\cdot \vbar \bs_{t}) }$} caused by small $\bsigma^2_{0}(\bs_{t})$ affects the gradients of $J_\pi(\phi)$ in KL-regularized RL and discuss of how and why $\bsigma^2_{0}(\bs_{t})$ may tend to zero in practice.

\vspace*{-3pt}
\subsection{Exploding Gradients in KL-Regularized Reinforcement Learning Objectives}
\label{subsec:explodgrad}

To understand how small predictive variances in behavioral reference policies can affect---and possibly destabilize---online training in KL-regularized RL, we consider the contribution of the behavioral reference policy's variance to the gradient of the policy objective in~\Cref{eq:objective_policy}.
Compared to entropy-regularized actor--critic methods (SAC, ~\citet{haarnoja2019soft}), which implicitly regularize against a uniform policy, the gradient estimator $\hat{\nabla}_{\phi} J_{\pi}(\phi )$ in KL-regularized RL gains an extra scaling term $\nabla_{\ba_{t}} \log \pi_{0}(\ba_{t} \vbar \bs_{t})$, the gradient of the prior log-density evaluated actions $\ba_{t} \sim \pi_{\phi}(\cdot \vbar \bs)$:
\begin{proposition}[Exploding Gradients in KL-Regularized RL]
\label{prop:gradients_policy}
Let $\pi_{0}(\cdot \vbar \bs)$ be a Gaussian behavioral reference policy with mean $\bmu_{0}(\bs_{t})$ and variance $\bsigma^{2}_{0}(\bs_{t})$, and let $\pi_{\phi}(\cdot \vbar \bs)$ be an online policy with reparameterization $\ba_{t} = f_{\phi} (\epsilon_{t} ; \bs_{t})$ and random vector $\epsilon_{t}$.
The gradient of the policy loss with respect to the online policy's parameters $\phi$ is then given by
\begin{align}
\begin{split}
    \hat{\nabla}_{\phi} J_{\pi}(\phi)
    &=
    \big( \alpha \nabla_{\ba_{t}} \log  \pi_{\phi}(\ba_{t} \mid \bs_{t}) - \alpha \textcolor{black}{ \nabla_{\ba_{t}} \log  \pi_{0}(\ba_{t} \mid \bs_{t})}
    \\
    &\qquad
    - \nabla_{\ba_{t}} Q(\bs_{t}, \ba_{t}) \big) \nabla_{\phi} f_{\phi} (\epsilon_{t} ; \bs_{t}) + \alpha \nabla_{\phi} \log \pi_{\phi}(\ba_{t} \mid \bs_{t})
\end{split}
\end{align}%
with
\mbox{$
\textcolor{black}{\nabla_{\ba_{t}} \log \pi_{0}(\ba_{t} \vbar \bs_{t})} = -\frac{\ba_{t}-\bmu_{0}(\bs_{t})}{\bsigma^{2}_{0}(\bs_{t})}.
$}
For fixed \mbox{$|\ba_{t}-\bmu_{0}(\bs_{t})|$}, \mbox{$\nabla_{\ba_{t}} \log \pi_{0}(\ba_{t} | \bs_{t})$} grows as \mbox{$\mathcal{O}(\bsigma^{-2}_{0}(\bs_{t}))$}; thus,
\begin{align*}
    \vbar \hat{\nabla}_{\phi} J_{\pi}(\phi) \vbar \rightarrow \infty \quad \text{as} \quad \bsigma^2_{0}(\bs_{t}) \rightarrow 0 \quad \text{whenever} \quad \nabla_{\phi} f_{\phi} (\epsilon_{t} ; \bs_{t}) \neq 0.
\end{align*}
\end{proposition}
\begin{proof}
    See~\Cref{appsec:grad_calcs}.
\end{proof}
\vspace*{-5pt}
This result formalizes the intuition presented in~\Cref{subsec:kl_divergence} that a behavioral reference policy with a sufficiently small predictive variance may cause KL-regularized reinforcement learning to suffer from pathological training dynamics in gradient-based optimization.
The smaller the behavioral reference policy's predictive variance, the more sensitive the policy objective's gradients will be to differences in the means of the online and behavioral reference policies.
As a result, for behavioral reference policies with small predictive variance, the \kld will heavily penalize online policies whose predictive means diverge from the predictive means of the behavioral policy---even in regions of the state space away from the expert trajectory where the behavioral policy's mean prediction is poor.

\vspace*{-3pt}
\subsection{Predictive Uncertainty Collapse Under Parametric Policies}
\label{subsec:preduc}

The most commonly used method for estimating behavioral policies is maximum likelihood estimation (MLE)~\citep{Siegel2020Keep,wu2019behavior}, where we seek $\pi_{0} \defines \pi_{\psi^\star}$ with \mbox{$\psi^\star
    \defines
    \argmax_{\psi} \left\{ \mathbb{E}_{(\bs, \ba) \sim \calD_{0}} [\log \pi_{\psi} (\ba \vbar \bs)] \right\} $}
for a parametric behavioral policy $\pi_\psi$.
In practice, $\pi_\psi$ is often assumed to be Gaussian, \mbox{$\pi_\psi(\cdot \vbar \bs) = \calN(\bmu_{\psi}(\bs), \bsigma_{\psi}^2(\bs))$}, with $\bmu_{\psi}(\bs)$ and $\bsigma_{\psi}^2(\bs)$ parameterized by a neural network.

While maximizing the likelihood of the expert trajectories under the behavioral policy is a sensible choice for behavioral cloning, the limited capacity of the neural network parameterization can produce unwanted behaviors in the resulting policy.
The maximum likelihood objective ensures that the behavioral policy's predictive mean reflects the expert's actions and the predictive variance the (aleatoric) uncertainty inherent in the expert trajectories.

However, the maximum likelihood objective encourages parametric policies to use their model capacity toward fitting the expert demonstrations and reflecting the aleatoric uncertainty in the data.
As a result, for states off the expert trajectories, the policy can become degenerate and collapse to point predictions instead of providing meaningful predictive variance estimates that reflect that the behavioral policy ought to be highly uncertain about its predictions in previously unseen regions of the state space.
Similar behaviors are well-known in parametric probabilistic models and well-documented in the approximate Bayesian inference literature~\citep{quinonero2005unifying,rudner2021fsvi}.

\begin{wrapfigure}{R}{0.5\textwidth}
    \centering
    \vspace*{-10pt}
    \includegraphics[width=0.5\textwidth]{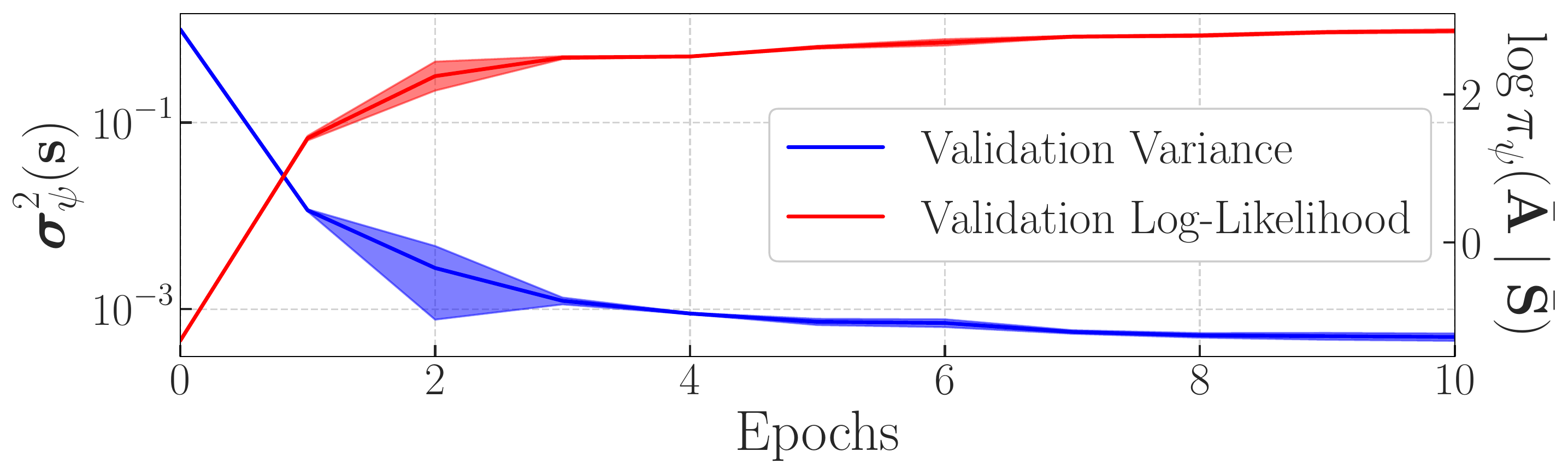}
    \vspace*{-19pt}
    \caption{
        Collapse in the predictive variance (in \textcolor{blue}{\textbf{blue}}) of a Gaussian behavioral policy parameterized by a neural network when training via maximum likelihood estimation.
        Lines and shaded regions denote means and standard deviations over five random seeds, respectively.
    }
    \vspace*{-10pt}
    \label{fig:nn_var_descent}
\end{wrapfigure}

\Cref{fig:heatmaps} demonstrates the collapse in predictive variance under maximum likelihood estimation in a low-dimensional representation of the ``door-binary-v0'' dexterous hand manipulation environment.
It shows that while the predictive variance is small close to the expert trajectories (depicted as black lines), it rapidly decreases further away from them.
Examples of variance collapse in other environments are presented in~\Cref{appsec:further_env_vis}.
\Cref{fig:nn_var_descent} shows that the predictive variance off the expert trajectories consistently decreases during training.
As shown in~\Cref{prop:gradients_policy}, such a collapse in predictive variance can result in pathological training dynamics in KL-regularized online learning---steering the online policy towards suboptimal trajectories in regions of the state space far away from the expert demonstrations and deteriorating performance.

\textbf{Effect of regularization on uncertainty collapse.}$~$
To prevent a collapse in the behavioral policy's predictive variance, prior work proposed adding entropy or Tikhonov regularization to the MLE objective~\citep{wu2019behavior}.
However, doing so does not succeed in preventing a collapse in predictive variance off the expert demonstration trajectories, as we show in~\Cref{appsec:entregpredvar}.
Deep ensembles~\citep{NIPS2017_7219}, whose predictive mean and variance are computed from the predictive means and variances of multiple Gaussian neural networks, are a widely used method for uncertainty quantification in regression settings.
However, model ensembling can be costly and unreliable, as it requires training multiple neural networks from scratch and does not guarantee well-calibrated uncertainty estimates~\citep{rudner2021fsvi,vanamersfoort2021improving}.
We provide visualizations in~\Cref{appsec:policy_ensemble} which show that ensembling multiple neural network policies does not fully prevent a collapse in predictive variance.

\begin{figure*}[t!]
\centering
\includegraphics[width=0.95\textwidth]{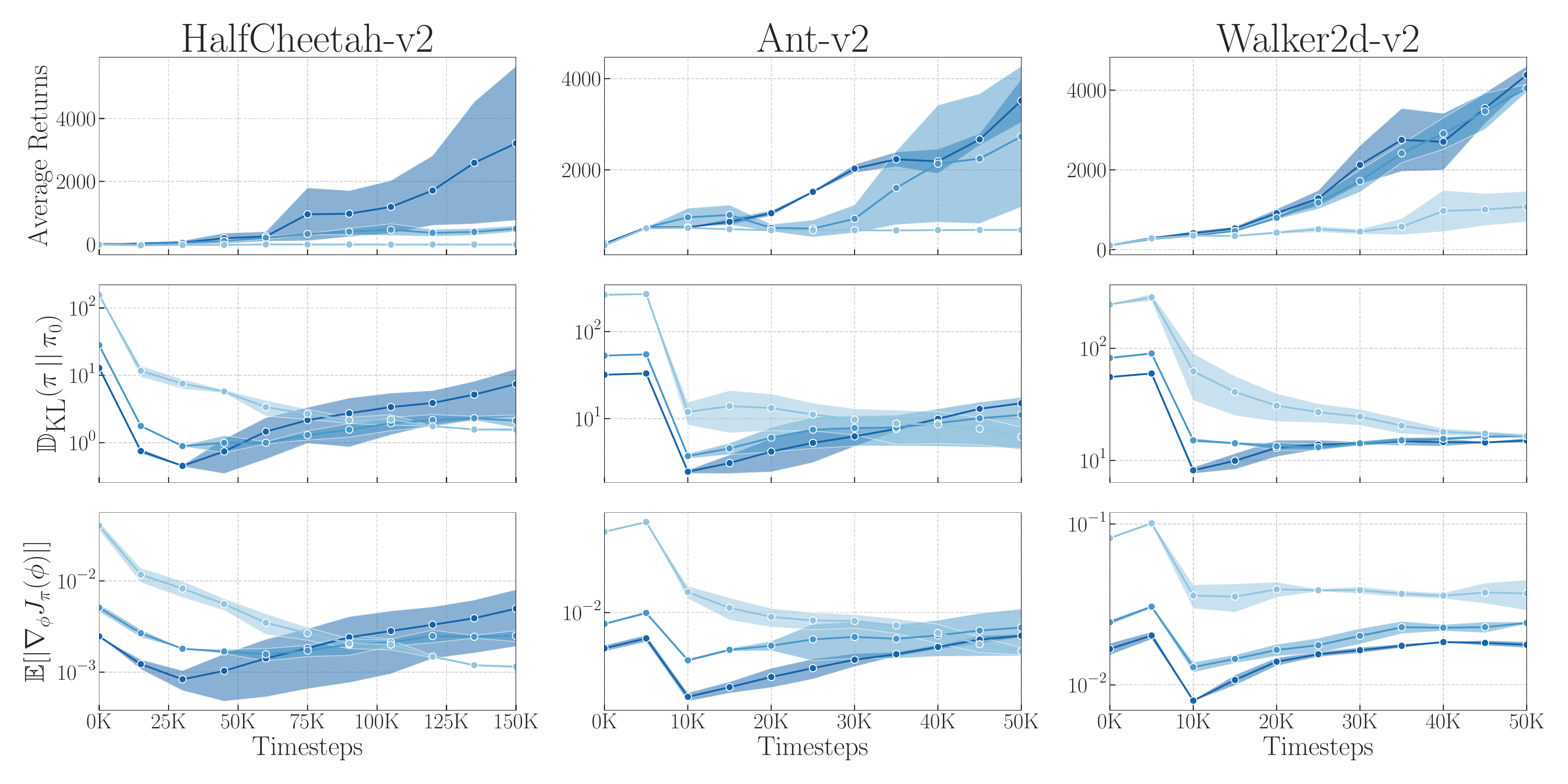}
\includegraphics[trim=0 100 0 100, clip, width=0.6\textwidth]{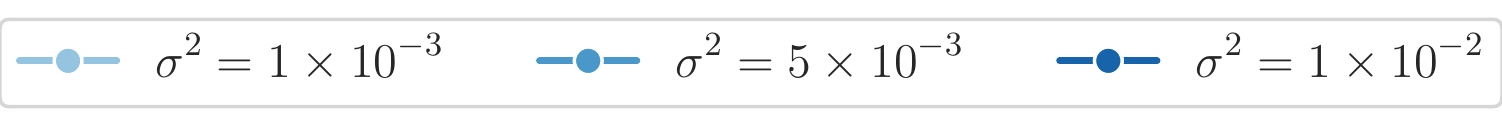}
    \caption{
        Ablation study showing the effect of predictive variance collapse on the performance of KL-regularized RL on MuJoCo environments.
        The plots show the average return of the learned policy, the magnitude of the KL penalty, and the magnitude of the average absolute gradients of the policy loss during online training.
        The lighter the shading, the lower the behavioral policy's predictive variance.
    }
    \label{fig:mujoco_gp_mean_fixed_var_abl}
    \vspace*{-10pt}
\end{figure*}

\vspace*{-3pt}
\subsection{Empirical Confirmation of Uncertainty Collapse}
\label{sec:empirical_evaluation}

To confirm~\Cref{prop:gradients_policy} empirically and assess the effect of the collapse in predictive variance on the performance of KL-regularized RL, we perform an ablation study where we fix the predictive mean function of a behavioral policy to a mean function that attains ~60\% of the optimal performance and vary the magnitude of the policy's predictive variance.
Specifically, we set the behavioral policy's predictive variance to different constant values in the set \mbox{$\{ 1\times 10^{-3}, 5\times 10^{-3}, 1\times 10^{-2} \}$} (following a similar implementation in~\citet{nair2020accelerating}).\footnote{We attempted to use  smaller values, but the gradients grew too large and caused arithmetic overflow.}
The results of this experiment are shown in~\Cref{fig:mujoco_gp_mean_fixed_var_abl}, which shows the average returns, the \kld, and the average absolute gradients of the policy loss over training.
The plots confirm that as the predictive variance of the offline behavioral policy tends to zero, the KL terms and average policy gradient magnitude explode as implied by~\Cref{prop:gradients_policy}, leading to unstable training and a collapse or dampening in average returns.

In other words, even for behavioral policies with accurate predictive means, smaller predictive variances slow down or even entirely prevent learning good behavioral policies.
This observation confirms that the pathology identified in~\Cref{prop:gradients_policy} occurs in practice and that it can have a significant impact on KL-regularized RL from expert demonstrations, calling into question the usefulness of KL regularization as a means for accelerating and improving online training.
In~\Cref{appsec:qfngrad}, we show that an analogous relationship exists for the gradients of the $Q$-function loss.

\clearpage

\vspace*{-3pt}
\section{Fixing the Pathology}
\label{sec:pathology_fix}

In order to address the collapse in predictive uncertainty for behavioral policies parameterized by a neural network trained via MLE, we specify a \emph{non-parametric} behavioral policy whose predictive variance is \emph{guaranteed} not to collapse about previously unseen states.
Noting that KL-regularized RL with a behavioral policy can be viewed as approximate Bayesian inference with an empirical prior policy~\citep{haarnoja2019soft,levine2018reinforcement,rudner2021odrl}, we  propose \textit{Non-Parametric Prior Actor--Critic} (\textsc{n-ppac}), an off-policy temporal difference algorithm for improved, accelerated, and stable online learning with behavioral policies.

\vspace*{-3pt}
\subsection{Non-Parametric Gaussian Processes Behavioral Policies}
\label{subsec:gps}

Gaussian processes (\gps)~\citep{10.5555/1162254} are models over functions defined by a mean $m(\cdot )$ and covariance function $k(\cdot, \cdot )$.
When defined in terms of a non-parametric covariance function, that is, a covariance function constructed from infinitely many basis functions, we obtain a non-degenerate \gp, which has sufficient capacity to prevent a collapse in predictive uncertainty away from the training data.
Unlike parametric models, whose capacity is limited by their parameterization, a non-parametric model's capacity \emph{increases} with the amount of training data.

Considering a non-parametric \gp behavioral policy, $\pi_{0}(\cdot \vbar \bs)$, with
\begin{align}
    \bA \vbar \bs \sim \pi_{0}(\cdot \vbar \bs)
    =
    \mathcal{GP}\bigl(m(\bs), k(\bs, \bs')\bigr) ,
\end{align}
we can obtain a \emph{non-degenerate} posterior distribution over actions conditioned on the offline data \mbox{$\calD_{0} = \{ \bar{\bS}, \bar{\bA} \}$} with actions sampled according to the 
\begin{align}
    \bA \vbar \bs, \calD_{0} \sim \pi_{0}(\cdot \vbar \bs, \calD_{0})
    =
    \mathcal{GP}\bigl(\bmu_{0}(\bs), \bSigma_{0}(\bs, \bs')\bigr),
\end{align}\\[-18pt]
with
\begin{align*}
    \mu(\bs)
    \hspace*{-2pt}=\hspace*{-2pt}
    m(\bs) + k(\bs, \bar{\bS}) k(\bar{\bS}, \bar{\bS})^{-1} (\bar{\bA} - m(\bar{\bA}))
    ~~\,\textrm{and}~~\,
    \Sigma(\bs, \bs')
    \hspace*{-2pt}=\hspace*{-2pt}
    k(\bs, \bs') + k(\bs, \bar{\bS}) k(\bar{\bS}, \bar{\bS})^{-1} k(\bar{\bS}, \bs').
\end{align*}
To obtain this posterior distribution, we perform exact Bayesian inference, which naively scales as $\calO(N^3)$ in the number of training points $N$, but \citet{NEURIPS2019_01ce8496} show that exact inference in \gp regression can be scaled to $N>1,000,000$.
Since expert demonstrations usually contain less than 100k datapoints, non-parametric \gp behavioral policies are applicable to a wide array of real-world tasks.
For an empirical evaluation of the time complexity of using a \gp prior, see~\Cref{sec:complexity}.

\Cref{fig:heatmaps} confirms that the non-parametric \gp's predictive variance is well-calibrated:
It is small in magnitude in regions of the state space near the expert trajectories and large in magnitude in other regions of the state space.
While actor--critic algorithms like SAC implicitly use a uniform prior to explore the state space, using a behavioral policy with a well-calibrated predictive variance has the benefit that in regions of the state space close to the expert demonstrations the online policy learns to match the expert, while elsewhere the predictive variance increases and encourages exploration.

\textbf{Algorithmic details.}$~$
In our experiments, we use a KL-regularized objective with a standard actor--critic implementation and Double DQN~\citep{NIPS2010_3964}.
Pseudocode is provided in~(\Cref{appsec:algorithm}).

\begin{figure*}[t!]
\centering
 \begin{subfigure}[b]{0.96\textwidth}
    \includegraphics[width=\textwidth]{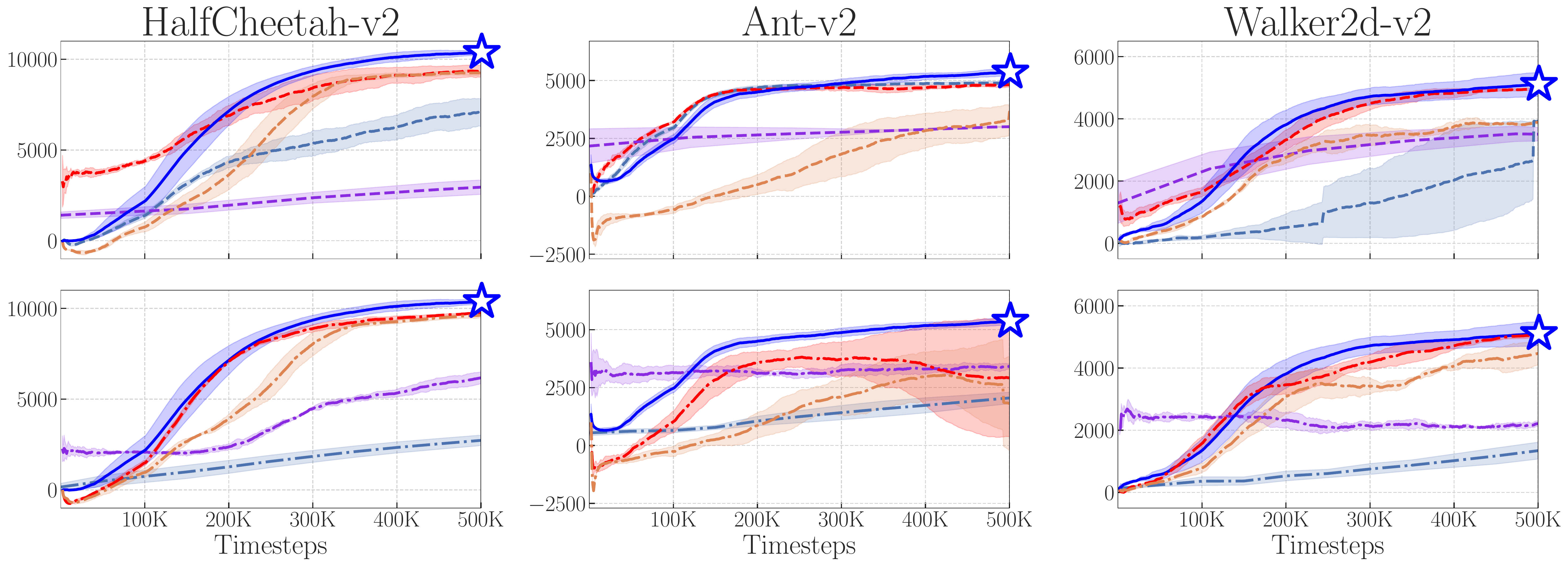}
    \label{fig:mujoco_combined}
  \end{subfigure}
  \begin{subfigure}[b]{\textwidth}
  \vspace*{-10pt}
    \includegraphics[trim=0 95 0 95, clip, width=\columnwidth]{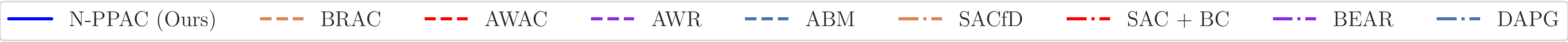}
    \label{fig:mujoco_legend}
  \end{subfigure}
  \vspace{-20pt}
      \caption{Comparison of \textsc{n-ppac} (ours) vs. previous baselines on standard MuJoCo benchmark tasks.
      \textbf{Top}: KL-based methods (dashed lines), \textbf{Bottom}: Non-KL-based methods (dash-dotted lines).
      Both top and bottom plots include \textsc{n-ppac} (\textcolor{blue}{\textbf{blue}}).
      BRAC uses the same actor--critic algorithm as \textsc{n-ppac}, but uses a parametric behavioral policy, and results in slower learning and worse final performance.
      }
    \label{fig:mujcoo_eval}
    \vspace*{-10pt}
\end{figure*}

\vspace*{-3pt}
\section{Empirical Evaluation}
\label{sec:experiments}

We carry out a comparative empirical evaluation of our proposed approach vis-\`a-vis related methods that integrate offline data into online training.
We provide a detailed description of the algorithms we compare against in~\Cref{appsec:prior_comparisons}.
We perform experiments on the MuJoCo benchmark suite and the substantially more challenging dexterous hand manipulation suite with sparse rewards.

We show that KL-regularized RL with a non-parametric behavioral reference policy can rapidly learn to solve difficult high-dimensional continuous control problems given only a small set of expert demonstrations and (often significantly) outperforms state-of-the-art methods, including ones that use offline reward information---which our approach does not require.
Furthermore, we demonstrate that the \gp behavioral policy's predictive variance is crucial for KL-regularized objectives to learn good online policies from expert demonstrations.
Finally, we perform ablation studies that illustrate that non-parametric \gp behavioral reference policies also outperform parametric behavioral reference policies with improved uncertainty quantification, such as deep ensembles and Bayesian neural networks (BNNs) with Monte Carlo dropout, and that the difference between non-parametric and parametric models is exacerbated the fewer expert demonstrations are available.
We use the expert data from~\citet{nair2020accelerating}, every experiment uses six random seeds, and we use a fixed KL-temperature for each environment class.
For further implementation details, see~\Cref{appsec:hyperparam}.

\vspace*{-3pt}
\subsection{Environments}

\textbf{MuJoCo locomotion tasks.}$~$
We evaluate \textsc{n-ppac} on three representative tasks: ``Ant-v2'', ``HalfCheetah-v2'', and ``Walker2d-v2''.
For each task, we use 15 demonstration trajectories collected by a pre-trained expert, each containing 1,000 steps.
The behavioral policy is specified as the posterior distribution of a \gp with a squared exponential kernel, which is well-suited for modeling smooth functions.

\textbf{Dexterous hand manipulation tasks.}$~$
Real-world robot learning is a setting where human demonstration data is readily available, and many deep RL approaches fail to learn efficiently.
We study this setting in a suite of challenging dexterous manipulation tasks~\citep{rajeswaran2018learning} using a 28-DoF five-fingered simulated ADROIT hand.
The tasks simulate challenges common to real-world settings with high-dimensional action spaces, complex physics, and a large number of intermittent contact forces.
We consider two tasks in particular: in-hand rotation of a pen to match a target and opening a door by unlatching and pulling a handle.
We use binary rewards for task completion, which is significantly more challenging than the original setting considered in~\citet{rajeswaran2018learning}.
25 expert demonstrations were provided for each task, each consisting of 200 environment steps which are not fully optimal but do successfully solve the task.
The behavioral policy is specified as the posterior distribution of a \gp with a Mat\'ern kernel, which is more suitable for modeling non-smooth data.

\vspace*{-3pt}
\subsection{Results}

On MuJoCo environments, KL-regularized RL with a non-parametric behavioral policy consistently outperforms all related methods across all three tasks, successfully accelerating learning from offline data, as shown in~\Cref{fig:mujcoo_eval}.
Most notably, it outperforms methods such as AWAC~\citep{nair2020accelerating}---the previous state-of-the-art---which attempts to eschew the problem of learning behavioral policies but instead uses an implicit constraint.
Our approach, \textsc{n-ppac}, exhibits an increase in stability and higher returns compared to comparable methods such as ABM and BRAC that explicitly regularize the online policy against a parametric behavioral policy and plateau at suboptimal performance levels as they are being forced to copy poor actions from the behavioral policy away from the expert data.
In contrast, using a non-parametric behavioral policy allows us to avoid such undesirable behavior.

On dexterous hand manipulation environments, KL-regularized RL with a non-parametric behavioral policy performs on par or outperforms all related methods on both tasks, as shown in~\Cref{fig:dex_eval}.
Most notably, on the door opening task, it achieves a stable success rate of 90\% within only 100,000 environment interactions
For comparison, AWAC requires 4$\times$ as many environment interactions to achieve the same performance and is significantly less stable, while most other methods fail to learn any meaningful behaviors.

\textbf{Alternative divergence metrics underperform KL-regularization.}$~$
KL-regularized RL with a non-parametric behavioral policy consistently outperforms methods that use alternative divergence metrics, as shown in the bottom plots of Figures~\ref{fig:mujcoo_eval}~and~\ref{fig:dex_eval}.
\begin{figure*}[t!]
\centering
   \begin{subfigure}[b]{0.68\textwidth}
   \begin{subfigure}[t]{\textwidth}
    \includegraphics[width=\textwidth]{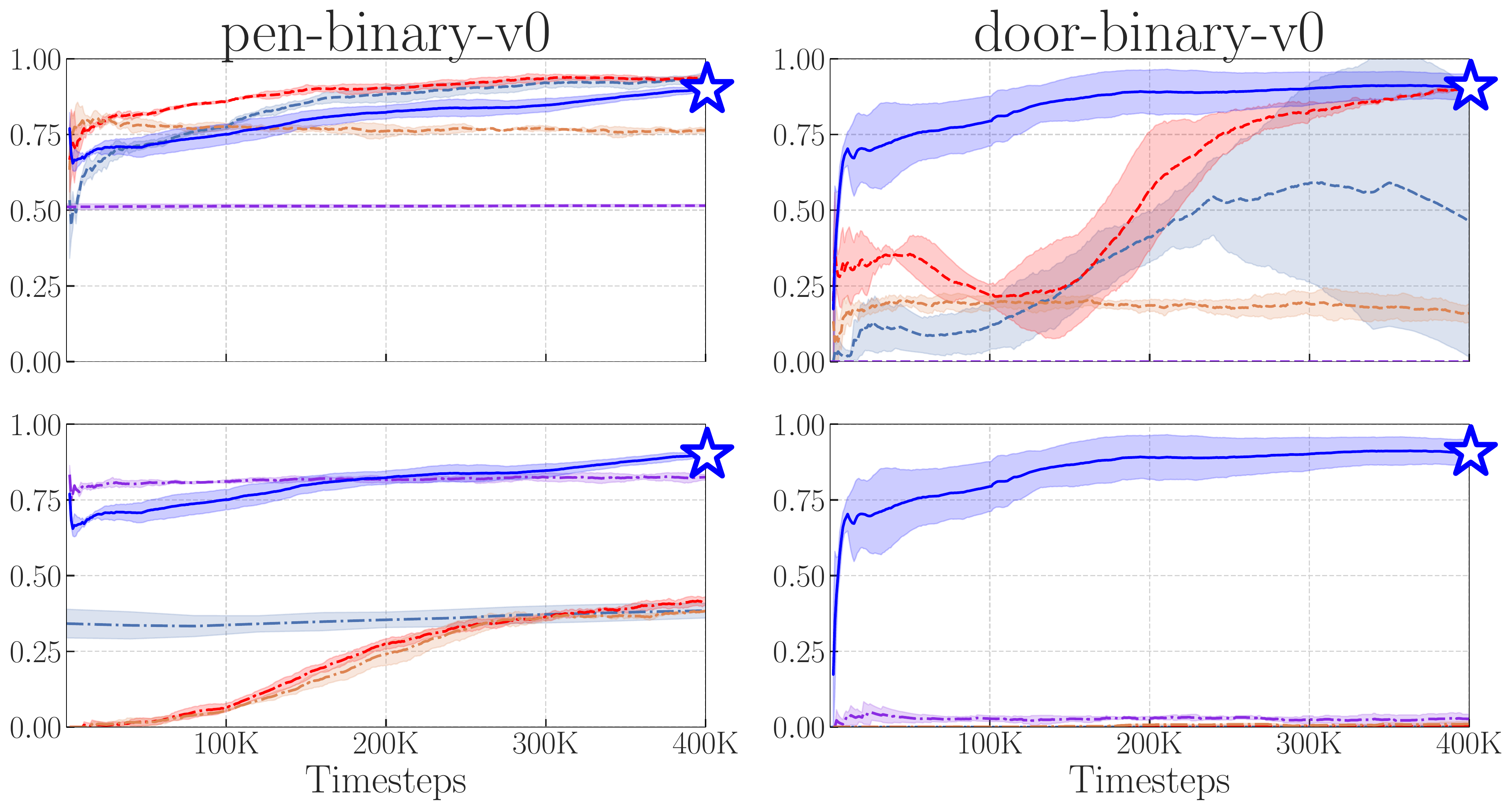}
    \label{fig:hand_combined}
   \end{subfigure}%
  \vspace*{-2pt}
    \centering
  \begin{subfigure}[b]{0.95\textwidth}
    \vspace*{-8pt}\includegraphics[trim=0 95 0 95, clip, width=\columnwidth]{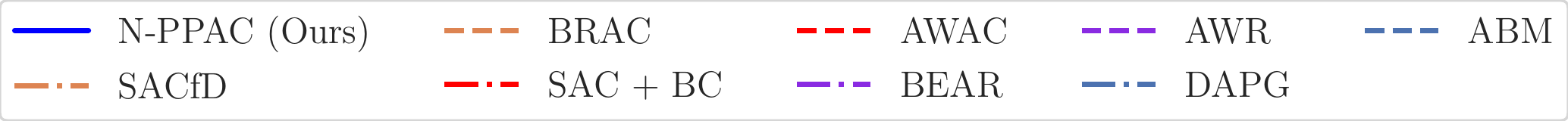}
    \label{fig:dex_legend}
  \end{subfigure}
  \end{subfigure}
  \begin{subfigure}[t]{0.26\textwidth}
  \vspace*{-172pt}
  \centering
  \hspace*{16pt}\begin{subfigure}[t]{\textwidth}
    \includegraphics[width=85pt, trim=2 2 2 2]{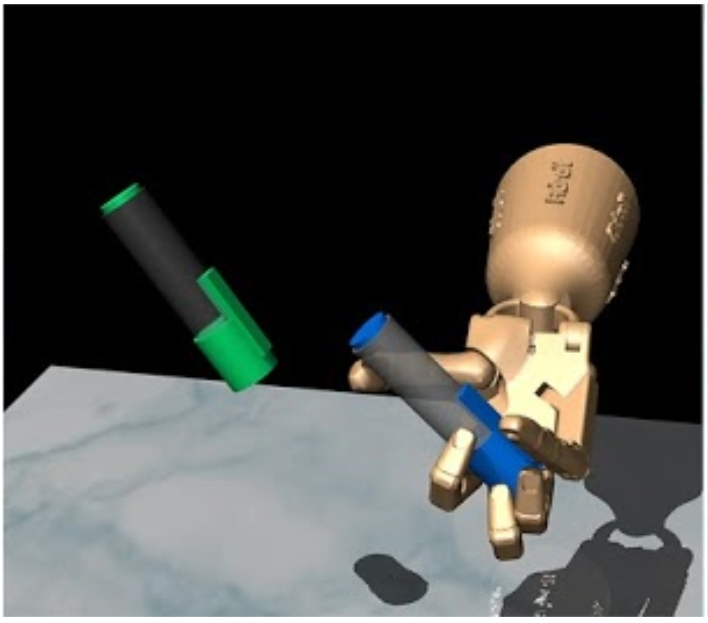}
  \end{subfigure}
  \\[5pt]
  \centering
  \hspace*{16pt}\begin{subfigure}[b]{\textwidth}
    \includegraphics[width=85pt, trim=0 2 0 0]{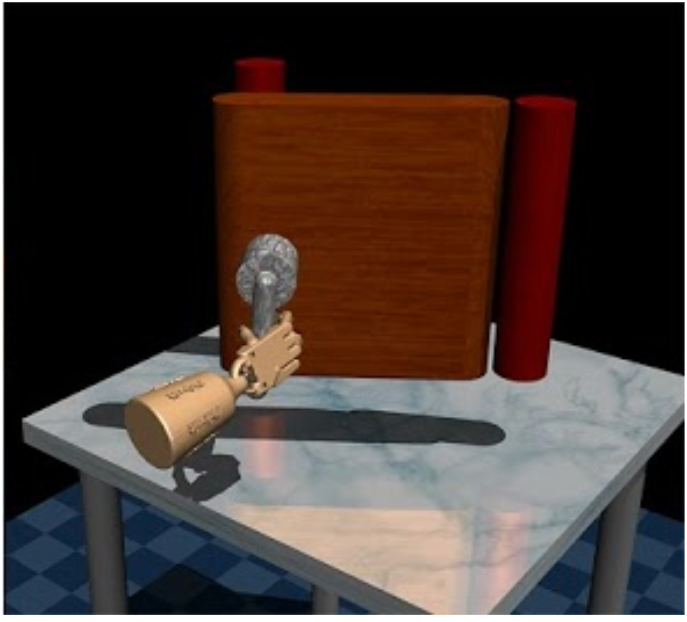}
      \vspace*{-10pt}
  \end{subfigure}
  \end{subfigure}
  \vspace{-1.1\baselineskip}
  \caption{
  \textbf{Left \& Center}: Comparison of \textsc{n-ppac} (ours) vs. previous baselines on dexterous hand manipulation tasks. \textbf{Top}: KL-based methods (dashes), \textbf{Bottom}: Non-KL-based methods (dots and dashes).
  Both top and bottom plots include \textsc{n-ppac} (\textcolor{blue}{\textbf{blue}}).
  \textbf{Right}:
  The pen-binary-v0 (top) and door-binary-v0 (bottom) environments.
  }
  \label{fig:dex_eval}
  \vspace*{-10pt}
\end{figure*}

\vspace*{-3pt}
\subsection{Can the Pathology Be Fixed by Improved Parametric Uncertainty Quantification?}

\begin{wrapfigure}{R}{0.43\textwidth}
\centering
\vspace*{-17pt}
  \centering
  \begin{subfigure}[t]{0.9\linewidth}
    \includegraphics[trim=8 0 10 0, clip, width=\linewidth]{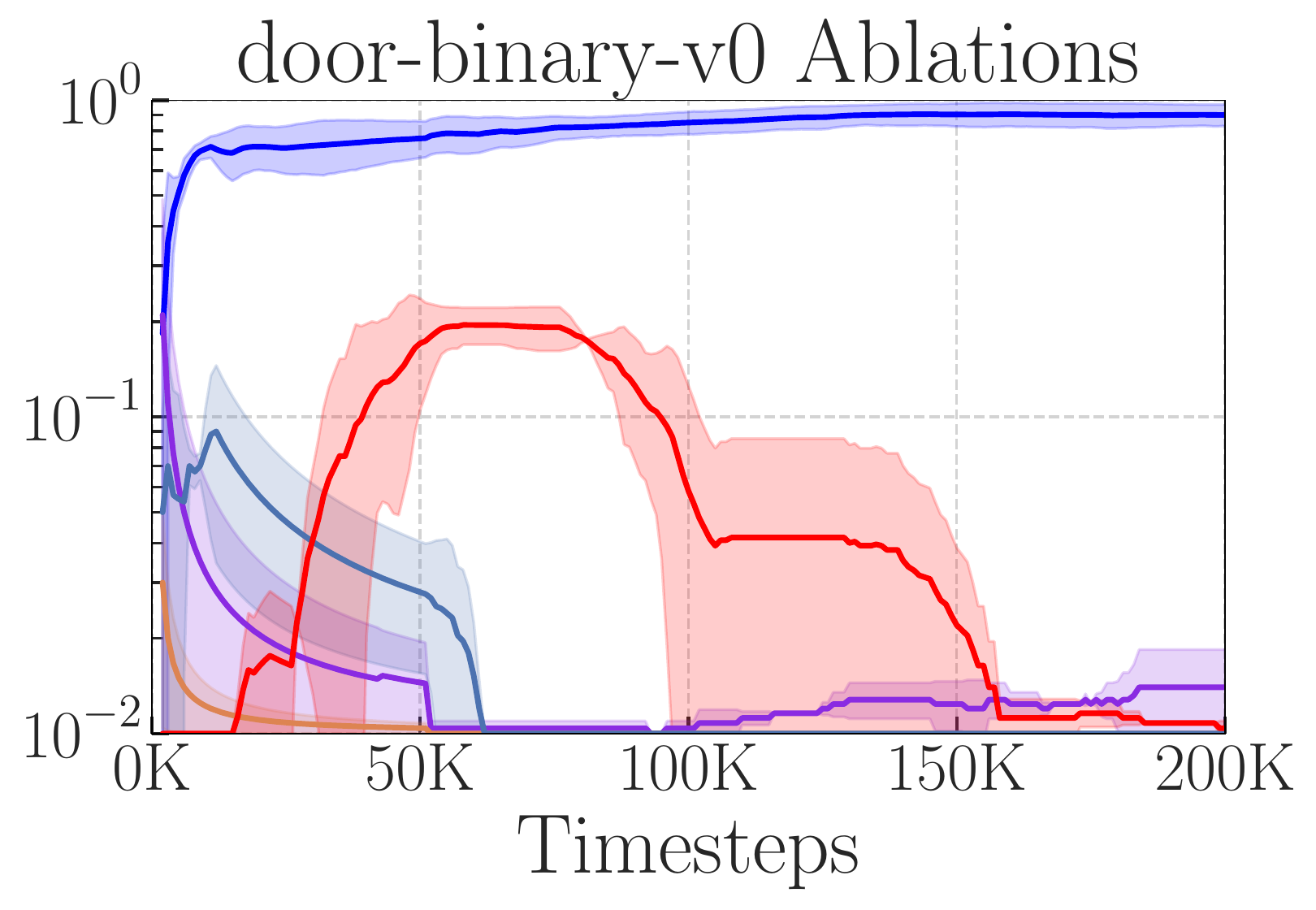}
  \end{subfigure}
  \\
  \centering
  \begin{subfigure}[b]{0.9\linewidth}
    \includegraphics[trim=30 50 30 50, clip, width=\textwidth]{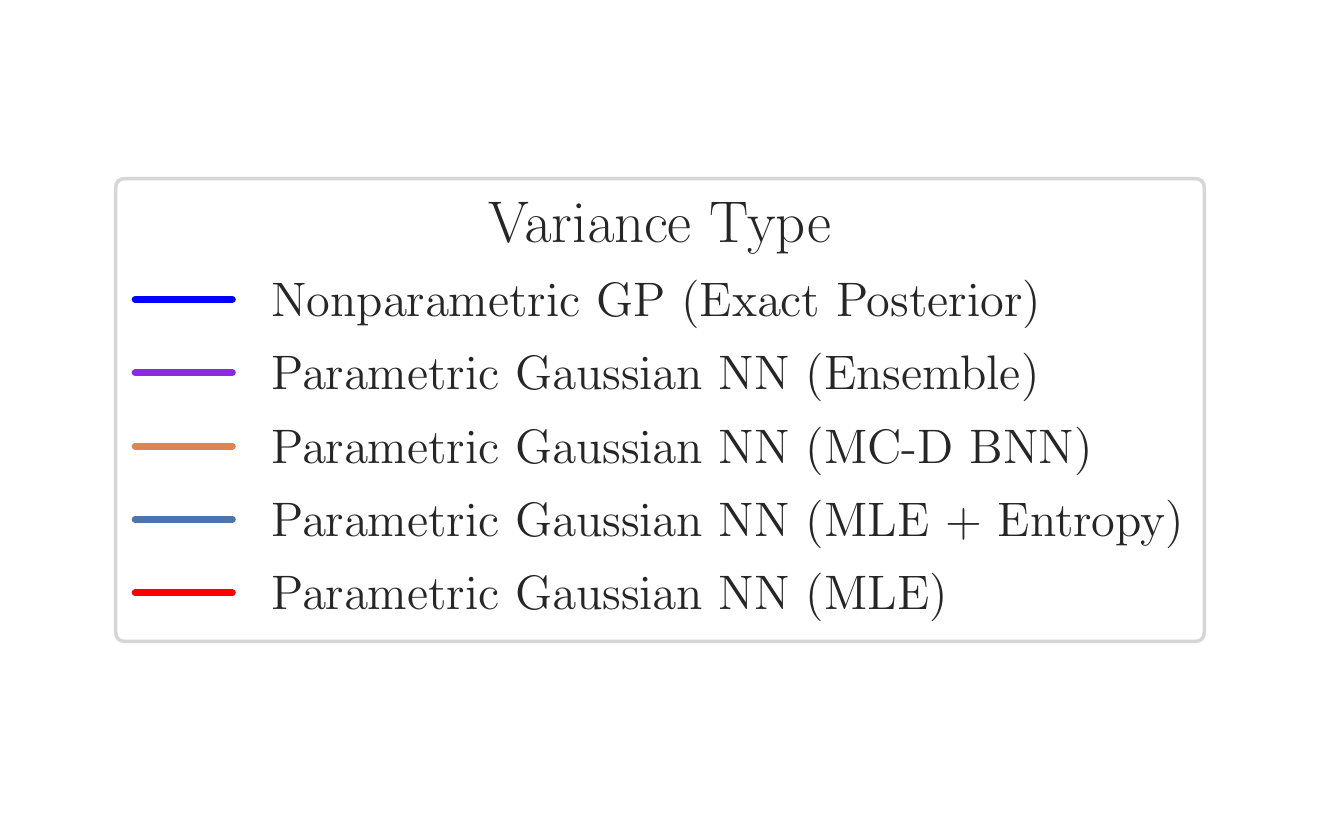}
  \end{subfigure}
  \vspace*{-3pt}
    \caption{
        Post-online training success rates with different behavioral policy variance functions.
    }
    \label{fig:dex_varabl}
  \vspace*{-15pt}
\end{wrapfigure}

To assess whether the success of non-parametric behavioral reference policies is due to their predictive variance estimates---as suggested by~\Cref{prop:gradients_policy}---or due to better generalization from their predictive means, we perform an ablation study on the predictive variance of the behavioral policy.
To isolate the effect of the predictive variance on optimization, we perform online training using behavioral policies with different predictive variance functions (parametric and non-parametric) and identical mean functions, which we set to be the predictive mean of the \gp posterior (which achieves a success rate of \char`~80\%).
If the pathology identified in~\Cref{prop:gradients_policy} can be remedied by commonly used parametric uncertainty quantification methods, we would expect the parametric and non-parametric behavioral policy variance functions to result in similar online policy success rates.
We consider the challenging ``door-binary-v0'' environment for this ablation study.

\textbf{Parametric uncertainty quantification is insufficient.}$~$
\Cref{fig:dex_eval} shows that parametric variance functions result in online policies that only achieve success rates of up to 20\% and eventually deteriorate, whereas the non-parametric variance yields an online policy that achieves a success rate of nearly $100\%$.
This finding shows that commonly used uncertainty quantification methods, such as deep ensembles or BNNs with Monte Carlo dropout, do not generate sufficiently well-calibrated uncertainty estimates to remedy the pathology, and better methods may be needed~\citep{farquhar2020radial,rudner2021fsvi,rudner2021cfsvi}.

\textbf{Lower-bounding the predictive variance does not remedy the pathology.}$~$
The predictive variance of all MLE-based and ensemble behavioral reference policies in all experiments are bounded away from zero at a minimum value of $\approx 10^{-2}$.
Hence, setting a floor on the variance is not sufficient to prevent pathological training dynamics.
This result further demonstrates the importance of accurate predictive variance estimation in allowing the online policy to match expert actions in regions of the state space with low behavioral policy predictive variance and explore elsewhere.

\vspace*{-3pt}
\subsection{Can a Single Expert Demonstration Be Sufficient to Accelerate Online Training?}
\label{sec:model_comp_reduced}

\begin{wrapfigure}{R}{0.6\textwidth}
\vspace*{-10pt}
\centering
\hspace*{-10pt}\begin{subfigure}[t]{0.52\linewidth}
\includegraphics[width=\linewidth]{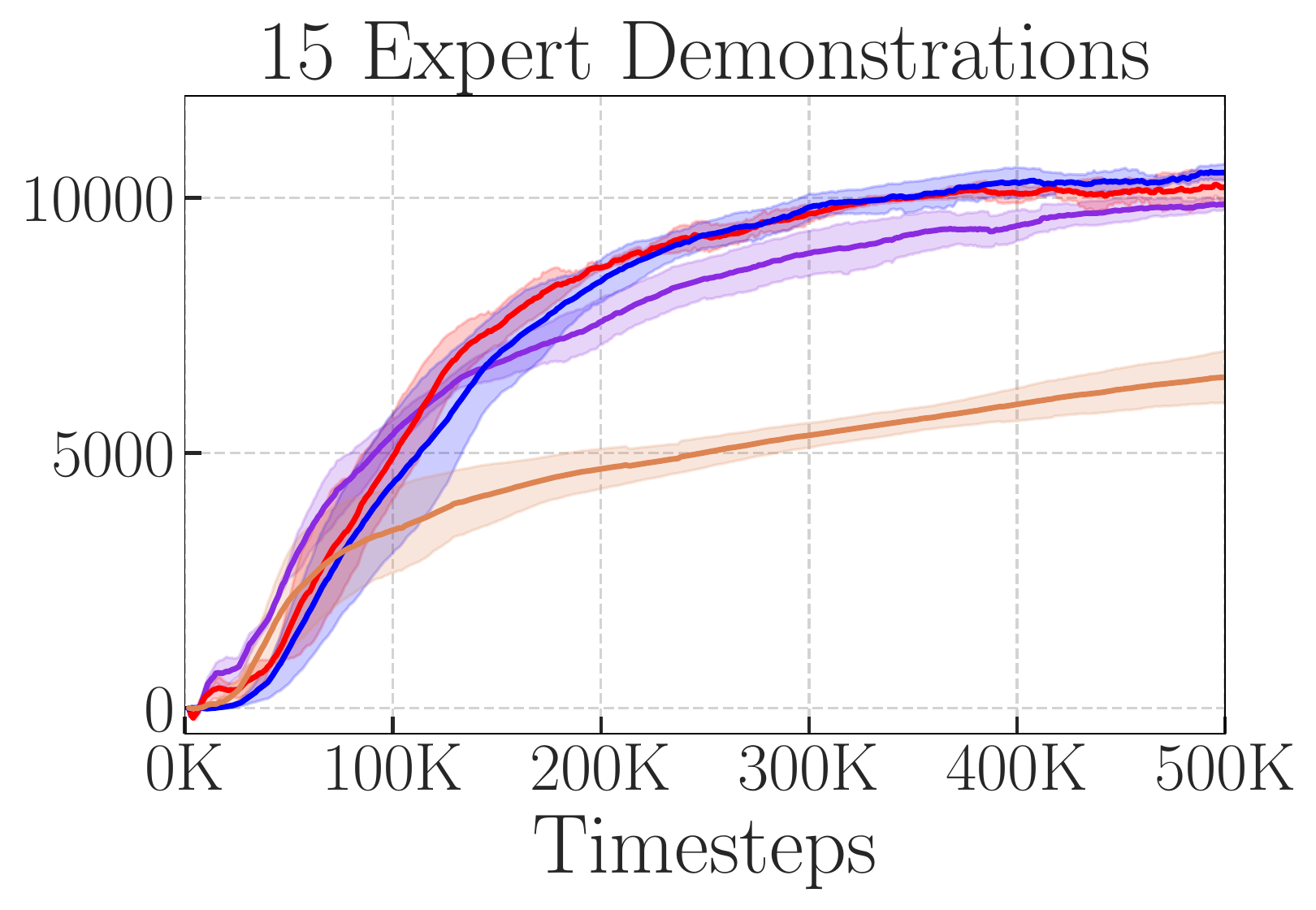}
\end{subfigure}%
\hspace*{-2pt}\begin{subfigure}[t]{0.52\linewidth}
\includegraphics[width=\linewidth]{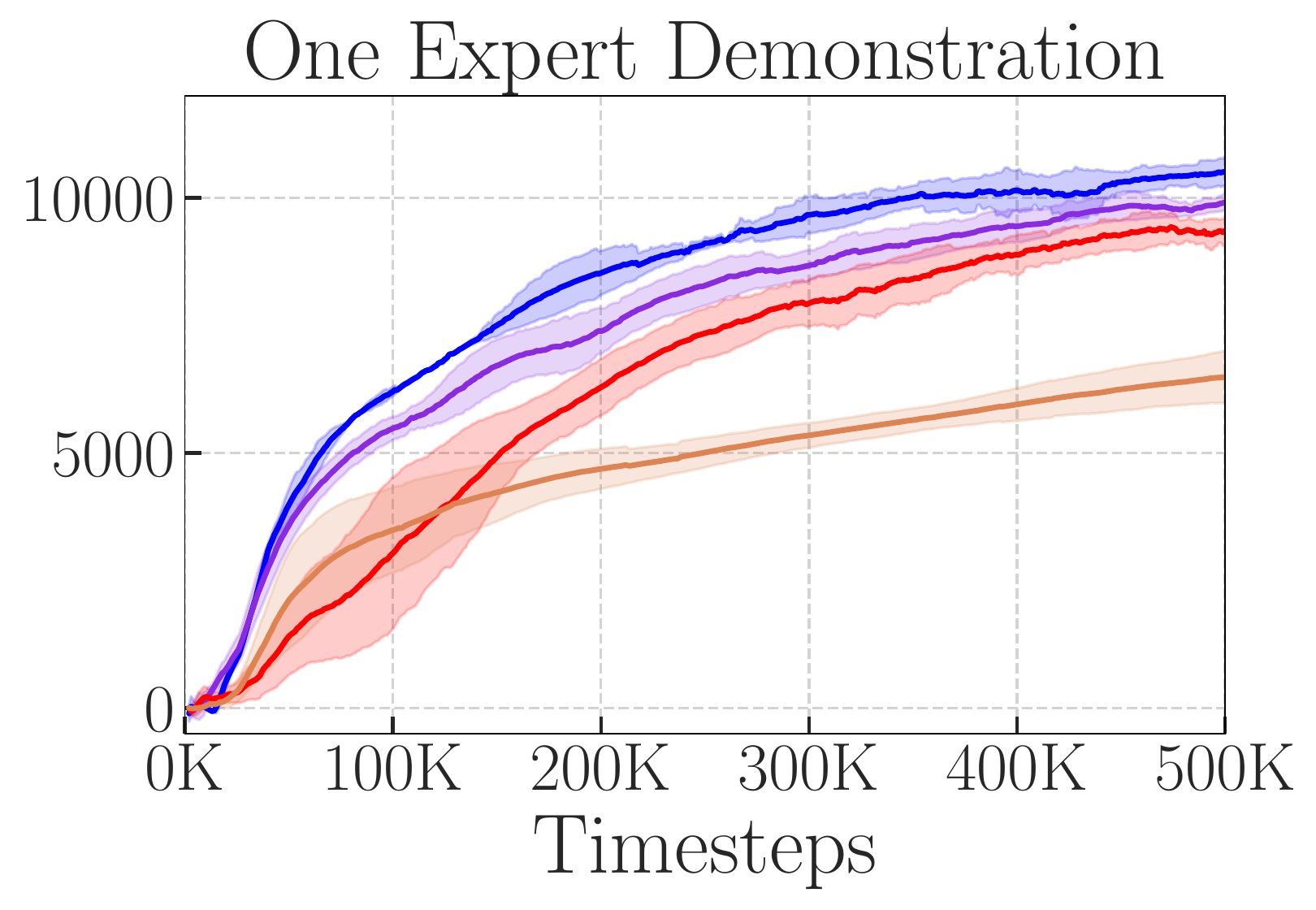}
\end{subfigure}%
\\[-1pt]
\begin{subfigure}[b]{0.93\linewidth}
\includegraphics[trim=0 70 0 90, clip, width=\linewidth]{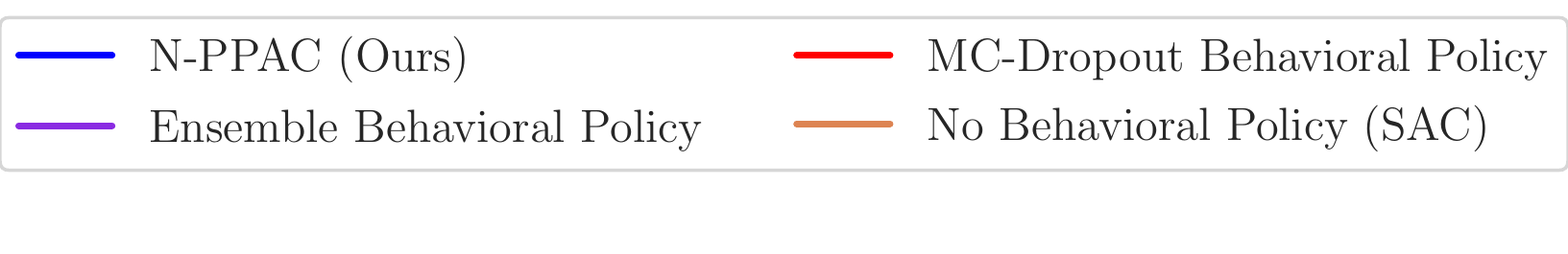}
\end{subfigure}
\vspace*{-13pt}
\caption{
    Returns during online training with different behavioral policies and varying amounts of expert demonstration data on ``HalfCheetah-v2''.
    }
\label{fig:model_classes}
\vspace{-10pt}
\end{wrapfigure}

To assess the usefulness of non-parametric behavioral reference policies in settings where only few expert demonstrations are available, we investigate whether the difference in performance between online policies trained with non-parametric and parametric behavioral reference policies, respectively, is exacerbated the fewer expert demonstrations are available.
To answer this question, we consider the ``HalfCheetah-v2'' environment and compare online policies trained with different behavioral reference policies---non-parametric \gps, deep ensembles, and BNNs with Monte Carlo dropout---estimated either from 15 expert demonstrations (i.e., 15 state--action trajectories, containing 15,000 samples) or from a single expert demonstration (i.e., a single state--action trajectory, containing 1,000 samples).

\textbf{A single expert demonstration is sufficient for non-parametric behavioral reference policies.}$~$
\Cref{fig:model_classes} shows the returns for online policies trained with behavioral reference policies estimated from the full dataset (top plot) and from only a single expert state--action trajectory (bottom plot).
On the full dataset, we find that all three methods are competitive and improve on the prior state-of-the-art but that the \gp behavioral policy leads to the highest return.
Remarkably, non-parametric \gp behavioral policies perform just as well with only a single expert demonstration as with all 15 (i.e., with 1,000 data points, instead of 15,000 data points).
These results further emphasizes the usefulness of non-parametric behavioral policies when accelerating online training with expert demonstrations---even when only very few expert demonstrations are available.

\vspace*{-3pt}
\subsection{Are Non-Parametric GP Behavioral Reference Policies Too Computationally Expensive?}
\label{sec:complexity}

\Cref{tab:gp-time-comparison} presents the time complexity of KL-regularized RL under non-parametric \gp and parametric neural network behavioral reference policies, as measured by the average time elapsed per epoch on the ``door-binary-v0'' and ``HalfCheetah-v2'' environments.
One epoch of online training on ``door-binary-v0'' and ``HalfCheetah-v2'' requires computing the \kld over 1,000 mini-batches of size 256 and 1,024, respectively.
The time complexity of evaluating the log-density of a \gp behavioral reference policy---needed for computing gradients of the \kld during online training---scales quadratically in the number of training data points and linearly in the dimensionality of the state and action space, respectively.
As can be seen in~\Cref{tab:gp-time-comparison}, non-parametric \gp behavioral reference policies only lead to a modest increase in the time needed to complete one epoch of training while resulting in significantly improved performance as shown in Figures~\ref{fig:mujcoo_eval}~and~\ref{fig:dex_eval}.

\setlength{\tabcolsep}{13.2pt}
\begin{table}[h!]
\vspace*{-12pt}
\centering
\caption{
    Time per epoch under different behavioral reference policies for expert demonstration data of varying size computed on a GeForce RTX 3080 GPU.
    The first and second value in each entry of the table give the time required when using a single parametric neural network and a \gp behavioral reference policy, respectively.
}
\vspace*{3pt} 
\begin{tabular}{lccc}
\toprule
\textbf{Dataset} & \textbf{1,000 Data Points} & \textbf{5,000 Data Points} & \textbf{15,000 Data Points}
\\
\midrule
HalfCheetah-v2   & 12.00s / 16.06s           & 11.59s / 18.31s         & 12.00s / 46.54s          \\
door-binary-v0   & 19.62s / 23.78s          & 19.62s / 33.62s         & -                       \\ \hline\hline
\end{tabular}
\label{tab:gp-time-comparison}
\vspace*{-5pt}
\end{table}

\vspace{-5pt}
\section{Conclusion}
\label{sec:conclusion}

We identified a previously unrecognized pathology in KL-regularized RL from expert demonstrations and showed that this pathology can significantly impede and even entirely prevent online learning.
To remedy the pathology, we proposed the use of non-parametric behavioral reference policies, which we showed can significantly accelerate and improve online learning and yield online policies that (often significantly) outperform current state-of-the-art methods on challenging continuous control tasks.
We hope that this work will encourage further research into better model classes for deep reinforcement learning algorithms, including and especially for reinforcement from image inputs.

\begin{ack}
We thank Ashvin Nair for sharing his code and results, as well as for providing helpful insights about the dexterous hand manipulation suite.
We also thank Clare Lyle, Charline Le Lan, and Angelos Filos for detailed feedback on an early draft of this paper, Avi Singh for early discussions about behavioral cloning in entropy-regularized RL, and Tim Pearce for a useful discussion on the role of good models in RL.
TGJR and CL are funded by the Engineering and Physical Sciences Research Council (EPSRC).
TGJR is also funded by the Rhodes Trust and by a Qualcomm Innovation Fellowship.
We gratefully acknowledge donations of computing resources by the Alan Turing Institute.
\end{ack}

\bibliographystyle{plainnat}
\bibliography{references}

\clearpage

\begin{appendices}

\crefalias{section}{appsec}
\crefalias{subsection}{appsec}
\crefalias{subsubsection}{appsec}

\setcounter{equation}{0}
\renewcommand{\theequation}{\thesection.\arabic{equation}}

\onecolumn

\section*{\LARGE Supplementary Material}
\label{sec:appendix}

\section*{Table of Contents}
\vspace*{-10pt}
\startcontents[sections]
\printcontents[sections]{l}{1}{\setcounter{tocdepth}{2}}

\section{Derivations and Further Technical Details}

\subsection{Proof of Proposition 1}
\label{appsec:grad_calcs}

\begin{customproposition}{1}[Exploding Gradients in KL-Regularized RL]
\label{prop:gradients_policy_apdx}
Let $\pi_{0}(\cdot \vbar \bs)$ be a Gaussian behavioral reference policy with mean $\bmu_{0}(\bs)$ and variance $\bsigma^{2}_{0}(\bs)$, and let $\pi(\cdot \vbar \bs)$ be an online policy with reparameterization $\ba_{t} = f_{\phi} (\epsilon_{t} ; \bs_{t})$ and random vector $\epsilon_{t}$.
The gradient of the policy loss with respect to the online policy's parameters $\phi$ is then given by
\begin{align}
\begin{split}
    \hat{\nabla}_{\phi} J_{\pi}(\phi)
    &=
    \big( \alpha \nabla_{\ba_{t}} \log  \pi_{\phi}(\ba_{t} \vbar \bs_{t}) - \alpha \textcolor{black}{ \nabla_{\ba_{t}} \log  \pi_{0}(\ba_{t} \vbar \bs_{t})}
    \\
    &\qquad
    - \nabla_{\ba_{t}} Q(\bs_{t}, \ba_{t}) \big) \nabla_{\phi} f_{\phi} (\epsilon_{t} ; \bs_{t}) + \alpha \nabla_{\phi} \log \pi_{\phi}(\ba_{t} \vbar \bs_{t})
\end{split}
\end{align}%
with
\begin{align}
\SwapAboveDisplaySkip
    \textcolor{black}{\nabla_{\ba_{t}} \log \pi_{0}(\ba_{t} \vbar \bs_{t})} = -\frac{\ba_{t}-\bmu_{0}(\bs_{t})}{\bsigma^{2}_{0}(\bs_{t})}.
\end{align}
For fixed $|\ba_{t}-\bmu_{0}(\bs_{t})|$, $\nabla_{\ba_{t}} \log \pi_{0}(\ba_{t} \vbar \bs_{t})$ grows as $\calO(\bsigma^{-2}_{0}(\bs_{t}))$; thus,
\begin{align}
    \vbar \hat{\nabla}_{\phi} J_{\pi}(\phi) \vbar \rightarrow \infty \quad \text{as} \quad \bsigma^2_{0}(\bs_{t}) \rightarrow 0,
\end{align}
when $\nabla_{\phi} f_{\phi} (\epsilon_{t} ; \bs_{t}) \neq 0$.
\end{customproposition}

\begin{proof}
The policy loss, as given in~\Cref{eq:objective_policy}, is:
\begin{align}
   \label{eq:qpp_policyloss}
   J_{\pi}(\phi )
   &=
   \mathbb{E}_{\bs_{t} \sim \calD} \left[ \mathbb{D}_{\textrm{KL}}\bigl( \pi_\phi (\cdot \vbar \bs_{t}) \,||\, \pi_{0}(\cdot \vbar \bs_{t}) \bigr) \right]
   - \mathbb{E}_{\bs_{t} \sim \calD} \left[ \mathbb{E}_{\ba_{t} \sim \pi_{\phi}} \left[ Q_{\theta}(\bs_{t}, \ba_{t}) \right] \right].
\end{align}
To obtain a lower-variance gradient estimator, the policy is reparameterized using a neural network transformation
\begin{align}
    \ba_{t} = f_{\phi} (\epsilon_{t} ; \bs_{t})
\end{align}
where $\epsilon_{t}$ is an input noise vector.
Following~\citet{haarnoja2019soft}, we can now rewrite~\Cref{eq:qpp_policyloss} as
\begin{align}
\SwapAboveDisplaySkip
\label{eq:qpp_policyloss2}
   J_{\pi}(\phi )
   &=
   \mathbb{E}_{\bs_{t} \sim \calD, \epsilon_{t}} \left[ \alpha \big( \log  \pi_{\phi}(f_{\phi} (\epsilon_{t} ; \bs_{t}) \vbar \bs_{t}) - \log \pi_{0}(f_{\phi} (\epsilon_{t} ; \bs_{t}) \vbar \bs_{t}) \big) - Q(\bs_{t}, f_{\phi} (\epsilon_{t} ; \bs_{t})) \right] 
\end{align}
where $\calD$ is a replay buffer and $\pi_{\phi}$ is defined implicitly in terms of $f_{\phi}$.
We can approximate the gradient of~\Cref{eq:qpp_policyloss2} with
\begin{align}
\begin{split}
    \hat{\nabla}_{\phi} J_{\pi}(\phi)
    &=
    \big( \alpha \nabla_{\ba_{t}} \log \pi_{\phi}(\ba_{t} \vbar \bs_{t}) - \alpha \textcolor{black}{ \nabla_{\ba_{t}} \log  \pi_{0}(\ba_{t} \vbar \bs_{t})}
    \\
    &\qquad
    - \nabla_{\ba_{t}} Q(\bs_{t}, \ba_{t}) \big) \nabla_{\phi} f_{\phi} (\epsilon_{t} ; \bs_{t}) + \alpha \nabla_{\phi} \log \pi_{\phi}(\ba_{t} \vbar \bs_{t}).
\end{split}
\end{align}

\vspace{4pt}

Next, consider the term $\nabla_{\ba_{t}} \log  \pi_{0}(\ba_{t} \vbar \bs_{t})$ for a Gaussian policy:
\begin{align}
\SwapAboveDisplaySkip
    \log \pi_{0}(\ba_{t} \vbar \bs_{t})
    =
    \log \left( \frac{1}{\bsigma_{0}(\bs_{t})\sqrt{2\pi }} \right) -\frac{1}{2} \left(\frac{\ba_{t}-\bmu_{0}(\bs_{t})}{\bsigma_{0}(\bs_{t})(\bs_{t})}\right)^{2}
\end{align}
Thus,
\begin{align}
\SwapAboveDisplaySkip
    \textcolor{black}{\nabla_{\ba_{t}} \log \pi_{0}(\ba_{t} \vbar \bs_{t})} = -\frac{\ba_{t}-\bmu_{0}(\bs_{t})}{\bsigma^{2}_{0}(\bs_{t})}.
\end{align}
For fixed $|\ba_{t}-\bmu_{0}(\bs_{t})|$, $\nabla_{\ba_{t}} \log (\pi_{0}(\ba_{t} \vbar \bs_{t}))$ grows as $\mathcal{O}(\bsigma^{-2}_{0}(\bs_{t}))$, and so,
\begin{align}
    |\hat{\nabla}_{\phi} J_{\pi}(\phi) |\rightarrow \infty \quad \text{as} \quad \bsigma^2_{0}(\bs_{t}) \rightarrow 0.
\end{align}
whenever $\nabla_{\phi} f_{\phi} (\epsilon_{t} ; \bs_{t}) \neq 0$.
\end{proof}

\subsection{Laplace Parametric Behavioral Reference Policy}

A Laplace behavioral reference policy may be able to mitigate some of the problems posed by~\Cref{prop:gradients_policy} due to the heavy tails of the distribution.
The gradient for a Laplace behavioral reference policy
\begin{align}
    \pi_{0}(\ba_{t} \vbar \bs_{t})
    \defines
    \frac{1}{2 \bsigma_{0}(\bs_{t})} \exp \left(-\frac{|\a_{t}-\bmu_{0}(\bs_{t})|}{\bsigma_{0}(\bs_{t})}\right) ,
\end{align}
increases linearly for a given distance between $\ba_{t}$ and the mean $\bmu_{0}(\bs_{t})$ as the scale $\bsigma_{0}(\bs_{t})$ tends to zero.

\subsection{Regularized Maximum Likelihood Estimation}
\label{appsec:entregpredvar}

To address the collapse in predictive variance away from the offline dataset under MLE training seen in~\Cref{fig:heatmaps}, \citet{wu2019behavior} in practice augment the usual MLE loss with an entropy bonus as follows:
\begin{align}
    \pi_{0} \defines \pi_{\psi^\star} ~~~\text{with}~~~ \psi^\star
    \defines
    \argmax_{\psi} \left\{ \mathbb{E}_{(\bs, \ba) \sim \calD} [\log \pi_{\psi} (\ba \vbar \bs) + \beta \calH(\pi_{\psi}(\cdot\vbar \bs))] \right\} .
\end{align}
where $\beta$ is temperature tuned to an entropy constraint similar to~\citet{haarnoja2019soft}.
The entropy bonus is estimated by sampling from the behavioral policy as
\begin{align}
    \calH(\pi_{\psi}(\cdot\vbar \bs)) = \mathbb{E}_{\ba\sim\pi_{\psi}}[-\log\pi_{\psi}(\ba \vbar \bs)]
\end{align}

\Cref{fig:app_ent_heatmaps} shows the predictive variances of behavioral policies trained on expert demonstrations for the ``door-binary-v0'' environment with various entropy coefficients $\beta$. Whilst entropy regularization partially mitigates the collapse of predictive variance away from the expert demonstrations, we still observe the wrong trend similar to~\Cref{fig:heatmaps} with predictive variances high near the expert demonstrations and low on unseen data. The variance surface also becomes more poorly behaved, with ``islands'' of high predictive variance appearing away from the data.

We may also add Tikhonov regularization~\citep{groetsch1984theory} to the MLE objective, explicitly,
\begin{align}
    \pi_{0} \defines \pi_{\psi^\star} ~~~\text{with}~~~ \psi^\star
    \defines
    \argmax_{\psi} \left\{ \mathbb{E}_{(\bs, \ba) \sim \calD} [\log \pi_{\psi} (\ba \vbar \bs) - \lambda \psi^\top\psi] \right\} .
\end{align}
where $\lambda$ is the regularization coefficient.

\Cref{fig:app_tik_heatmaps} shows the predictive variances of behavioral policies trained on expert demonstrations for the ``door-binary-v0'' environment with varying Tikhonov regularization coefficients $\lambda$. Similarly, Tikhonov regularization does not resolve the issue with calibration of uncertainties. We also observe that too high a regularization strength causes the model to underfit to the variances of the data.

\clearpage

\subsection{Comparison to Prior Works}
\label{appsec:prior_comparisons}

To assess the usefulness of KL regularization for improving the performance and sample efficiency of online learning with expert demonstrations, we compare our approach to methods that incorporate expert demonstrations into online learning implicitly or explicitly via KL regularization as well as by means other than KL regularization.

\textbf{ABM \citep{Siegel2020Keep}.}$~$
ABM explicitly KL-regularizes the online policy against a behavioral policy.
This behavioral policy can be estimated via MLE, like BRAC, or alternatively via an ``advantage-weighted behavioral model'' where the RL algorithm is biased to choose actions that are both supported by the offline data and that are good for the current task.
This objective filters trajectory snippets by advantage-weighting, using an $n$-step advantage function.
We show that no carefully chosen objective with additional hyperparameters is required.

\textbf{AWAC \citep{nair2020accelerating}.}$~$
AWAC performs online fine-tuning of a policy pre-trained on offline.
It achieves state-of-the-art results on the dexterous hand manipulation and MuJoCo continuous locomotion tasks.
AWAC implicitly constrains the \kld of the online policy to be close to the behavioral policy by sampling from the replay buffer, which is initially filled with the offline data. 
The method requires additional off-policy data to be generated to saturate the replay buffer, thereby requiring a hidden number of environment interactions that do not involve learning.
Our approach does not require the offline data to be added to the replay buffer before training.

\textbf{AWR \citep{peng2019advantageweighted}.}$~$
AWR approximates constrained policy search by alternating between supervised value function and policy regression steps.
The objective derived is similar to AWAC but instead estimates the value function of the behavioral policy which was demonstrated to be less efficient than $Q$-function estimation via bootstrapping \citep{nair2020accelerating}.
The method may be converted to use offline data by adding prior data to the replay buffer before training.

\textbf{BEAR \citep{NIPS2019_9349}.}$~$
BEAR attempts to stabilize learning from off-policy data (such as offline data) by tackling bootstrapping error from actions far from the training data.
This is achieved by searching for policies with the same support as the training distribution.
This approach is too restrictive for the problem considered in this paper, since only a small number of expert demonstrations is available, which requires exploration.
In contrast, our approach encourages exploration away from the data by wider behavioral policy predictive variances.
BEAR uses an alternate divergence measure to the \kld, Maximum Mean Discrepancy~\citep{10.1007/978-3-540-75225-7_5}.
Other divergences such as Wasserstein Distances~\citep{pmlr-v119-pacchiano20a} have also been proposed for regularization in RL.

\textbf{BRAC \citep{wu2019behavior}.}$~$
BRAC regularizes the online policy against an offline behavioral policy as our method does.
However, BRAC exhibits the pathologies we have shown by learning a poor behavioral policy via MLE.
To mitigate this, in practice, BRAC adds an entropy bonus to the supervised learning objective which stabilizes the variance around the training set but has no guarantees away from the data.
We demonstrate that behavioral policy obtained via maximum likelihood estimation with entropy regularization exhibit a collapse in predictive uncertainty estimates way from the training data, resulting in the pathology described in~\Cref{prop:gradients_policy}.

\textbf{DAPG \citep{Rajeswaran-RSS-18}.}$~$
DAPG incorporates offline data into policy gradients by initially pre-training with a behaviorally cloned policy and then augmenting the RL loss with a supervised-learning loss.
We similarly pre-train the online policy at the start to avoid noisy KLs at the beginning of training.
However, training a joint loss that combines two disparate and often divergent terms can be unstable.

\textbf{SAC+BC \citep{nair2018overcoming}.}$~$
SAC+BC represents the approach of~\citet{nair2018overcoming} but uses SAC instead of DDPG~\citep{DBLP:journals/corr/LillicrapHPHETS15} as the underlying RL algorithm.
The method maintains a secondary replay buffer filled with offline data that is sampled each update step, augmenting the policy loss with a supervised learning loss that is filtered by advantage and hindsight experience replay.
Our method requires far fewer additional ad-hoc algorithmic design choices.

\textbf{SACfD \citep{haarnoja2019soft}.}$~$
SACfD uses the popular Soft Actor--Critic (SAC) algorithm with offline data loaded into the replay buffer before online training.
Our algorithm uses the same approximate policy iteration scheme as SAC with a modified objective. \citet{nair2020accelerating} show that including the offline data into the replay buffer does not significantly improve the training performance over the unmodified SAC objective and that pre-training the online policy with offline data results in catastrophic forgetting.
Thus, a different approach is needed to integrate offline data with SAC-style algorithms.

\clearpage

\section{Further Experimental Results}

\subsection{Exploding $Q$-function Gradients}
\label{appsec:qfngrad}

In~\Cref{prop:gradients_policy} and \Cref{sec:empirical_evaluation}, we showed that the policy gradient $\hat{\nabla}_{\phi} J_{\pi}(\phi )$ explodes due to the blow-up of the gradient of the behavioral reference policy's log-density as the behavioral policy predictive variance $\bsigma_{0}(\bs)$ tends to zero.
A similar relationship holds for the $Q$-function gradients, which we confirm empirically in~\Cref{fig:mujoco_gp_mean_fixed_var_abl_qfn}.

\begin{figure*}[ht!]
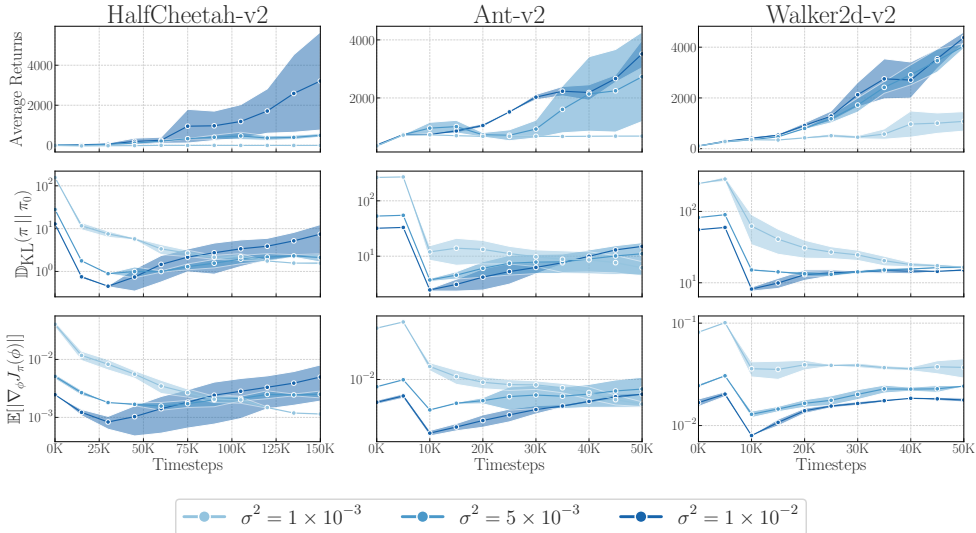

\centering
\includegraphics[width=0.95\textwidth]{figures/mujoco_gp_mean_fixed_var_abls.pdf}
\includegraphics[trim=0 100 0 100, clip, width=0.6\textwidth]{figures/gp_mean_cvar_abl_legend.pdf}
    \caption{
        Ablation study showing the effect of predictive variance collapse on the performance of KL-regularized RL on MuJoCo benchmarks.
        Policies shown from dark to light in order of decreasing constant predictive variance, simulating training under maximum likelihood estimation.
        The plots show the average return of the learned policy, magnitude of the KL penalty, and magnitude of the $Q$-function gradients during online training.
    }
    \label{fig:mujoco_gp_mean_fixed_var_abl_qfn}
\end{figure*}

\subsection{Ablation Study on the Effect of KL Divergence Temperature Tuning}
\label{appsec:kl_tuning}

\Cref{fig:kltemp_abl} shows that unlike in standard SAC \citep{haarnoja2019soft}, tuning of the KL-temperature is not necessary to achieve good online performance.
For simplicity, we use a fixed value throughout our experiments.

\begin{figure*}[ht!]
\centering
\includegraphics[width=0.95\textwidth]{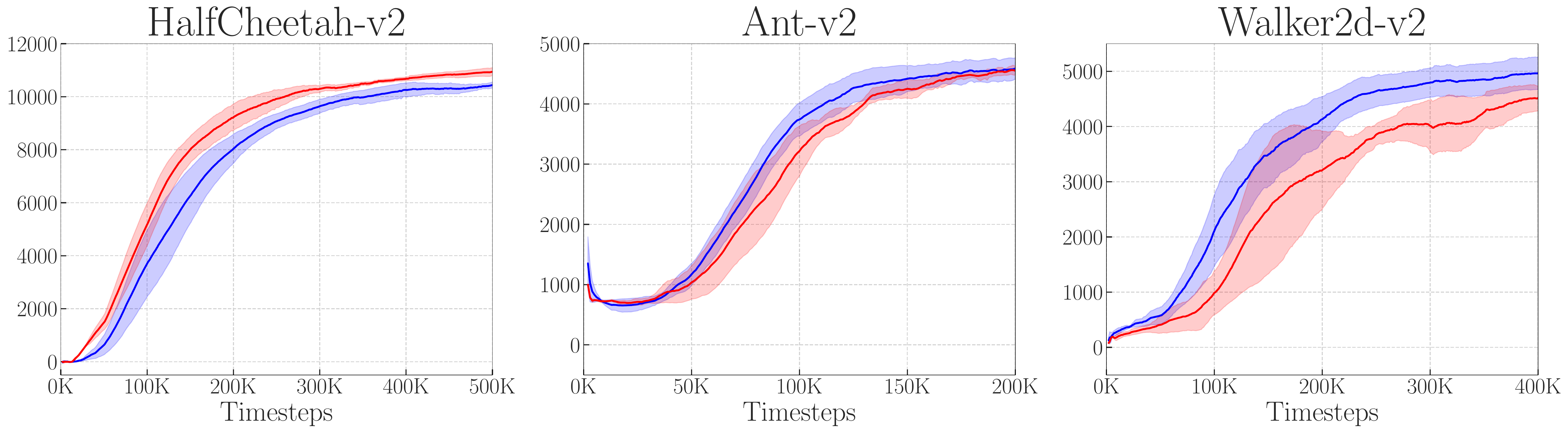}
\includegraphics[trim=0 100 0 100, clip, width=0.55\textwidth]{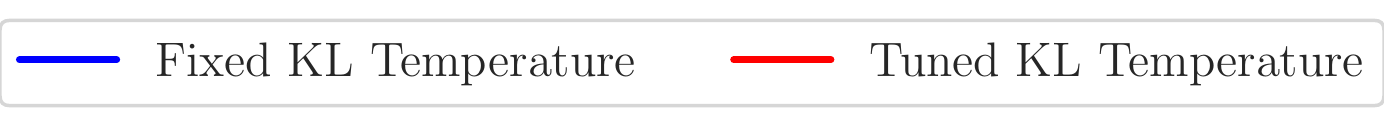}
\caption{
    Ablation study on the effect of automatic KL temperature tuning on the performance of KL-regularized RL with a non-parametric \gp behavioral reference policy on MuJoCo locomotion tasks.
}
\label{fig:kltemp_abl}
\end{figure*}

\clearpage

\subsection{Ablation Study: Performance under a Laplace Parametric Behavioral Reference Policy}
\label{appsec:laplace}

We use a Laplace behavioral reference policy to assess whether it is more effective at incorporating the expert demonstration data into online training.
\Cref{fig:lap_abl} shows empirical results using the Laplace behavioral reference policy compared against \textsc{n-ppac} (in blue) and a SAC baseline (in green) on three MuJoCo locomotion tasks.
We use automatic KL-temperature tuning for this ablation.
On the Ant-v2 environment, the Laplace behavioral reference policy slightly improves upon the baseline SAC performance, which does not use any prior information at all.
On the door and pen environment, the online policy learned under the Laplace behavioral reference policy does not learn any meaningful behavior.

In both MuJoCo locomotion tasks and the ``door-binary-v0'' and ``pen-binary-v0'' dexterous hand manipulation environments, \textsc{n-npac} significantly outperforms both the online policy learned under the Laplace behavioral reference policy and the SAC baseline.
We can understand the behavior under the Laplace behavioral reference policy in terms of collapse of predictive variance away from data for neural network parameterized policies, as it too has a decreasing variance away from the expert trajectories.

\begin{figure*}[ht!]
\centering
\includegraphics[width=0.95\textwidth]{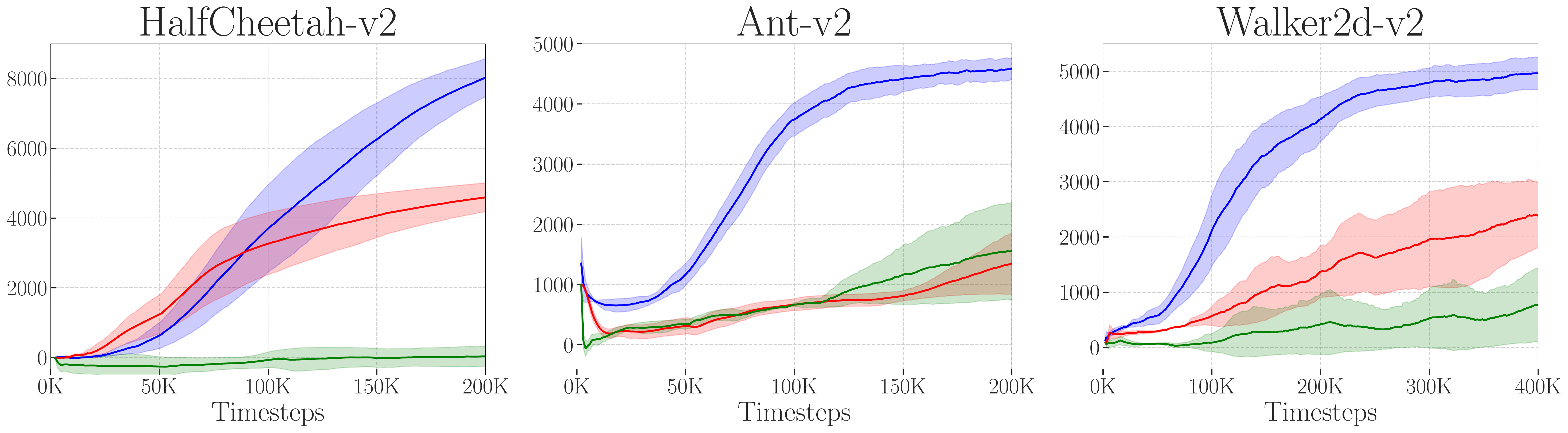}
\includegraphics[trim=0 100 0 100, clip, width=0.55\textwidth]{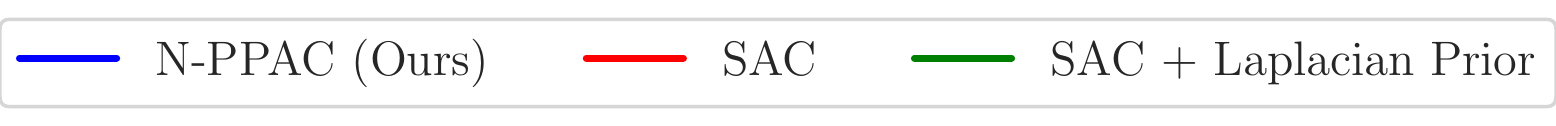}
    \caption{
        Ablation study using heavier-tailed Laplace behavioral reference policy on MuJoCo locomotion tasks.
    }
    \label{fig:lap_abl}
\end{figure*}

\clearpage

\subsection{Visualizations of Regularized Maximum Likelihood Parametric Behavioral Policies}

\begin{center}
    \large Maximum Likelihood + Entropy Maximization
\end{center}
\begin{figure}[ht!]
\vspace*{-15pt}
\centering
\begin{subfigure}[t]{0.28\columnwidth}
    \includegraphics[width=\columnwidth]{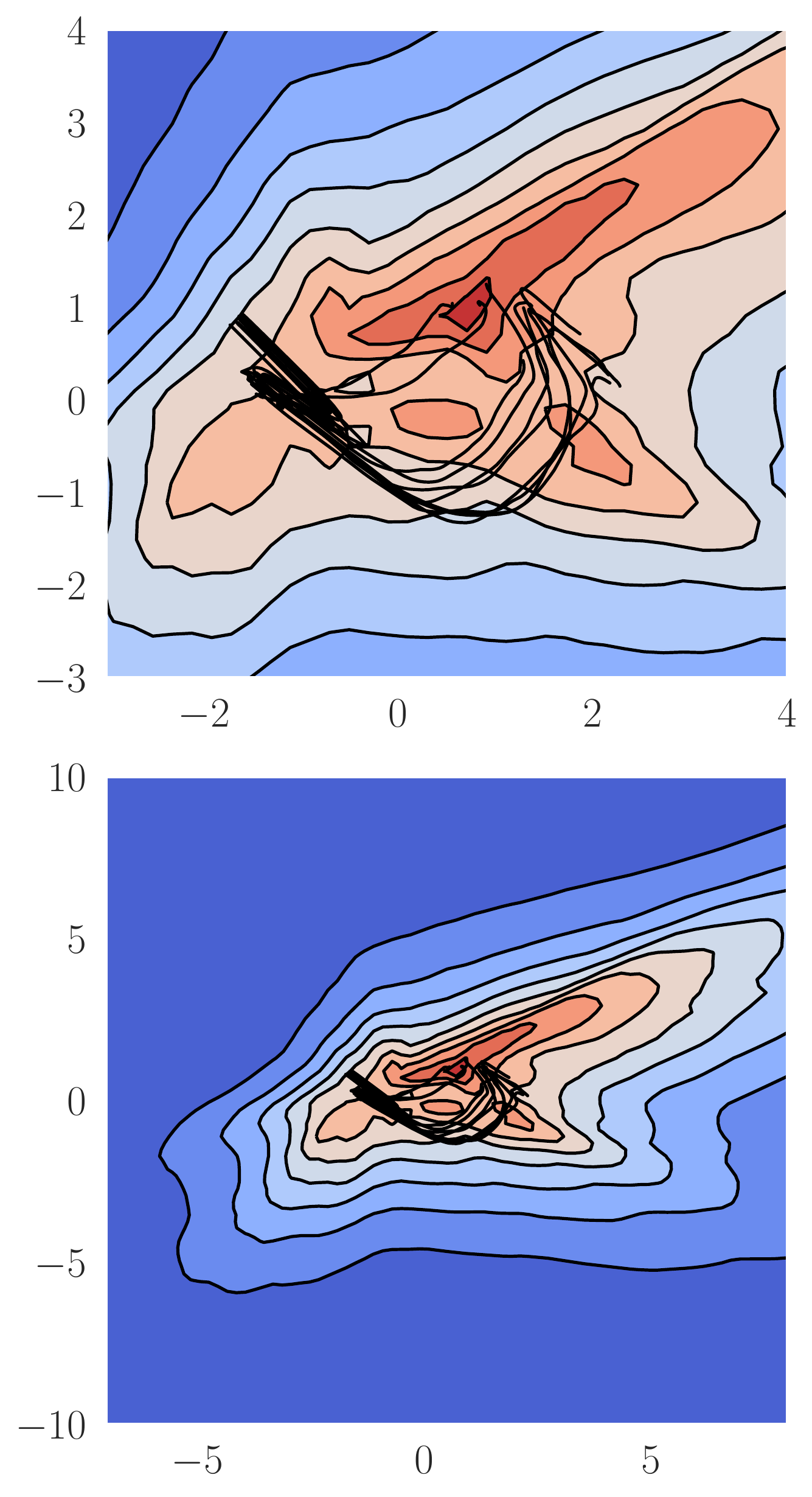}
    \caption{$\beta = 10^{-2}$}
\end{subfigure}
\begin{subfigure}[t]{0.28\columnwidth}
    \includegraphics[width=\columnwidth]{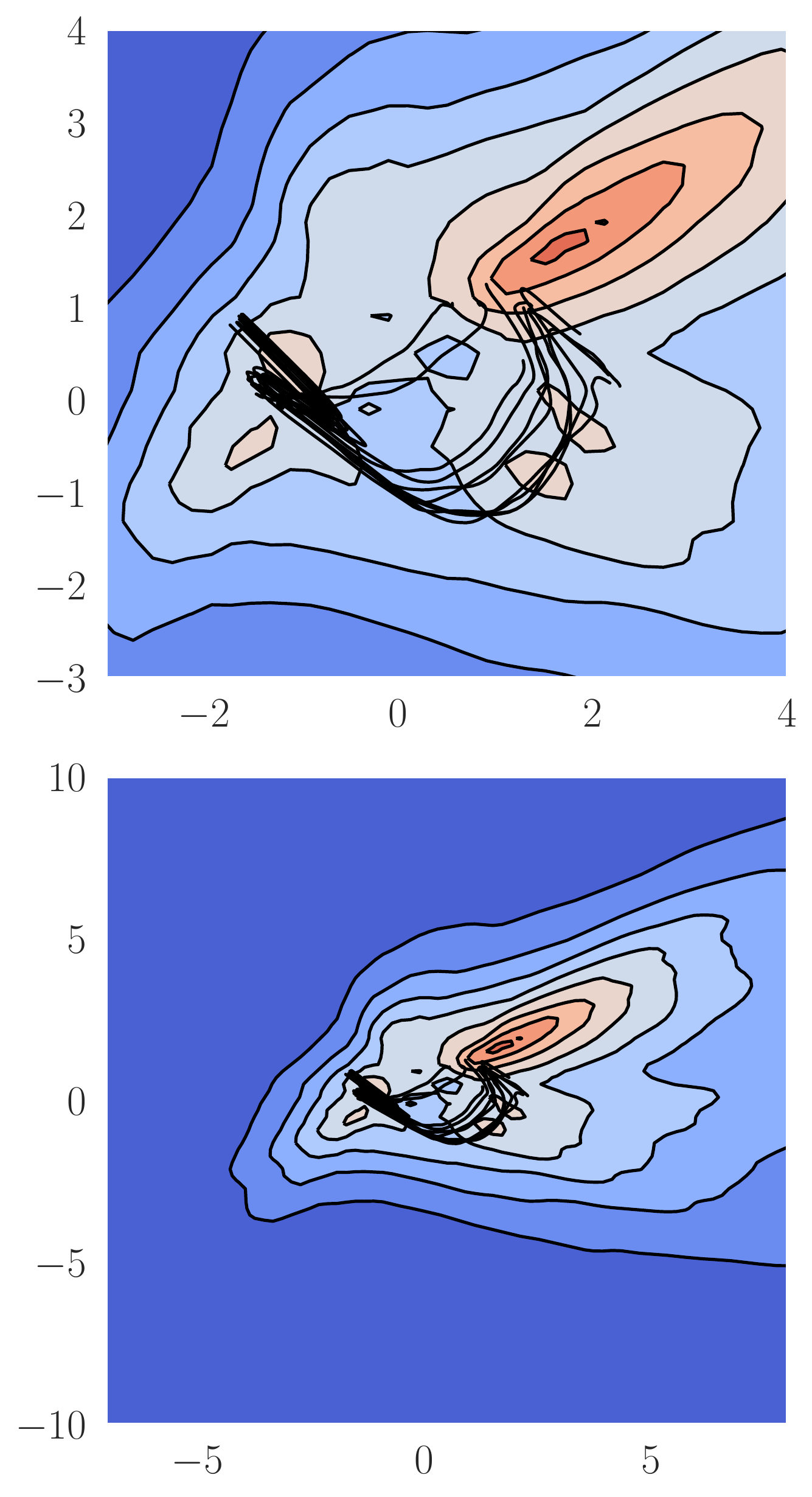}
    \caption{$\beta = 10^{-3}$}
\end{subfigure}
\begin{subfigure}[t]{0.28\columnwidth}
    \includegraphics[width=\columnwidth]{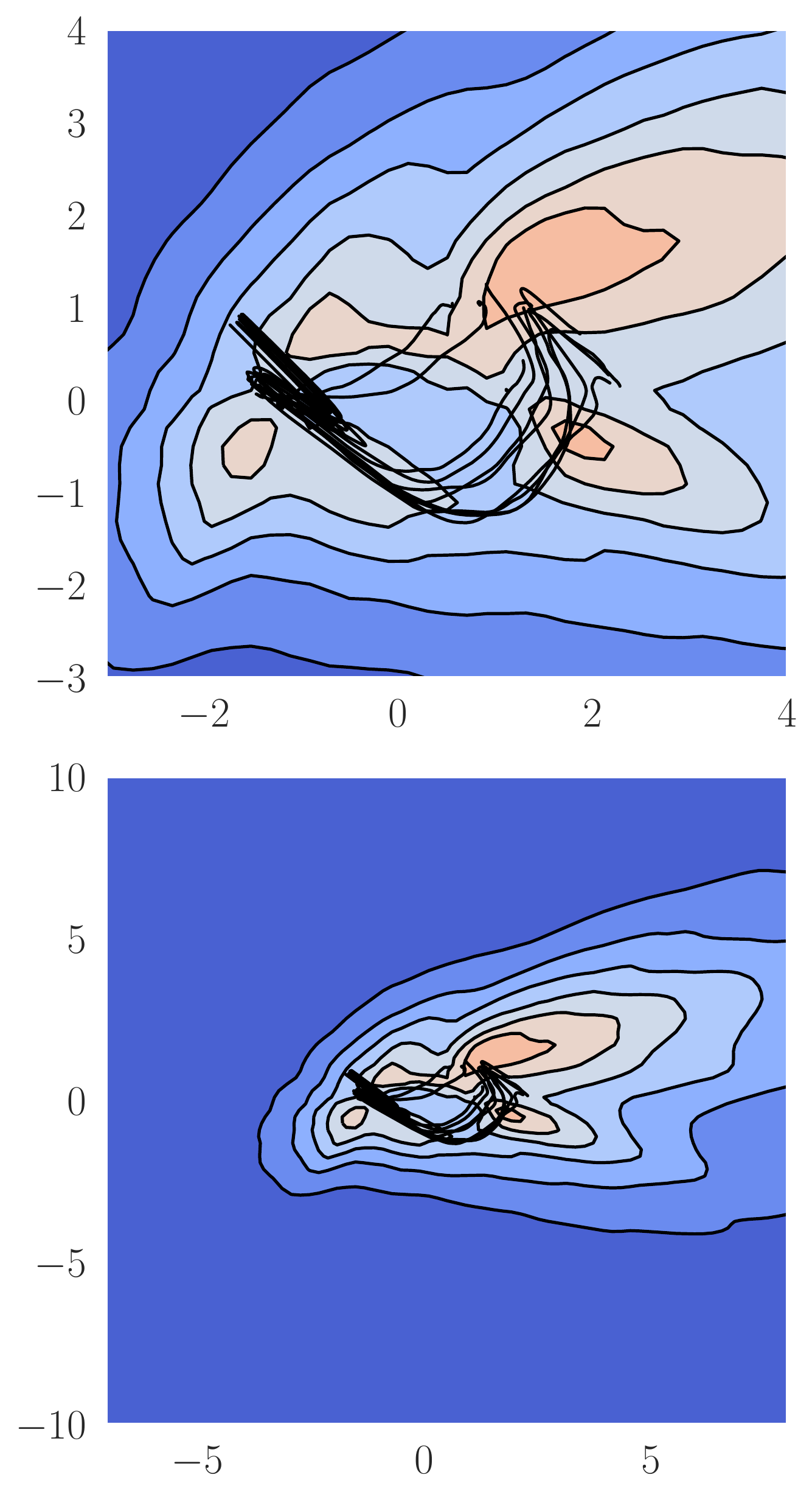}
    \caption{$\beta = 10^{-4}$}
\end{subfigure}
\begin{subfigure}[t]{0.071\columnwidth}
    \includegraphics[width=\columnwidth]{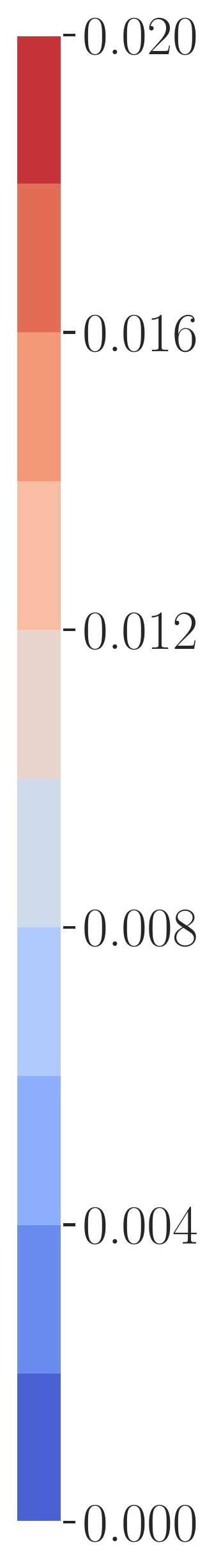}
\end{subfigure}
\caption{
    Predictive variances of parametric neural network Gaussian behavioral policies $\pi_{\psi}(\cdot \vbar \bs) = \calN(\bmu_{\psi}(\bs), \bsigma^2_{\psi}(\bs))$ trained with different entropy regularization coefficients $\beta$.
    }
\label{fig:app_ent_heatmaps}
\end{figure}

\medskip

\begin{center}
    \large Maximum Likelihood + Tikhonov Regularization
\end{center}
\begin{figure}[ht!]
\vspace*{-15pt}
\centering
\begin{subfigure}[t]{0.275\columnwidth}
    \includegraphics[width=\columnwidth]{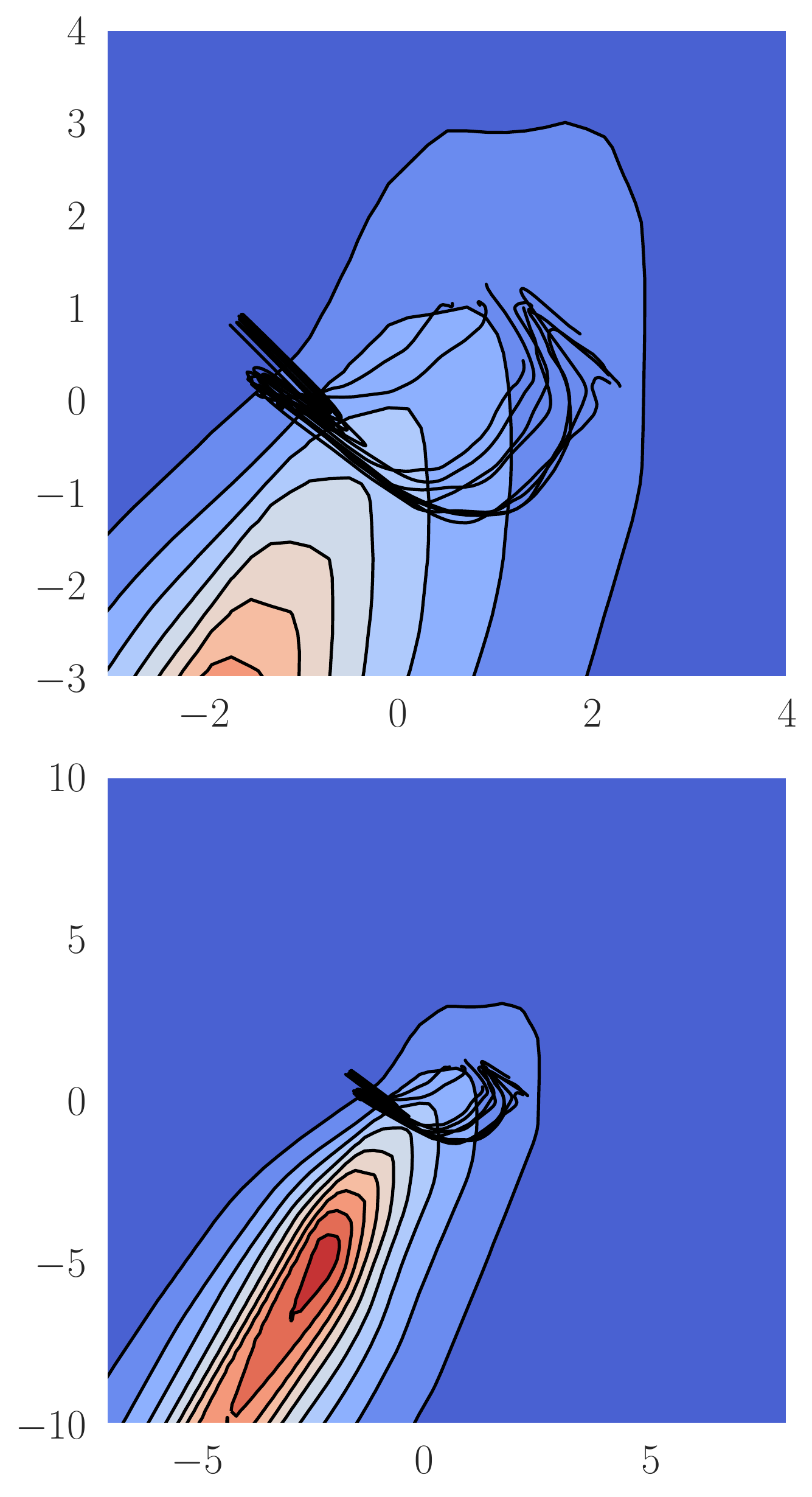}
    \caption{$\lambda = 10^{-2}$}
\end{subfigure}
\begin{subfigure}[t]{0.07\columnwidth}
    \includegraphics[width=\columnwidth]{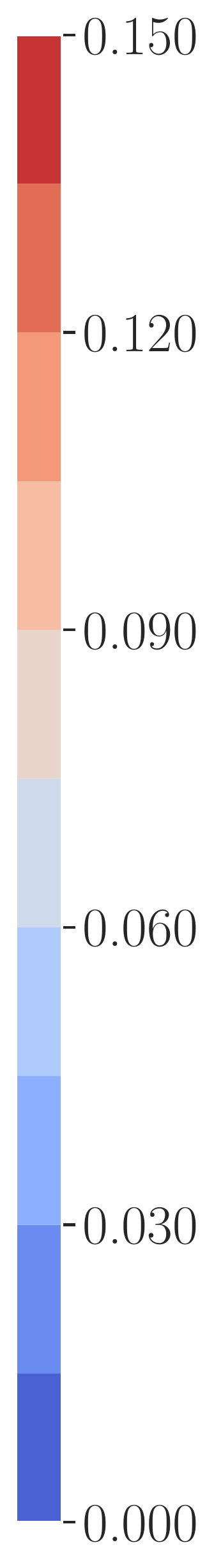}
\end{subfigure}
\begin{subfigure}[t]{0.275\columnwidth}
    \includegraphics[width=\columnwidth]{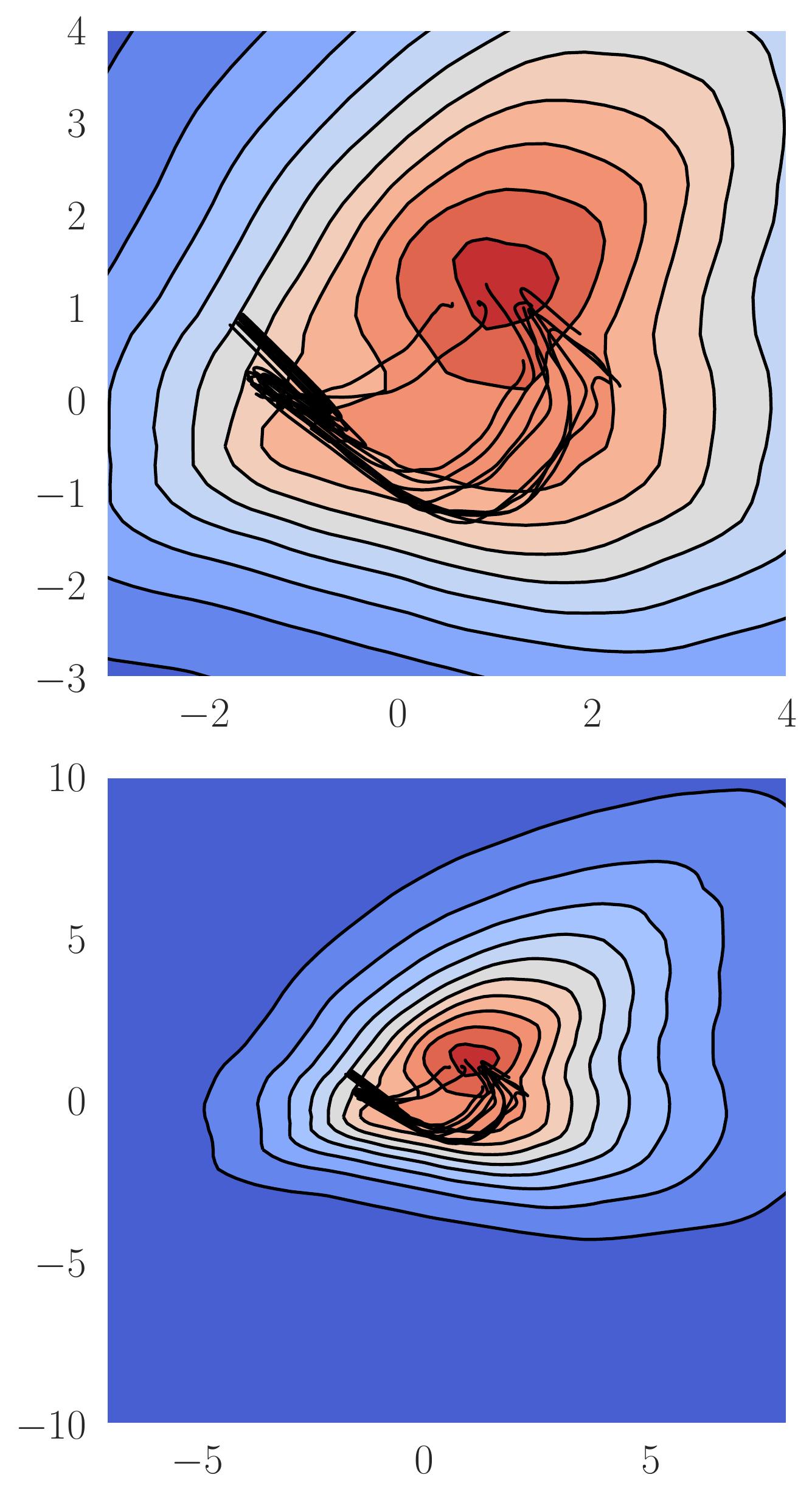}
    \caption{$\lambda = 10^{-3}$}
\end{subfigure}
\begin{subfigure}[t]{0.275\columnwidth}
    \includegraphics[width=\columnwidth]{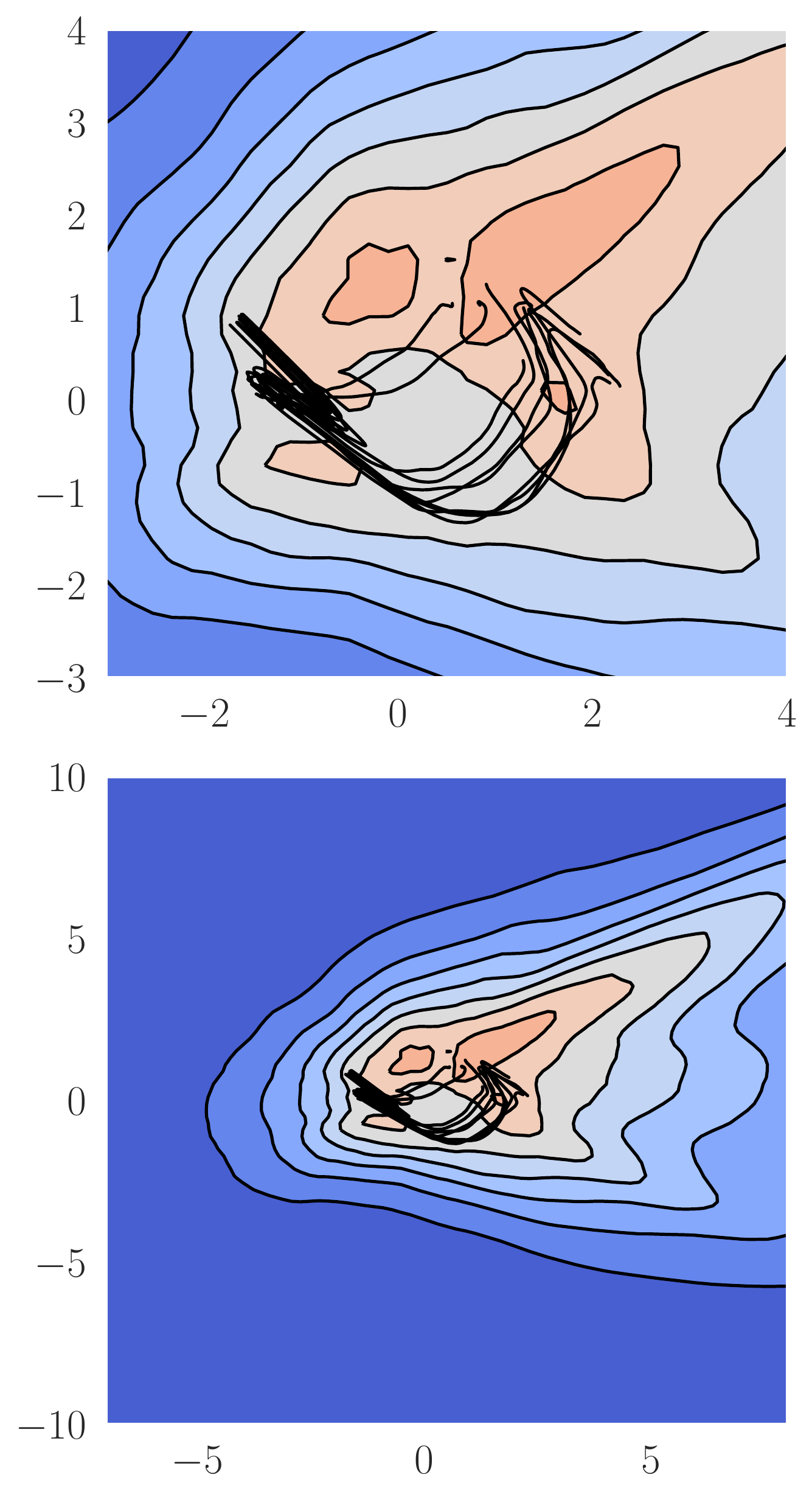}
    \caption{$\lambda = 10^{-4}$}
\end{subfigure}
\begin{subfigure}[t]{0.069\columnwidth}
    \includegraphics[width=\columnwidth]{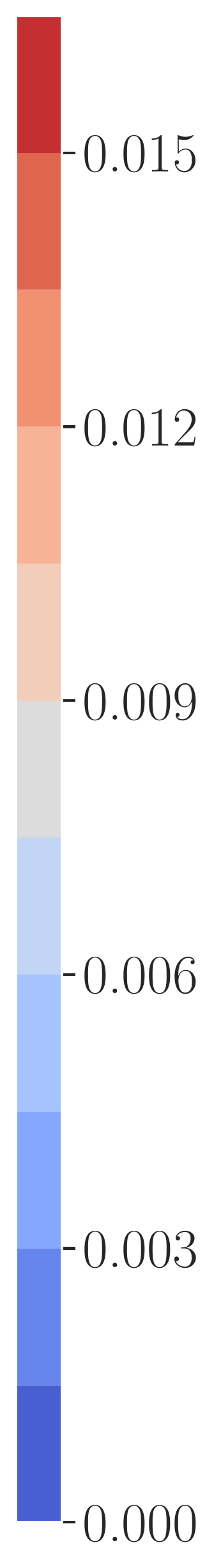}
\end{subfigure}
\caption{
    Predictive variances of parametric neural network Gaussian behavioral policies $\pi_{\psi}(\cdot \vbar \bs) = \calN(\bmu_{\psi}(\bs), \bsigma^2_{\psi}(\bs))$ trained with different Tikhonov regularization coefficients $\lambda$.
    }
\label{fig:app_tik_heatmaps}
\vspace*{-40pt}
\end{figure}

\clearpage

\subsection{Visualizations of Ensemble Maximum Likelihood Parametric Behavioral Policies}
\label{appsec:policy_ensemble}

On the ``door-binary-v0'' environment, we consider an ensemble of parametric neural network Gaussian policies $\pi_{\psi^{1:K}}(\cdot \vbar \bs) \defines \calN( \bmu_{\psi^{1:K}}(\bs), \bsigma^2_{\psi^{1:K}}(\bs) )$ with
\begin{align}
    \mu_{\psi^{1:K}}(\bs) \defines \frac{1}{K} \sum_{k=1}^K \bmu_{\psi^{k}}(\bs), \quad \bsigma^2_{\psi^{1:K}}(\bs) \defines \frac{1}{K} \sum_{k=1}^K \left( \bsigma^2_{\psi^{k}}(\bs) + \bmu^2_{\psi^{k}}(\bs) \right) - \mu^2_{\psi^{1:K}}(\bs)
\end{align}

\begin{figure}[ht!]
\vspace{-20pt}
\centering
\begin{subfigure}[t]{0.28\columnwidth}
    \includegraphics[width=\columnwidth]{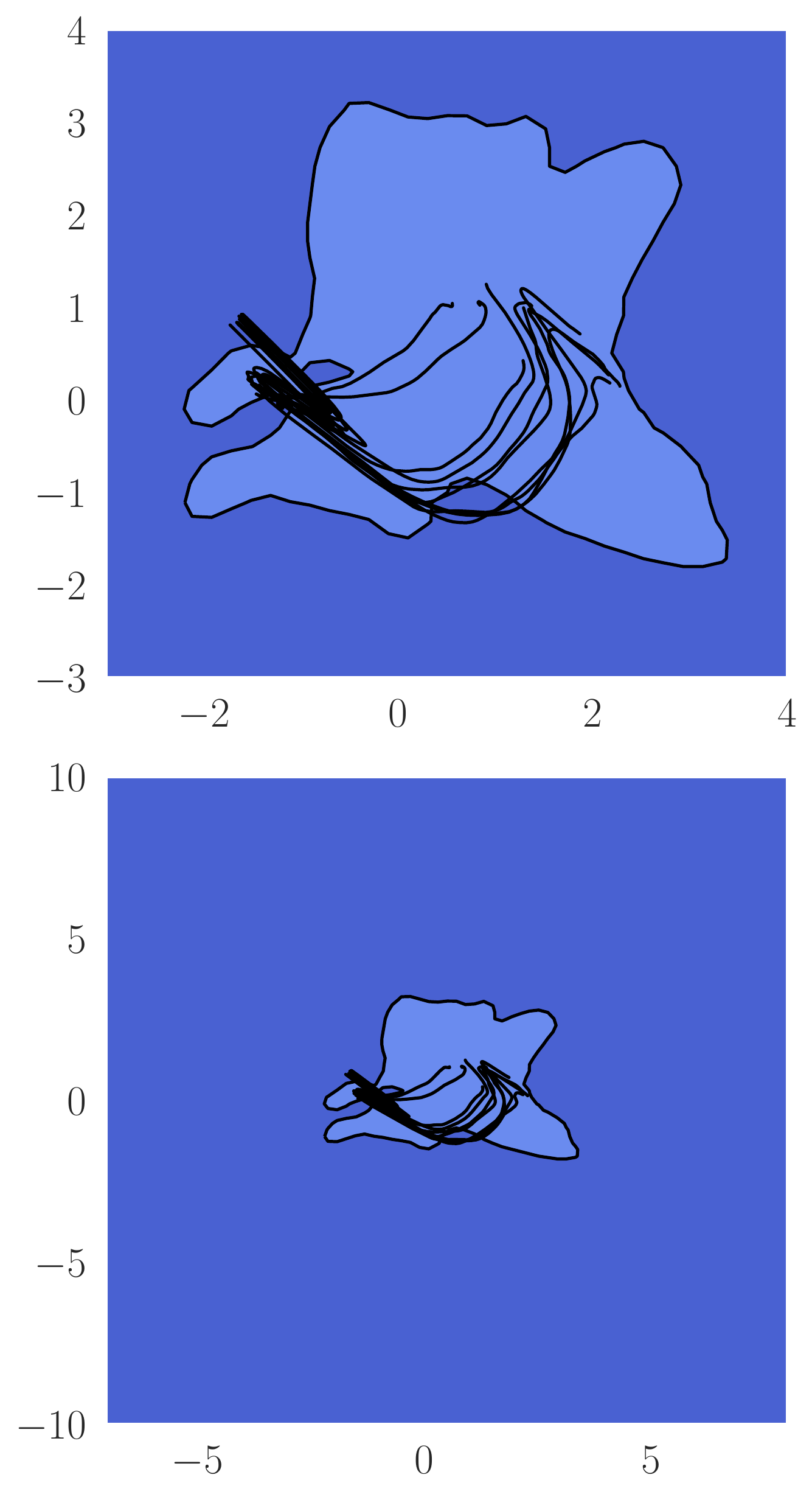}
    \caption{Single Model}
\end{subfigure}
\begin{subfigure}[t]{0.28\columnwidth}
    \includegraphics[width=\columnwidth]{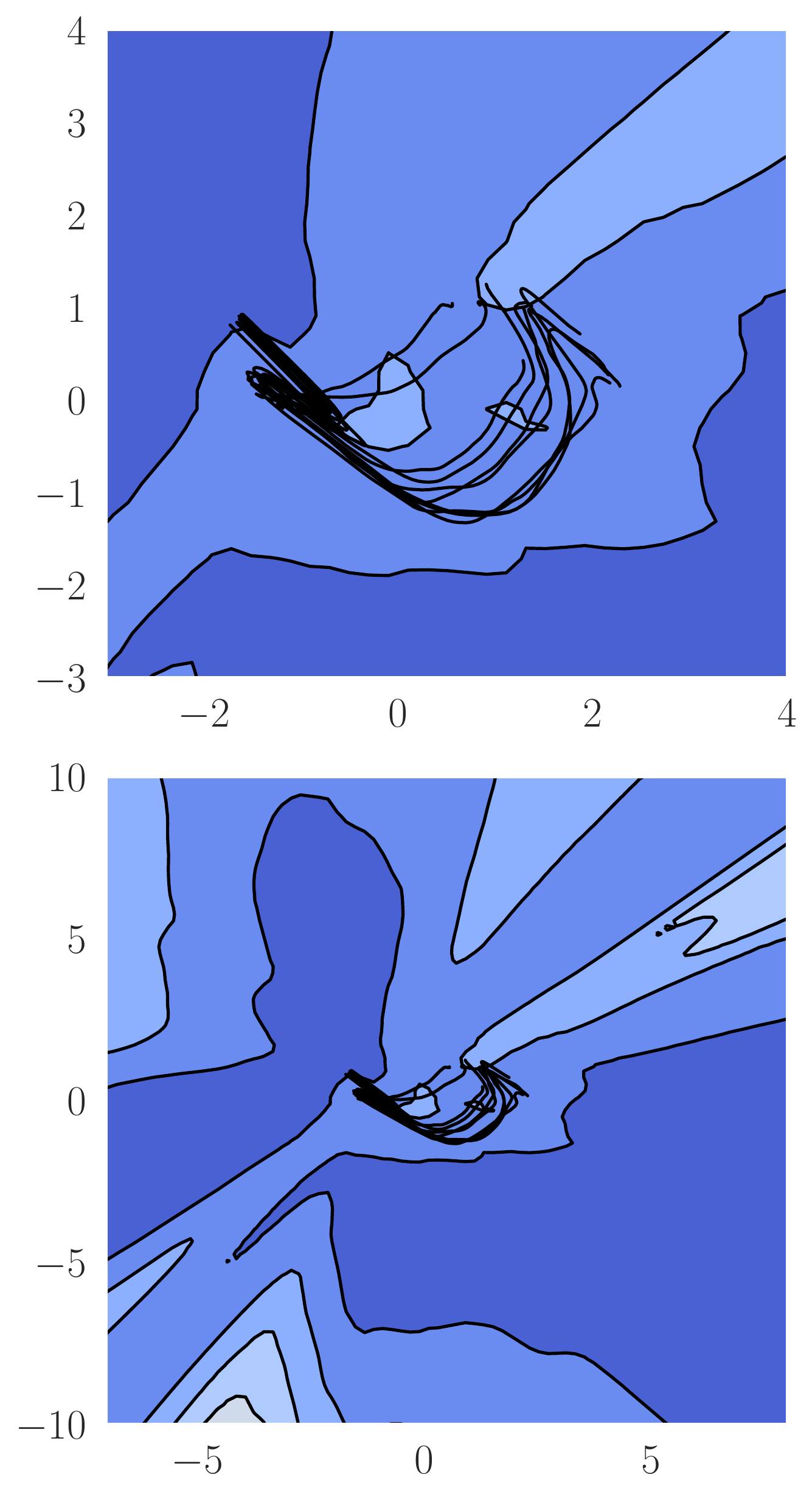}
    \caption{$K=2$}
\end{subfigure}
\begin{subfigure}[t]{0.28\columnwidth}
    \includegraphics[width=\columnwidth]{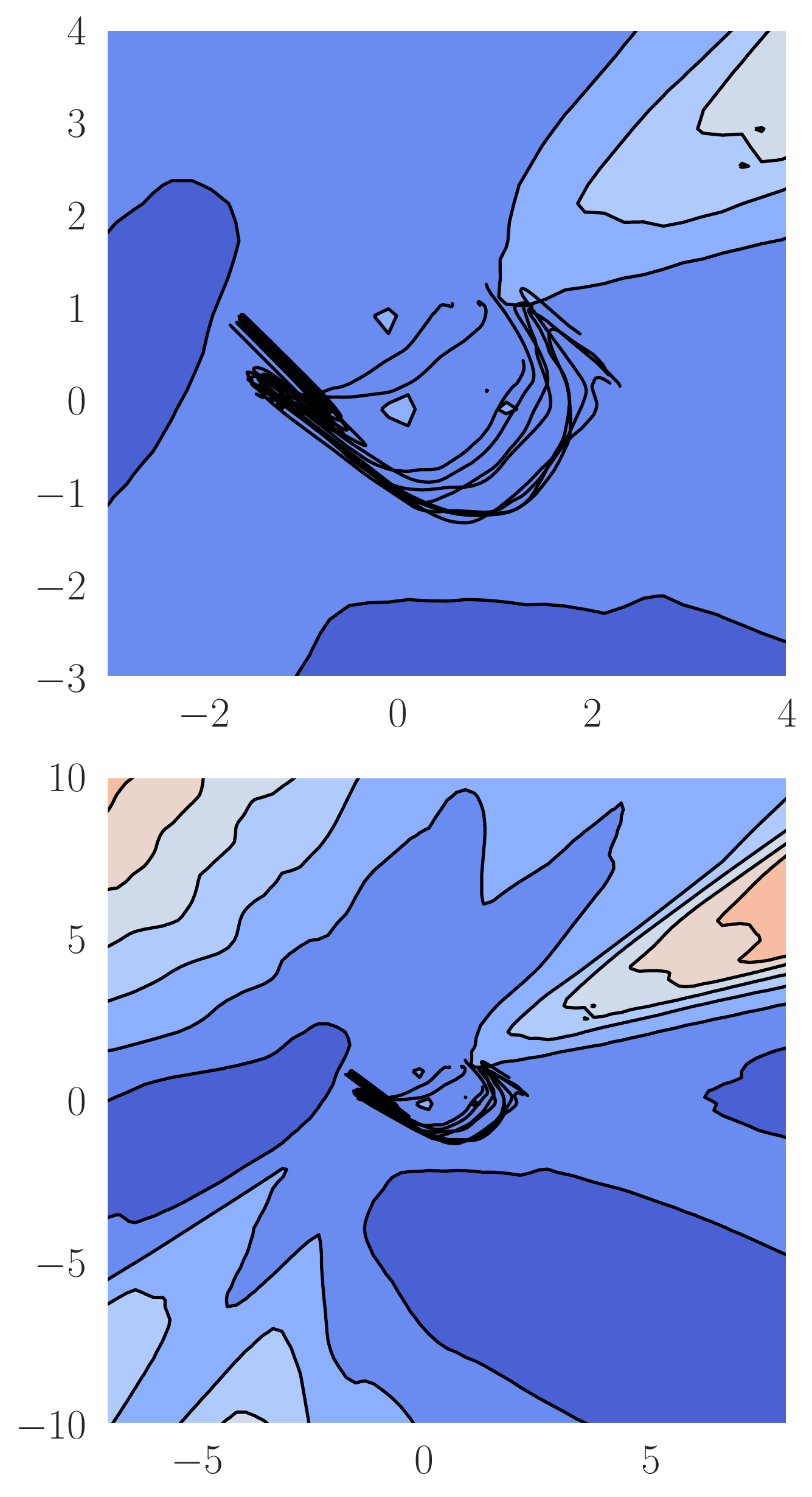}
    \caption{$K=4$}
\end{subfigure}
\begin{subfigure}[t]{0.071\columnwidth}
    \includegraphics[width=\columnwidth]{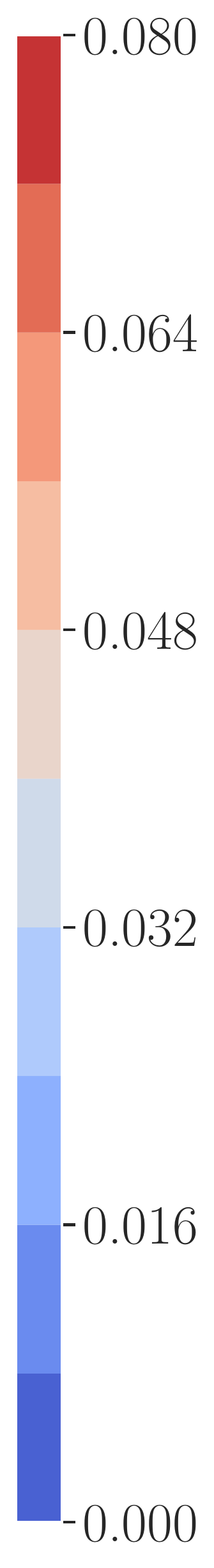}
\end{subfigure}
\begin{subfigure}[t]{0.28\columnwidth}
    \includegraphics[width=\columnwidth]{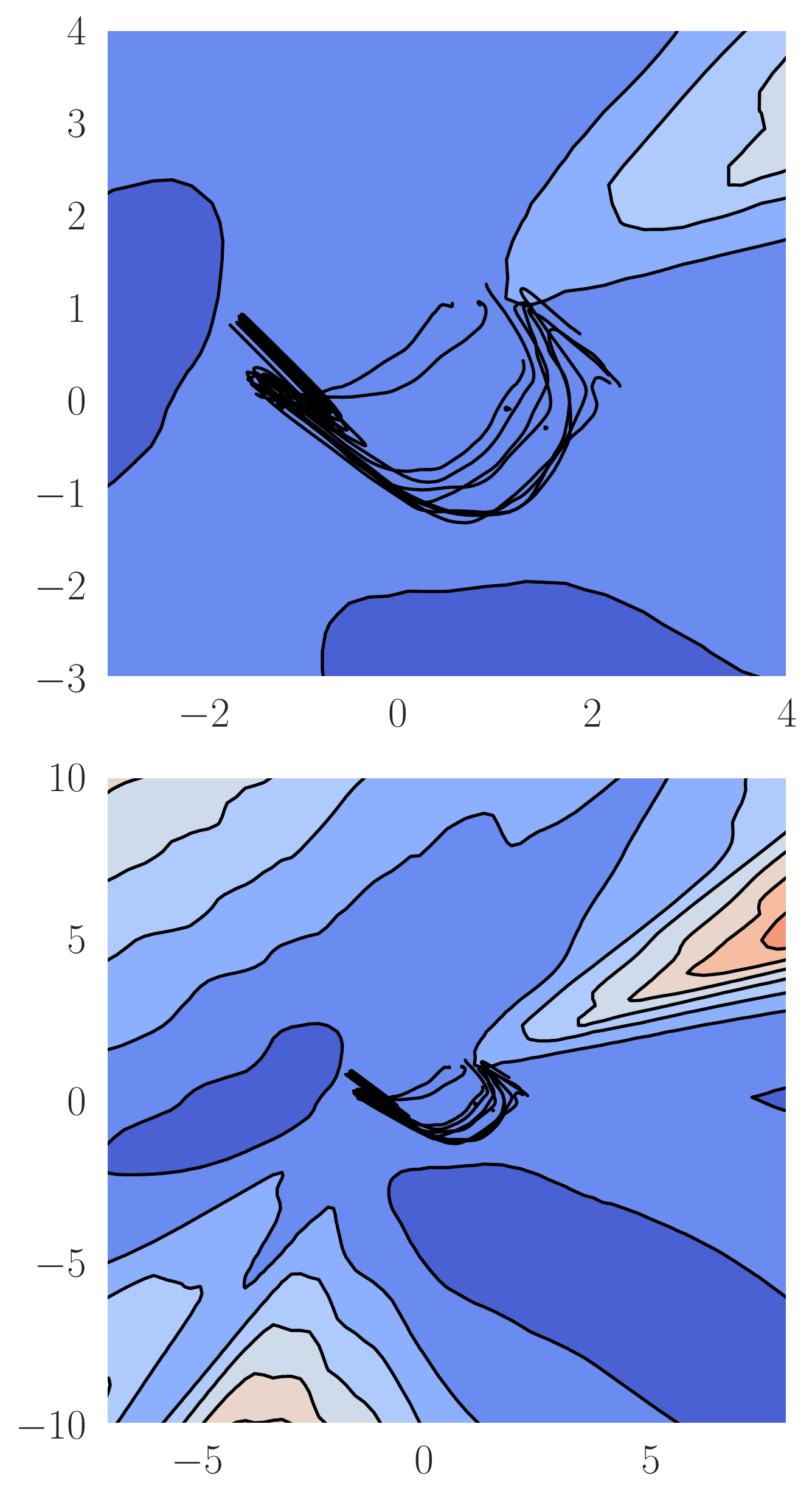}
    \caption{$K=6$}
\end{subfigure}
\begin{subfigure}[t]{0.28\columnwidth}
    \includegraphics[width=\columnwidth]{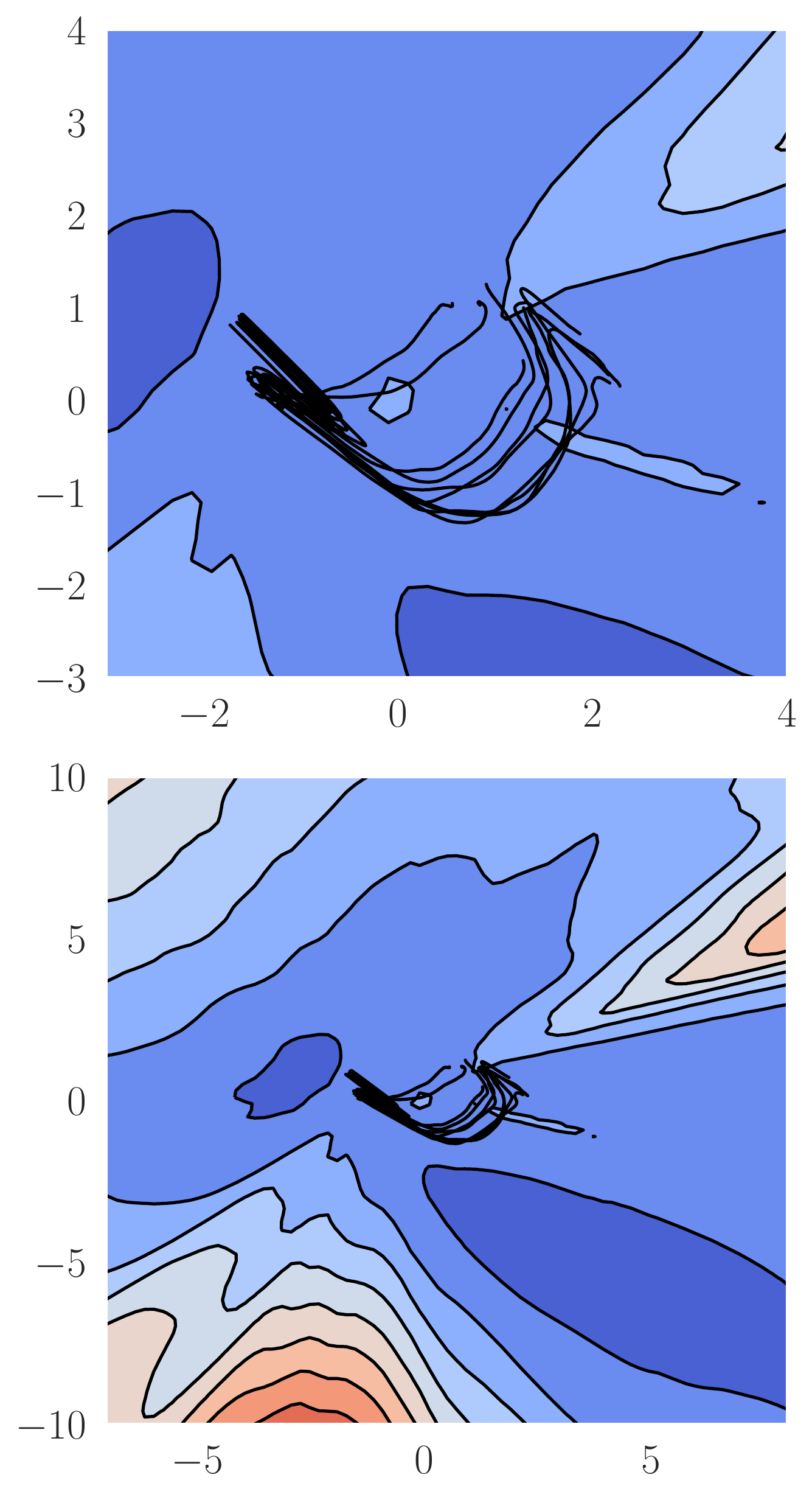}
    \caption{$K=8$}
\end{subfigure}
\begin{subfigure}[t]{0.28\columnwidth}
    \includegraphics[width=\columnwidth]{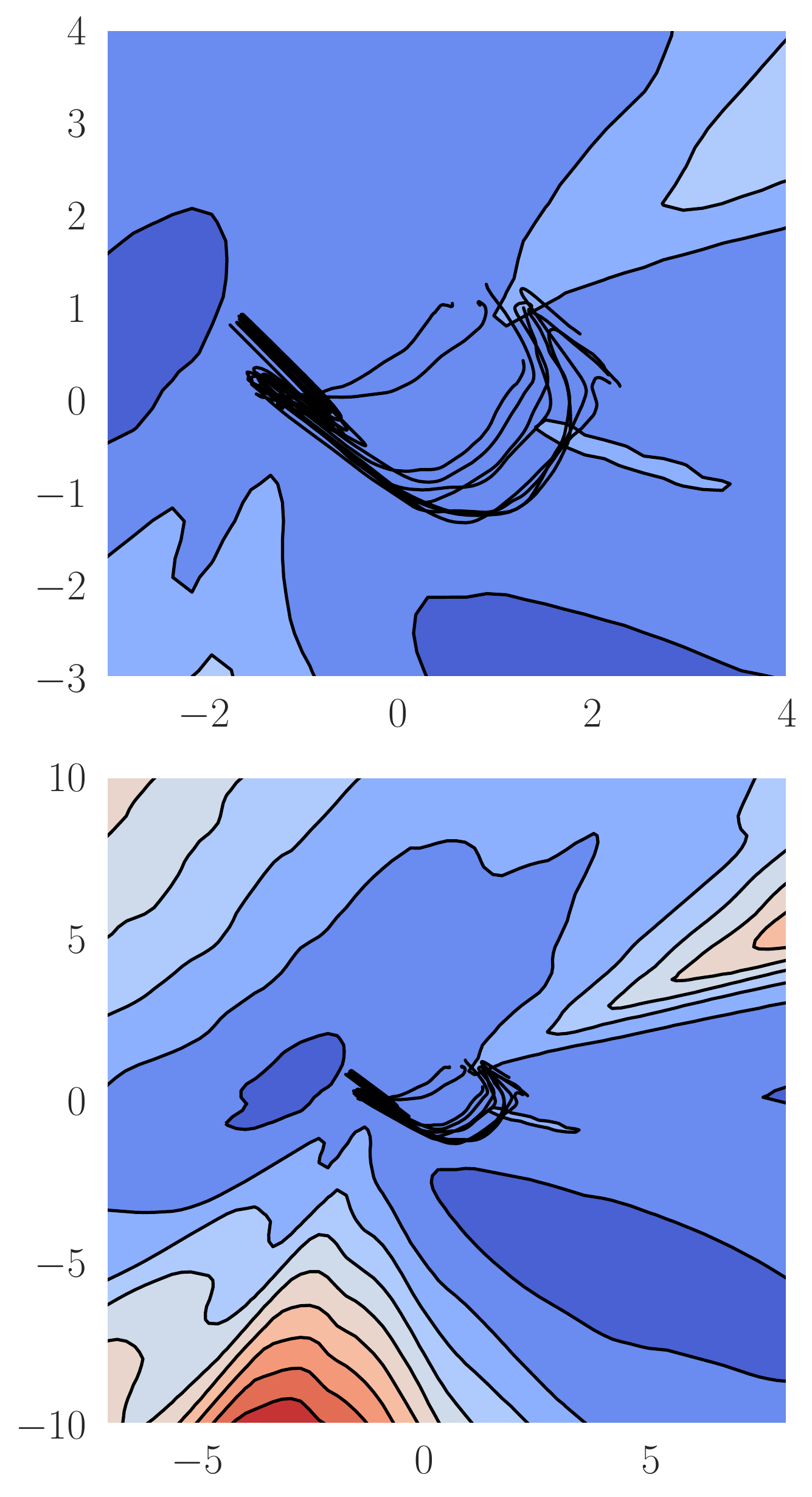}
    \caption{$K=10$}
\end{subfigure}
\begin{subfigure}[t]{0.071\columnwidth}
    \includegraphics[width=\columnwidth]{figures/appendix_heatmaps_ensemble/heatmap_ensemble_cbar.pdf}
\end{subfigure}
\caption{
    Predictive variances of ensembles of parametric neural network Gaussian behavioral policies \mbox{$\pi_{\psi^{1:K}}(\cdot \vbar \bs)$} with each neural network in the ensemble trained via MLE.
    The ensemble policies are marginally better calibrated than parametric neural network policies in that their predictive variance only collapses in some but not all regions away from the expert trajectories.
    }
\label{fig:app_ensemble_heatmaps}
\vspace{-40pt}
\end{figure}

\clearpage

\subsection{Parametric vs. Non-Parametric Predictive Variance Visualizations Across Environments}
\label{appsec:further_env_vis}

\Cref{fig:more_heatmaps} shows the predictive variances of non-parametric and parametric behavioral policies on low dimensional representations of the environments considered in Figures~\ref{fig:mujcoo_eval} and~\ref{fig:dex_eval} (excluding ``door-binary-v0'', which is shown in~\Cref{fig:heatmaps}).
\begin{figure}[h!]
\centering
\begin{subfigure}[t]{0.8\columnwidth}
\begin{subfigure}[t]{0.49\columnwidth}
    \begin{center}
        \large Non-Parametric
    \end{center}
    \includegraphics[width=\columnwidth]{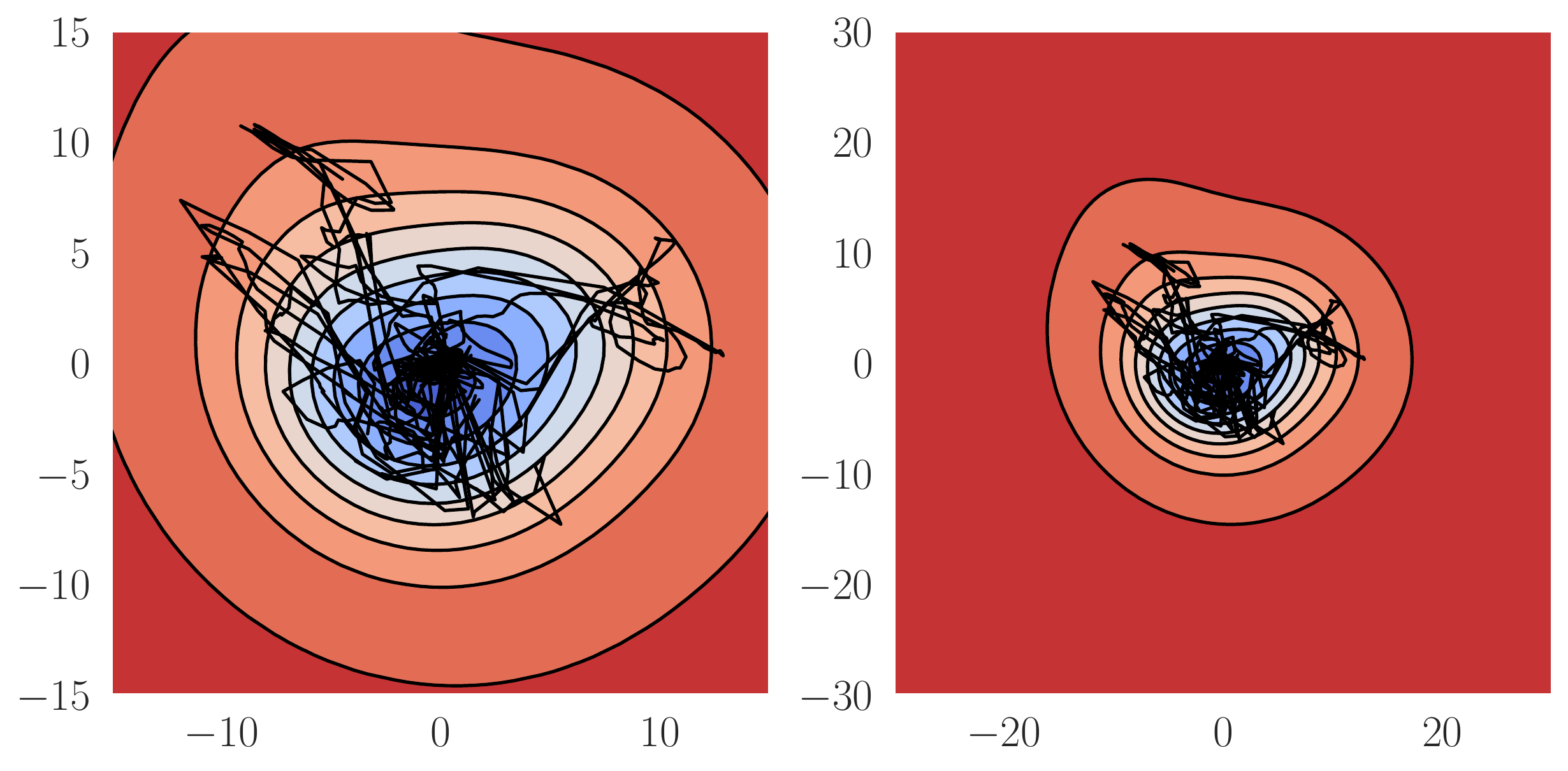}
\end{subfigure}
\begin{subfigure}[t]{0.49\columnwidth}
    \begin{center}
        \large Parametric
    \end{center}
    \includegraphics[width=\columnwidth]{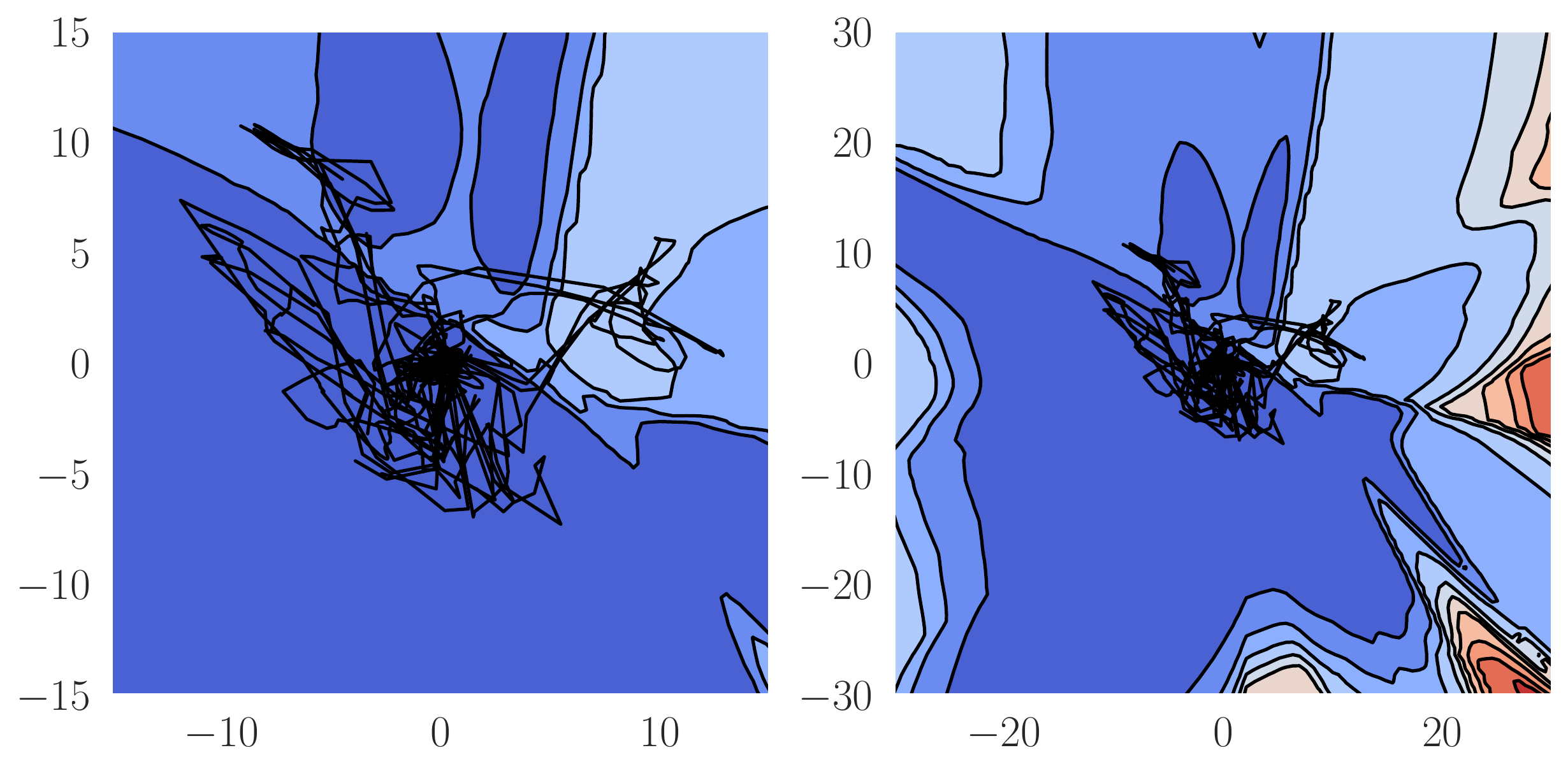}
\end{subfigure}
\\
\centering
\hspace*{6pt}
\begin{subfigure}[t]{0.47\columnwidth}
    \includegraphics[width=\columnwidth]{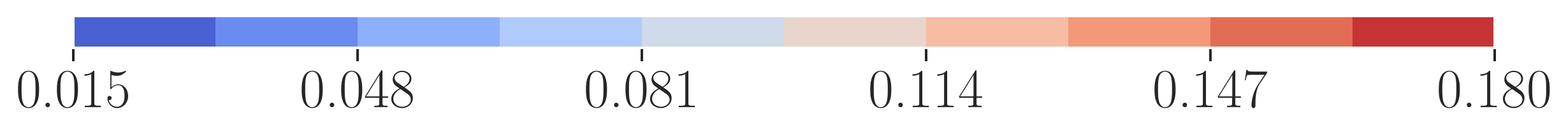}
\end{subfigure}
\hspace*{6pt}
\begin{subfigure}[t]{0.47\columnwidth}
    \includegraphics[width=\columnwidth]{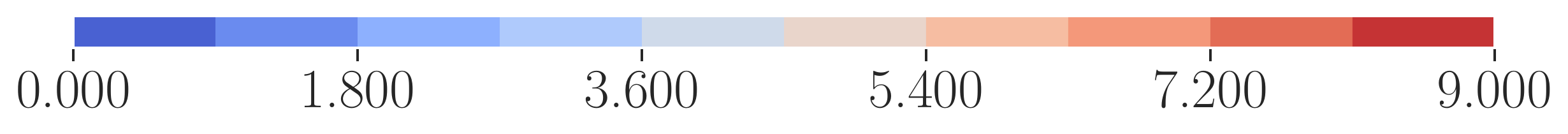}
\end{subfigure}
\caption{pen-binary-v0}
\end{subfigure}
\begin{subfigure}[t]{0.8\columnwidth}
\begin{subfigure}[t]{0.49\columnwidth}
    \includegraphics[width=\columnwidth]{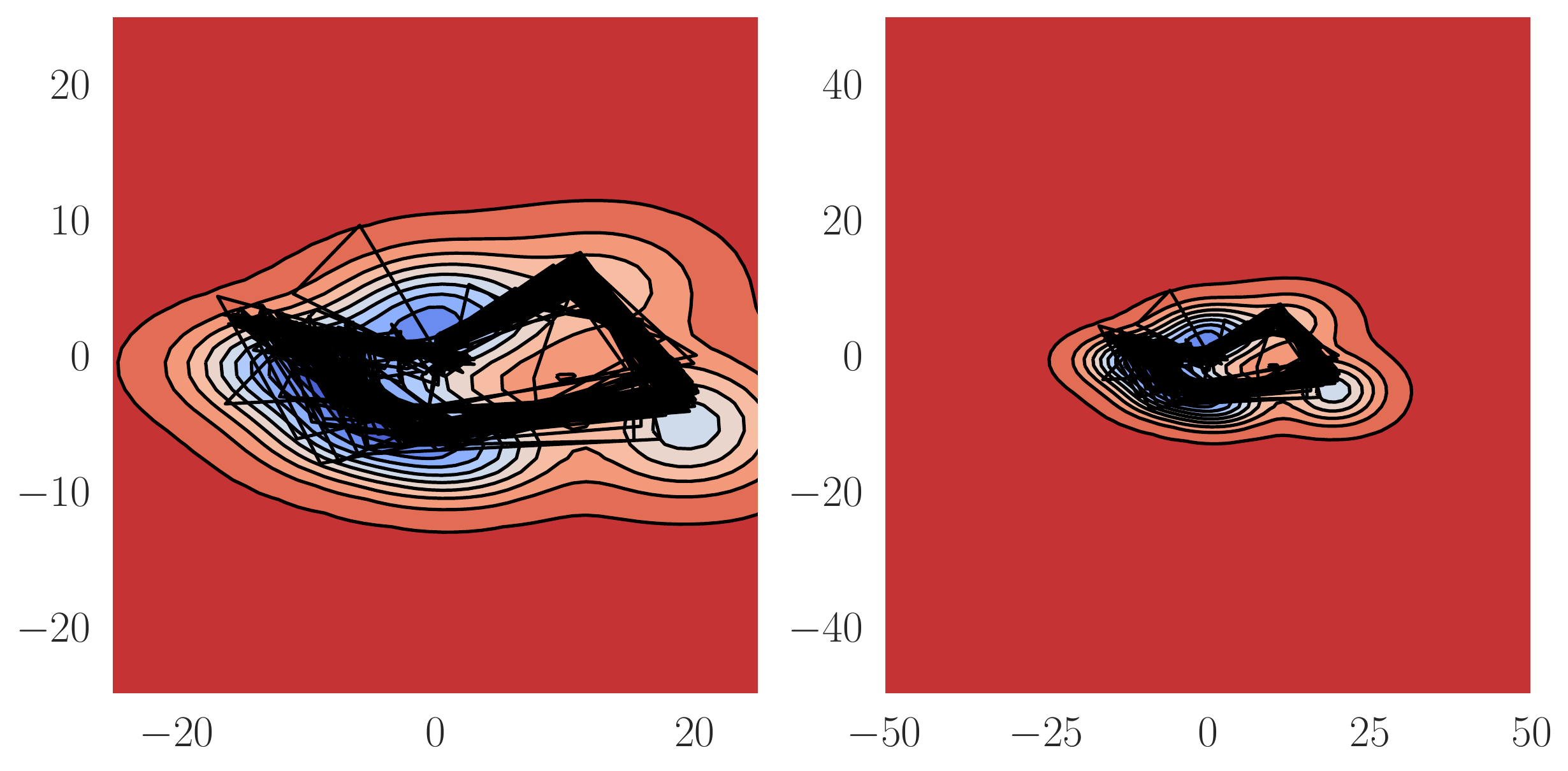}
\end{subfigure}
\begin{subfigure}[t]{0.49\columnwidth}
    \includegraphics[width=\columnwidth]{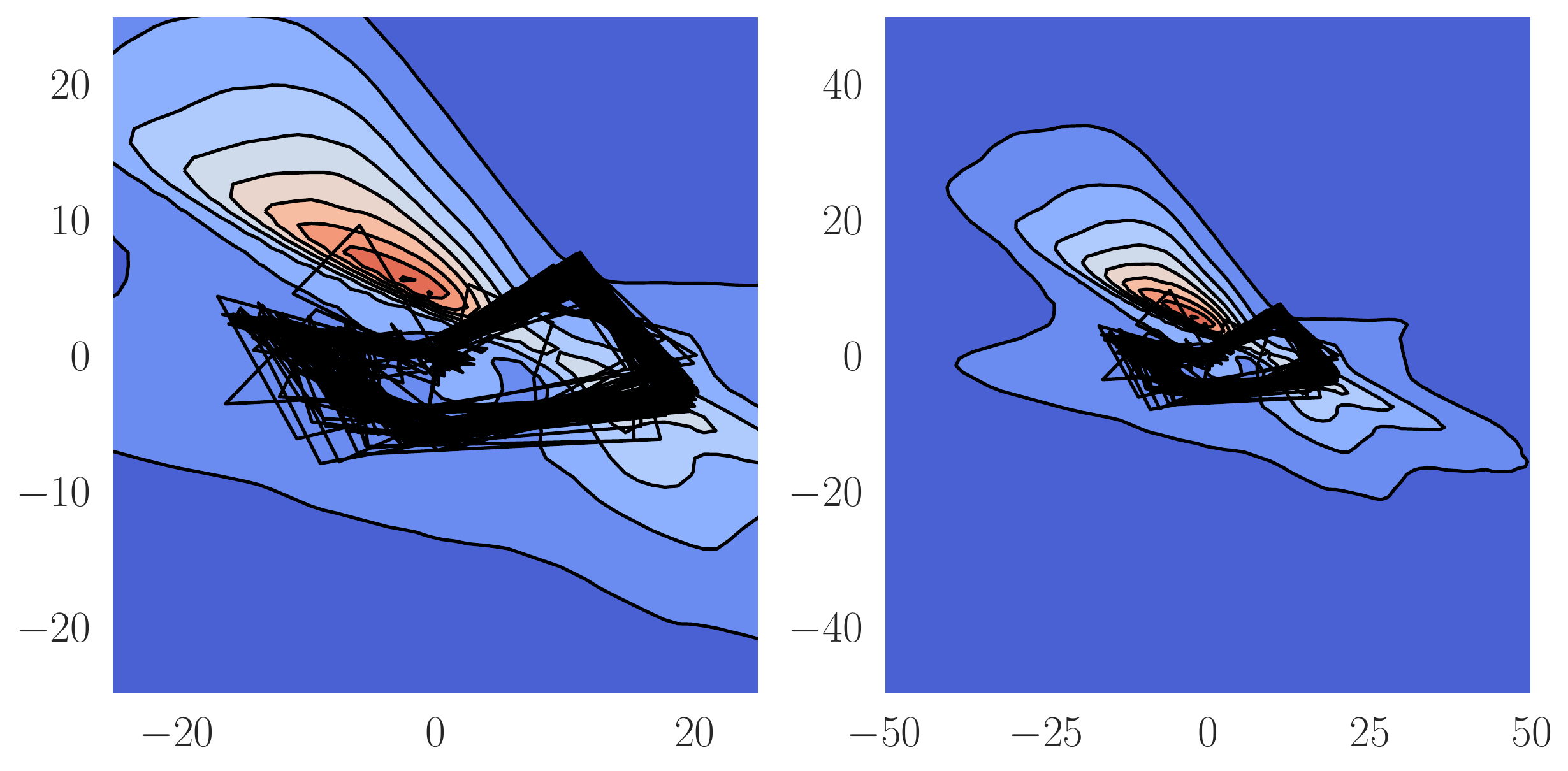}
\end{subfigure}
\\
\centering
\hspace*{6pt}
\begin{subfigure}[t]{0.47\columnwidth}
    \includegraphics[width=\columnwidth]{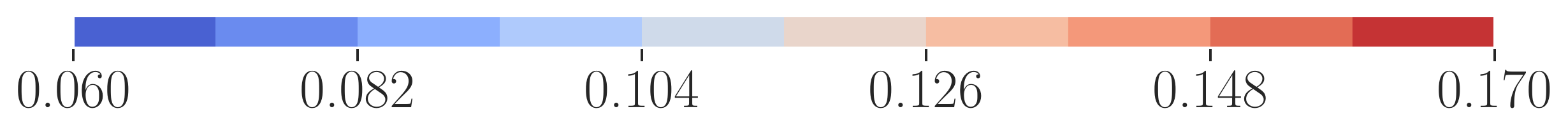}
\end{subfigure}
\hspace*{6pt}
\begin{subfigure}[t]{0.47\columnwidth}
    \includegraphics[width=\columnwidth]{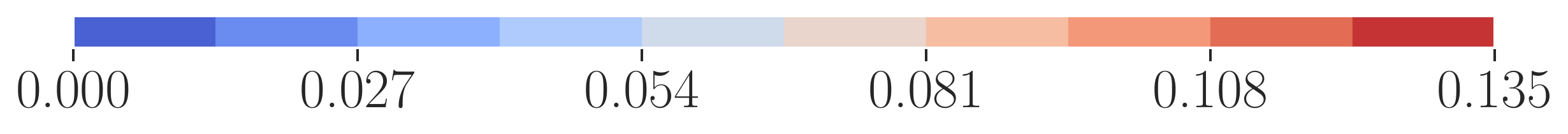}
\end{subfigure}
\caption{Ant-v2}
\end{subfigure}
\begin{subfigure}[t]{0.8\columnwidth}
\begin{subfigure}[t]{0.49\columnwidth}
    \includegraphics[width=\columnwidth]{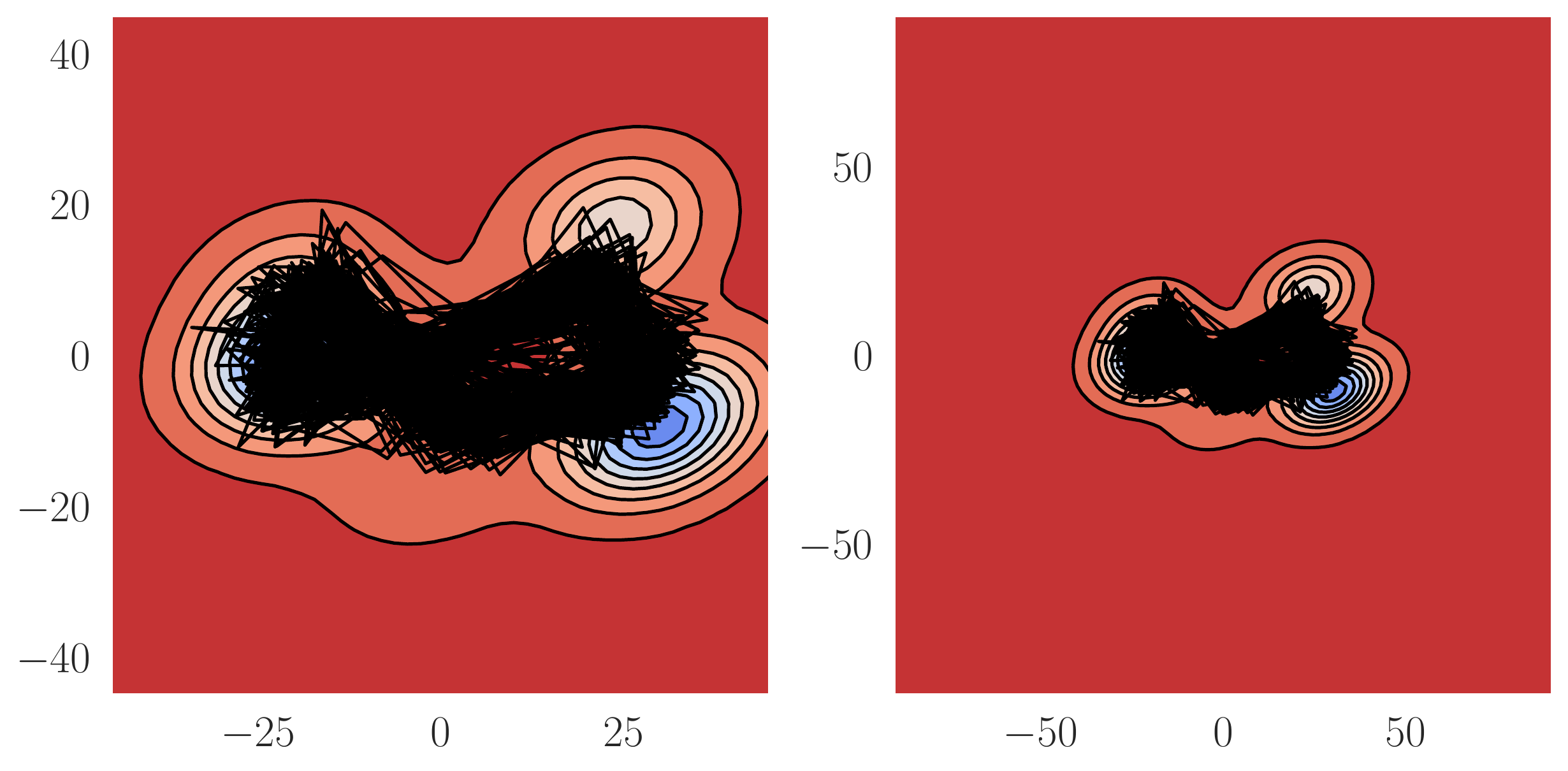}
\end{subfigure}
\begin{subfigure}[t]{0.49\columnwidth}
    \includegraphics[width=\columnwidth]{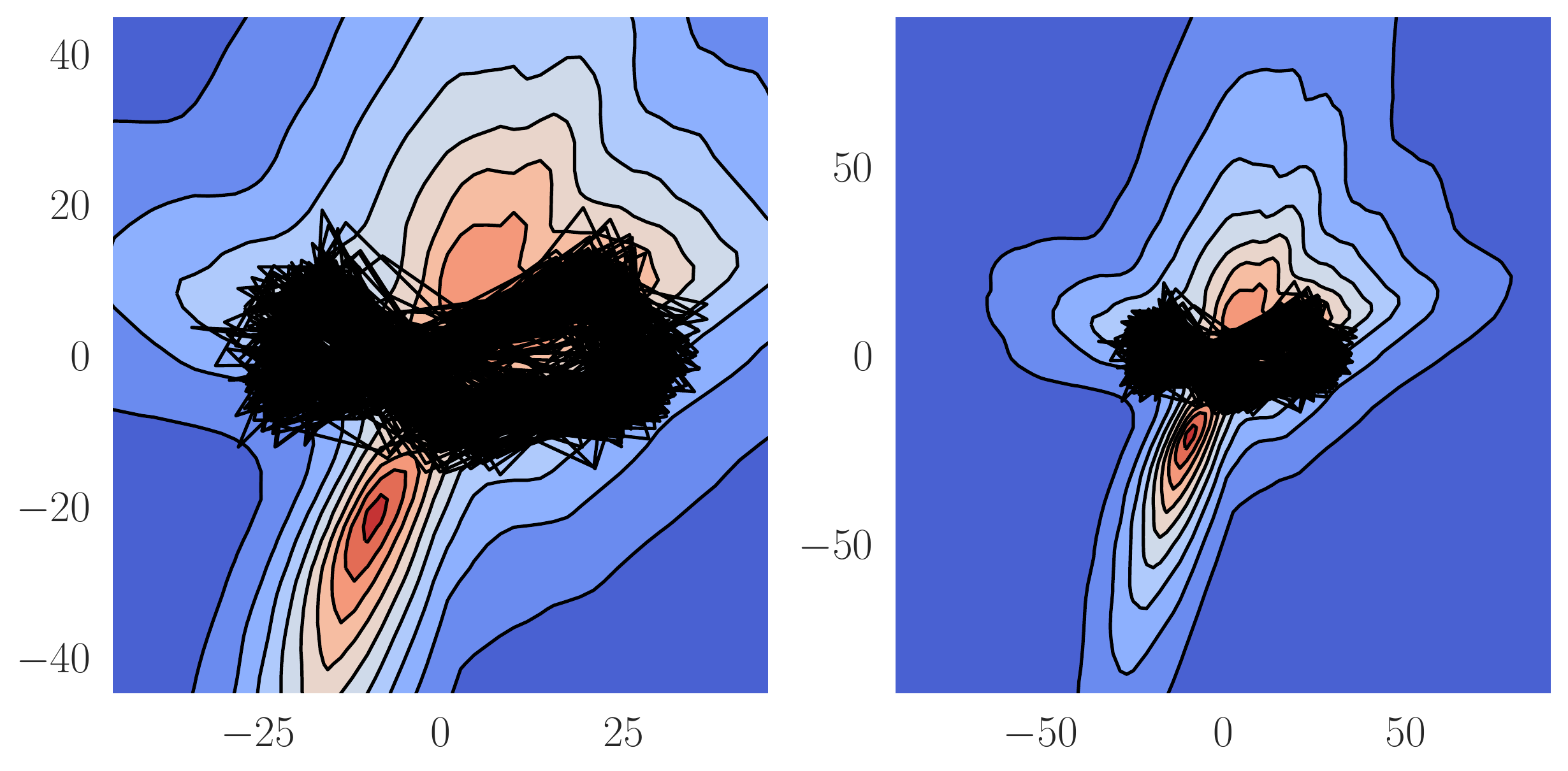}
\end{subfigure}
\\
\centering
\hspace*{6pt}
\begin{subfigure}[t]{0.47\columnwidth}
    \includegraphics[width=\columnwidth]{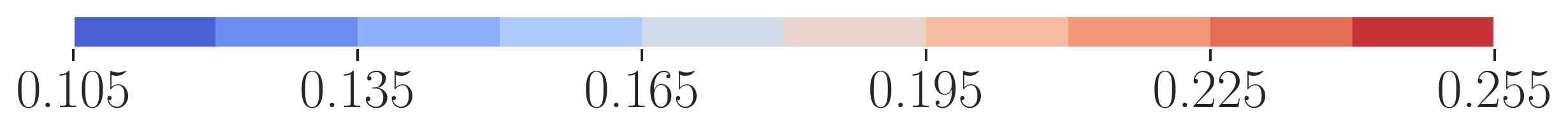}
\end{subfigure}
\hspace*{6pt}
\begin{subfigure}[t]{0.47\columnwidth}
    \includegraphics[width=\columnwidth]{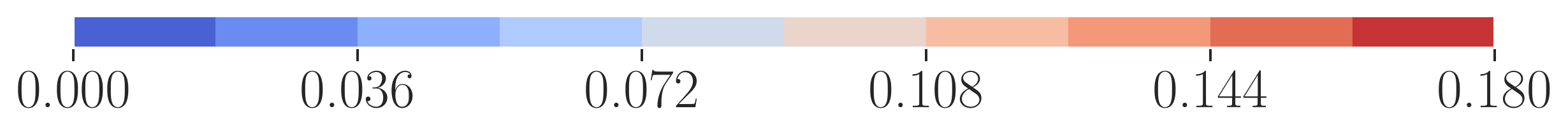}
\end{subfigure}
\caption{HalfCheetah-v2}
\end{subfigure}
\begin{subfigure}[t]{0.8\columnwidth}
\begin{subfigure}[t]{0.49\columnwidth}
    \includegraphics[width=\columnwidth]{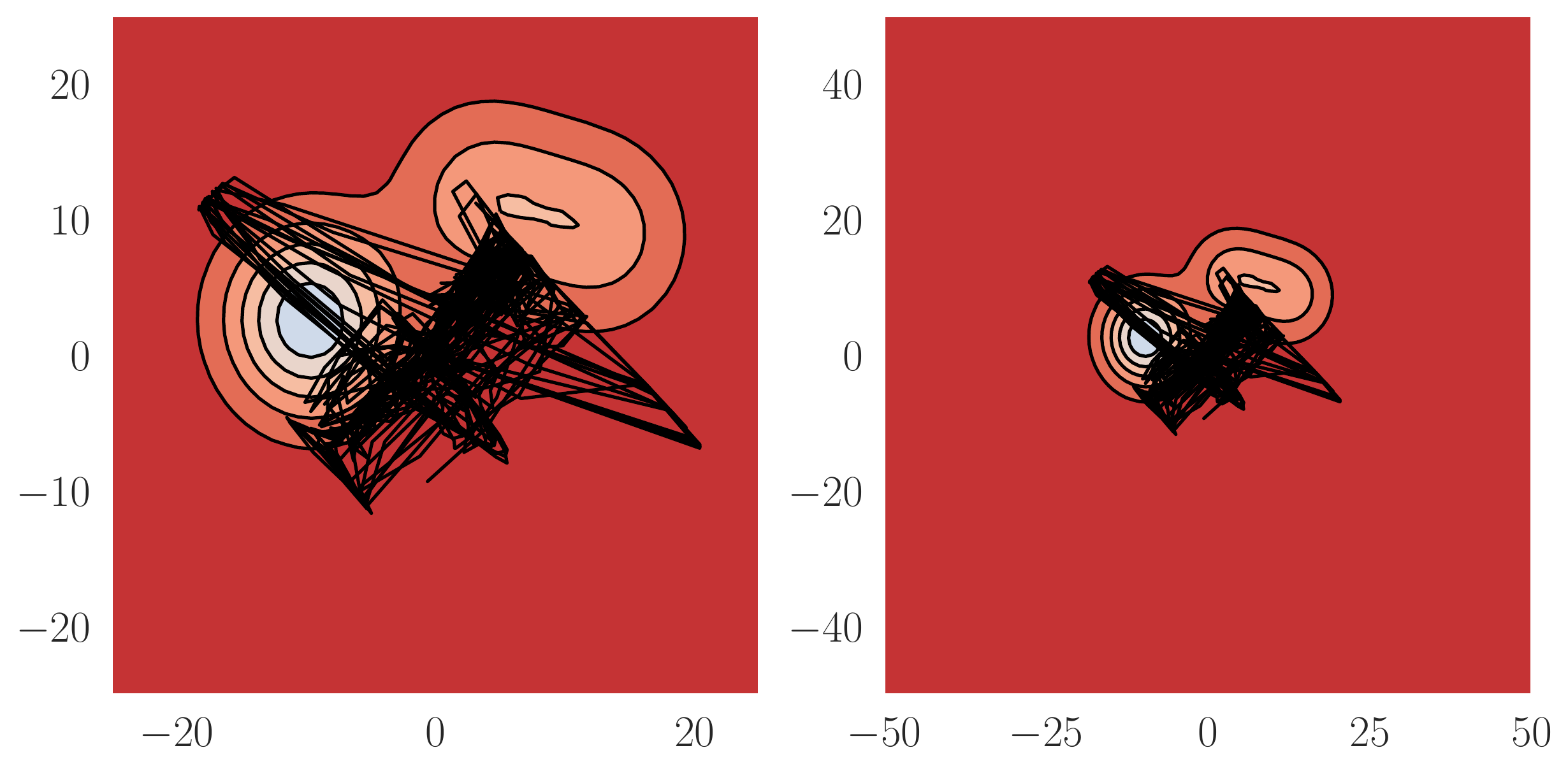}
\end{subfigure}
\begin{subfigure}[t]{0.49\columnwidth}
    \includegraphics[width=\columnwidth]{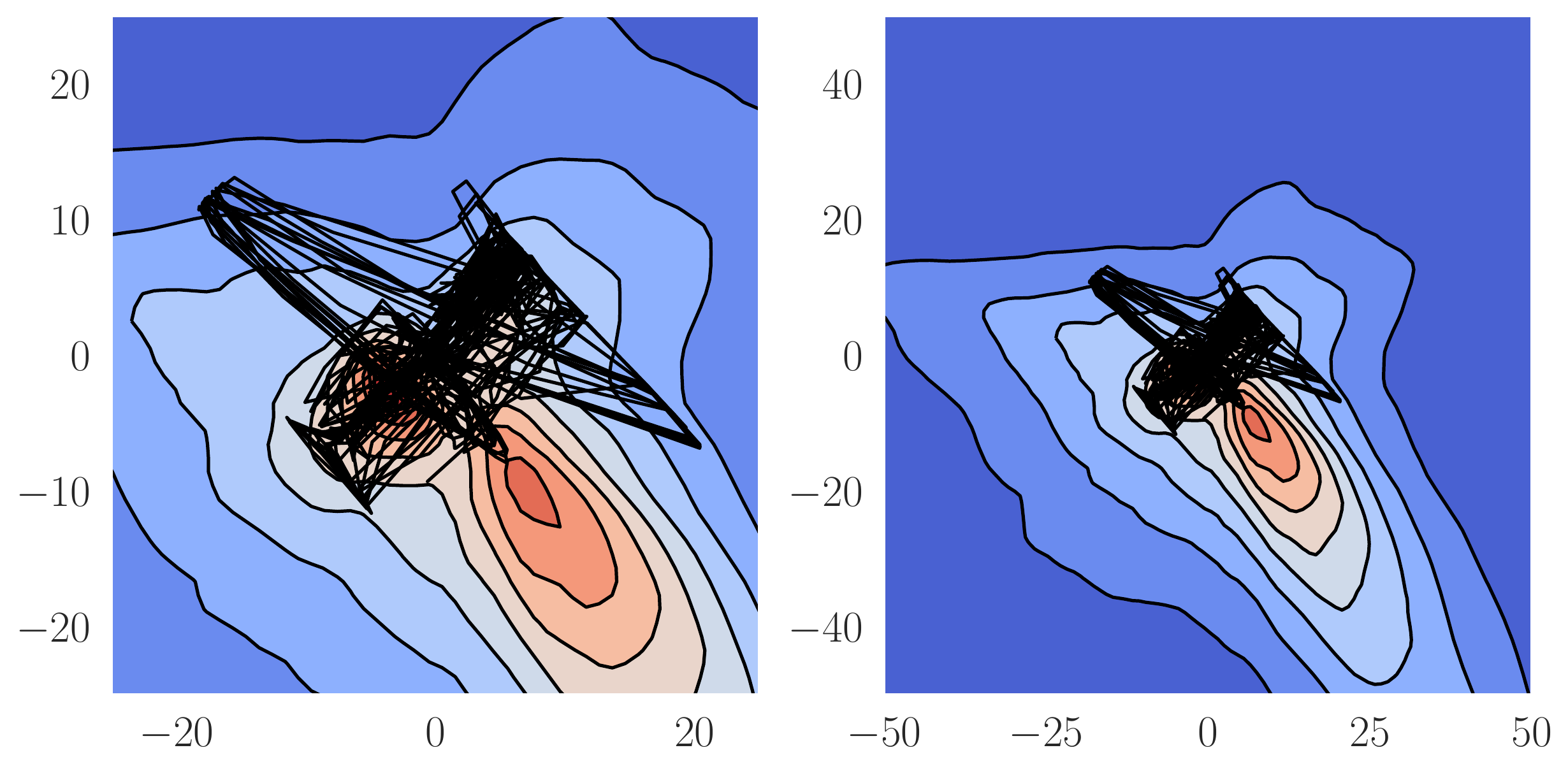}
\end{subfigure}
\\
\centering
\hspace*{6pt}
\begin{subfigure}[t]{0.47\columnwidth}
    \includegraphics[width=\columnwidth]{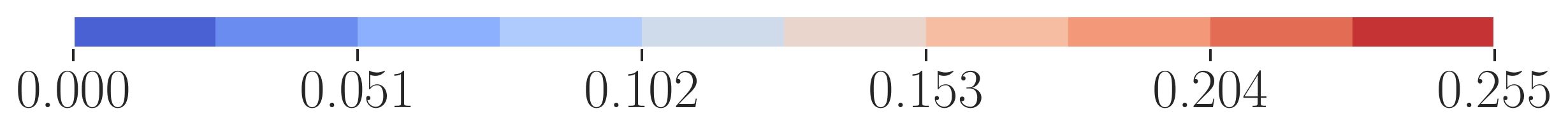}
\end{subfigure}
\hspace*{6pt}
\begin{subfigure}[t]{0.47\columnwidth}
    \includegraphics[width=\columnwidth]{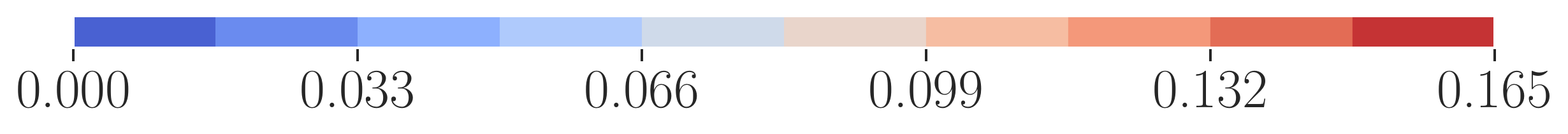}
\end{subfigure}
\caption{Walker-v2}
\end{subfigure}
\caption{
    Predictive variances of non-parametric and parametric behavioral policies on low dimensional representations of the environments considered in Figures~\ref{fig:mujcoo_eval} and~\ref{fig:dex_eval} (excluding ``door-binary-v0'', which is shown in~\Cref{fig:heatmaps}).
    \textbf{Left Column}: Non-parametric Gaussian process posterior behavioral policy \mbox{$\pi_{\mathcal{GP}}(\cdot \vbar \bs, \calD_{0}) = \mathcal{GP}(\bmu_{0}(\bs), \bSigma_{0}(\bs, \bs'))$}.
    \textbf{Right Column}: Parametric neural network Gaussian behavioral policy \mbox{$\pi_{\psi}(\cdot \vbar \bs) = \calN(\bmu_{\psi}(\bs), \bsigma^2_{\psi}(\bs))$}.
    Expert trajectories $\calD$ used to train the behavioral policies are shown in black.
    As in~\Cref{fig:heatmaps}, the predictive variance of the \gp is well-calibrated, whereas the predictive variance of the neural network is not.
    }
\label{fig:more_heatmaps}
\end{figure}

\clearpage

\subsection{Visual Comparison of Parametric vs. Non-Parametric Behavioral Policy Trajectories}
\label{appsec:visualizations}

To better understand the significance of the behavioral policy's model class, we sample trajectories from different behavioral policies on the door-opening task in~\Cref{fig:uncertainty_collapse}.
We visualize the mean trajectory and predictive variances of various behavioral policies showing a more sensible mean trajectory and predictive variance from the non-parametric \gp policy leading to better regularization compared to a behavioral policy parameterized by a neural network and the implicit uniform prior in SAC, a state-of-the-art RL algorithm.
On a randomly sampled unseen goal, we can see in~\Cref{fig:traj_plot} that a neural network policy trained via MLE produces a confident but incorrect trajectory.
The starting position is shown in black and the goal position is shown in green.
We also visualize a uniform prior, which SAC implicitly regularizes against.
Informative priors from offline data can greatly accelerate the online performance of such actor-critic methods.

\begin{figure*}[ht]
\centering
\vspace*{-8pt}
  \begin{subfigure}{0.18\textwidth}
    \includegraphics[width=\textwidth]{figures/hand_dapg_door.png}
    \caption{}
    \label{fig:door_3d_env}
  \end{subfigure}
  ~
  \hspace*{1cm}
  \begin{subfigure}{0.18\textwidth}
    \includegraphics[width=\textwidth]{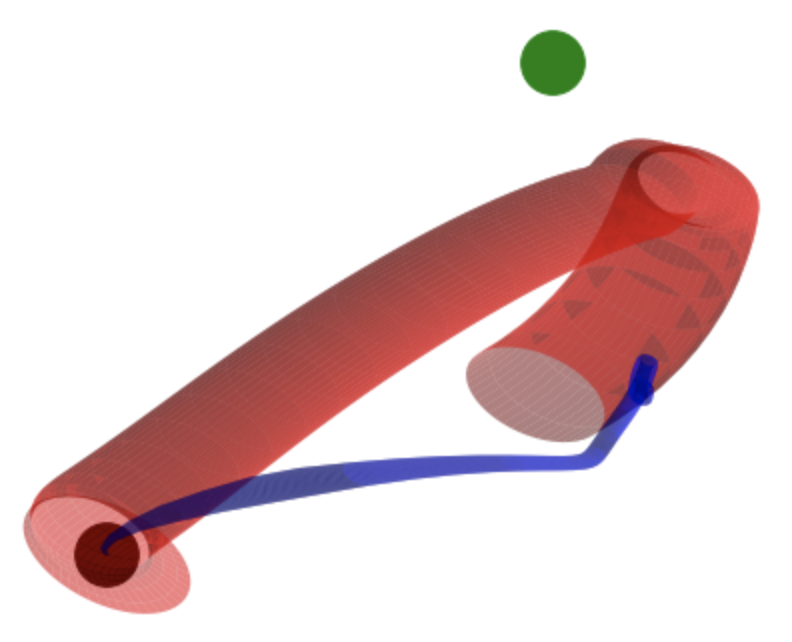}
    \caption{}
    \label{fig:traj_plot}
  \end{subfigure}
  ~
  \hspace*{1cm}
  \begin{subfigure}{0.18\textwidth}
    \includegraphics[width=\textwidth]{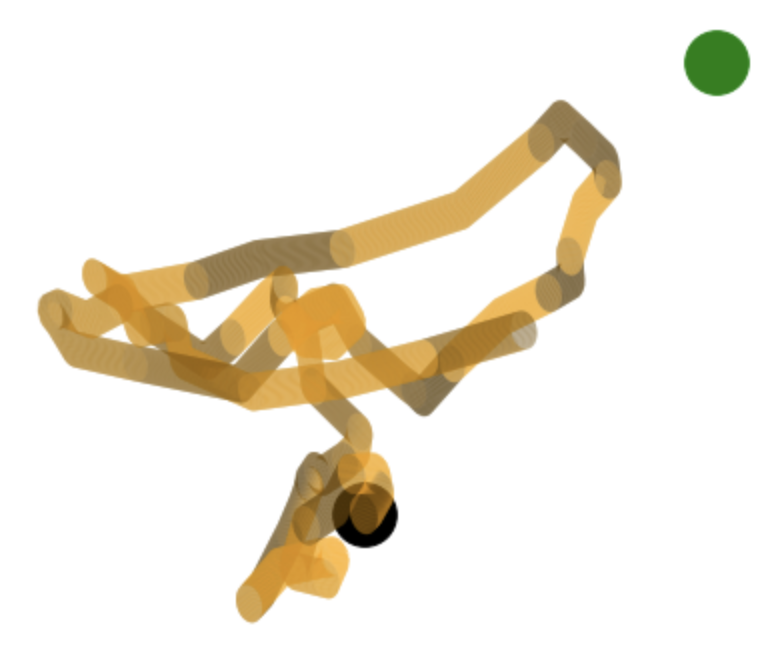}
    \caption{}
    \label{fig:rand_traj_plot}
  \end{subfigure}
  \\
  \begin{subfigure}{\textwidth}
  \includegraphics[trim=0 0 0 90, clip, width=\textwidth]{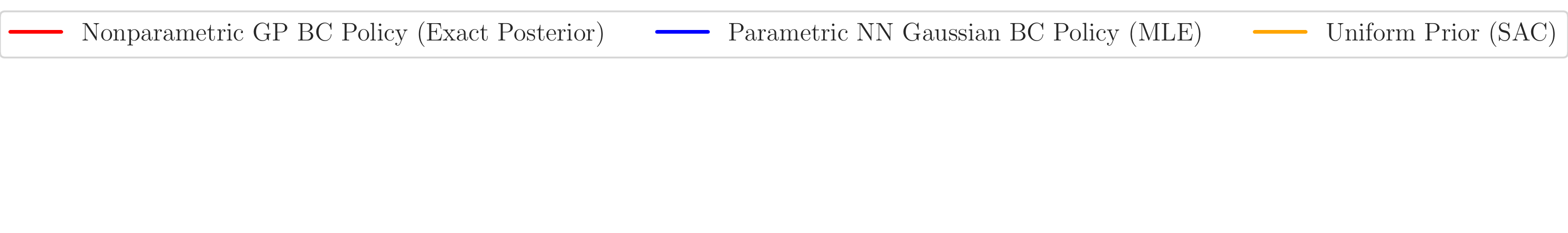}
  \end{subfigure}
  \vspace*{-45pt}
  \caption{
    Left: challenging door opening task~\citep{rajeswaran2018learning} which standard RL algorithms struggle on. Right and center: 3D plots of sampled mean trajectories and predictive variances from different behavioral policies from expert demonstration $\pi_{0}$, showing a more sensible mean trajectory and predictive variance from the non-parametric \gp policy leading to better regularization over both: \textbf{(b)} a behavioral policy using a poor model class, and \textbf{(c)} the implicit uniform prior in SAC.
    Starting position shown in \textbf{black} and goal position shown in \textbf{\color{darkgreen}{green}}.}
  \label{fig:uncertainty_collapse}
  \vspace*{-13pt}
\end{figure*}

\section{Further Implementation Details}

\subsection{Algorithmic Details}
\label{appsec:algorithm}

\textbf{Pre-training}$~$
On the dexterous hand manipulation tasks, before online training, the online policy is pre-trained to minimize the \kld to the behavioral reference policy on the offline dataset:
\begin{align*}
   J_{\mathcal{GP}}(\phi)
   \defines
   \mathbb{E}_{\bs \sim \calD_{0}} \left[ \DKL{\pi_{\phi}(\cdot \vbar \bs)}{\pi_{0}(\cdot \vbar \bs)} \right].
\end{align*}\\[-30pt]

\newcommand{\algrule}[1][.2pt]{\par\vskip.5\baselineskip\hrule height #1\par\vskip.5\baselineskip}

\begin{algorithm}[!hbt]
\caption{Non-Parametric Prior Actor--Critic}
\label{alg:nppac}
\begin{algorithmic}
   \State {\bfseries Input:} offline dataset $\calD_{0}$, initial parameters $\theta_{1}$, $\theta_{2}$, $\phi$, \gp $\pi_{0}(\cdot \vbar \bs) = \mathcal{GP}\bigl(m(\bs), k(\bs, \bs')\bigr)$
   
   \State Condition $\pi_{0}(\cdot \vbar \bs)$ on $\calD_{0}$ to obtain $\pi_{0}(\cdot \vbar \bs, \calD_{0})$
   
   \For{each offline batch}
      \State $\phi \gets \phi - \lambda_{\mathcal{GP}} \hat{\nabla}_{\phi} J_{\mathcal{GP}}(\phi)$ \Comment{Minimize KL between online and behavioral reference policy.}
   \EndFor
   
   \algrule
   
   \State $\bar{\theta}_{1} \gets \theta_{1}$, $\bar{\theta}_{2} \gets \theta_{2}$ \Comment{Initialize target network weights.}
   \State $\calD \gets \emptyset $ \Comment{Initialize an empty replay pool.}
   
   \For{each iteration}
   
   \For{each environment step}
      \State $\ba_{t} \sim \pi_{\phi}(\cdot \vbar \bs_{t})$
      \State $\bs_{t+1} \sim p(\cdot \vbar \bs_{t}, \ba_{t})$
      \State $\calD \gets \calD \cup \{ (\bs_{t}, \ba_{t}, r(\bs_{t}, \ba_{t}), \bs_{t+1}) \}$
   \EndFor
   
   \For{each gradient step}
      \State $\theta_{i} \gets \theta_{i} - \lambda_{Q} \hat{\nabla}_{\theta_{i}} J_{Q}(\theta_{i})$ for $i\in \{ 1, 2 \}$
      \State $\phi \gets \phi - \lambda_{\pi} \hat{\nabla}_{\phi} J_{\pi}(\phi)$ \Comment{Minimize $J_{Q}$ and $J_{\pi}$ using \gp $\pi_{0}(\cdot \vbar \bs, \calD_{0})$.}
      \State $\hat{\theta}_{i} \gets \tau \theta_{i} + (1 - \tau ) \hat{\theta}_{i}$ for $i\in \{ 1, 2 \}$ \Comment{Update target network weights.}
   \EndFor
   
   \EndFor
   
   \State {\bfseries Output:} Optimized parameters $\theta_{1}$, $\theta_{2}$, $\phi$
\end{algorithmic}
\end{algorithm}
\vspace*{-20pt}

\subsection{Hyperparameters}
\label{appsec:hyperparam}

\Cref{tab:fpac_hyperparams} lists the hyperparameters used for \textsc{n-ppac}.
For other hyperparameter values, we used the default values in the RLkit repository.
When multiple values are given, the former refer to MuJoCo continuous control and the latter to dexterous hand manipulation tasks.

\setlength{\tabcolsep}{47.5pt}
\begin{table}[h!]
\vspace*{-10pt}
\centering
\caption{\textsc{n-ppac} hyperparameters.}
\label{tab:fpac_hyperparams}
\begin{tabular}{l|c}
\toprule
\textbf{Parameter}                    & \textbf{Value(s)}   \\ \hline
optimizer                             & Adam                \\
learning rate                         & $3\cdot 10^{-4}$    \\
discount ($\gamma$)                   & 0.99                \\
reward scale                          & 1                   \\
replay buffer size                    & $10^{6}$            \\
number of hidden layers               & \{2, 4\}            \\
number of hidden units per layer      & 256                 \\
number of samples per minibatch       & \{256, 1024\}       \\
activation function                   & ReLU                \\
target smoothing coefficient ($\tau$) & 0.005               \\
target update interval                & 1                   \\
number of policy pretraining epochs   & 400                 \\ 
GP covariance function                & \{RBF, Mat\'ern\}   \\ \hline
\hline
\end{tabular}
\end{table}

\Cref{tab:gpopt_hyperparams} lists the hyperparameters used to train the Gaussian process on the offline data.
The hyperparameters are trained by maximizing the log-marginal likelihood.
The offline data is provided under the Apache License 2.0.

\setlength{\tabcolsep}{75.0pt}
\begin{table}[h!]
\centering
\vspace*{-10pt}
\caption{
    \gp optimization hyperparameters.
}
\begin{tabular}{l|c}
\toprule
\textbf{Parameter}                    & \textbf{Value}      \\ \hline
optimizer                             & Adam                \\
learning rate                         & 0.1                 \\
number of epochs                      & 500                 \\ \hline
\hline
\end{tabular}
\label{tab:gpopt_hyperparams}
\end{table}

\textbf{Hyperparameter Sweep for~\Cref{sec:model_comp_reduced}.}
For the BNN behavioral policy trained via Monte Carlo dropout, a dropout probability of $p=0.1$ and a weight decay coefficient $1e-6$ were used.
These values were found via a hyperparameter search over $\{ 0.1, 0.2 \}$ for $p$ and $\{ 1e-4, 1e-5, 1e-6, 1e-7 \}$ for the dropout probability and the weight decay coefficient, respectively.

For the deep ensemble behavioral policy, $M=15$ ensemble members and a weight decay coefficient of $1e-6$ were used.
The weight decay coefficient was found via a hyperparameter search over $\{ 5, 10, 15, 20\}$ for $M$ and $\{ 1e-4, 1e-5, 1e-6, 1e-7 \}$ for the weight decay coefficient.
Each ensemble member was trained on a different 80-20 training--validation split and initialized using different random seeds.

\end{appendices}

\end{document}